\documentclass[final]{cvpr}

\usepackage{times}
\usepackage{epsfig}
\usepackage{graphicx}
\usepackage{amsmath}
\usepackage{amssymb}
\usepackage{nicefrac}
\usepackage{mathtools}
\usepackage{multirow}
\usepackage{xcolor}
\usepackage{comment}
\usepackage{caption}
\usepackage{subcaption}
\usepackage{booktabs}
\usepackage[flushleft]{threeparttable}
\usepackage{amsthm}
\usepackage{float}

\usepackage{array}
\newcommand{\PreserveBackslash}[1]{\let\temp=\\#1\let\\=\temp}
\newcolumntype{C}[1]{>{\PreserveBackslash\centering}p{#1}}
\newcolumntype{R}[1]{>{\PreserveBackslash\raggedleft}p{#1}}
\newcolumntype{L}[1]{>{\PreserveBackslash\raggedright}p{#1}}

\usepackage{pifont}
\newcommand{\cmark}{\ding{51}}%
\newcommand{\xmark}{\ding{55}}%

\usepackage[pagebackref=true,breaklinks=true,colorlinks,bookmarks=false]{hyperref}

\newtheorem{definition}{Definition}

\newtheorem{proposition}{Proposition}
\newtheorem{prop}{Proposition}

\def\cvprPaperID{622} 
\def\confYear{CVPR 2022}

\begin{document}

\title{Towards Assessing and Characterizing the Semantic\\Robustness of Face Recognition}

\author{
\normalsize Juan~C. Pérez$^{1,2}$, 
Motasem Alfarra$^{1}$,
\normalsize Ali Thabet$^{3}$, Pablo Arbeláez$^{2}$, Bernard Ghanem$^{1}$ \\ 
\small $^{1}$King Abdullah University of Science and Technology (KAUST),\\ 
\small $^{2}$Center for Research and Formation in Artificial Intelligence, Universidad de los Andes,\\
\small $^{3}$Facebook Reality Labs\\
}

\maketitle
\begin{abstract}
Deep Neural Networks (DNNs) lack robustness against imperceptible perturbations to their input. 
Face Recognition Models (FRMs) based on DNNs inherit this vulnerability. 
We propose a methodology for assessing and characterizing the robustness of FRMs against semantic perturbations to their input. 
Our methodology causes FRMs to malfunction by designing adversarial attacks that search for identity-preserving modifications to faces. 
In particular, given a face, our attacks find identity-preserving variants of the face such that an FRM fails to recognize the images belonging to the same identity. 
We model these identity-preserving semantic modifications via direction- and magnitude-constrained perturbations in the latent space of StyleGAN. 
We further propose to characterize the semantic robustness of an FRM by statistically describing the perturbations that induce the FRM to malfunction. 
Finally, we combine our methodology with a certification technique, thus providing (i) theoretical guarantees on the performance of an FRM, and (ii) a formal description of how an FRM may model the notion of face identity.
\end{abstract}
\vspace{-.25cm}
\section{Introduction}
Deep Neural Networks (DNNs) have achieved impressive performance across fields such as computer vision~\cite{he2015delving}, natural language processing~\cite{mikolov2013efficient}, and reinforcement learning~\cite{mnih2013playing}.
Despite their remarkable success, DNNs are particularly vulnerable against imperceptible perturbations to their input, known as adversarial attacks~\cite{szegedy2014intriguing,goodfellow2015explaining}.
The unexpected vulnerability of DNNs against adversarial attacks highlights our narrow understanding of these models and their limitations~\cite{geirhos2020shortcut,zhang2021understanding}.

This vulnerability further poses potentially negative ramifications in the real-world.
Specifically, the deployment of DNNs for security-critical applications may be hampered, since ``why'' or ``how'' these systems fail is largely unknown.
A case of utmost importance in security-critical applications is that of Face Recognition Models (FRMs).
These systems have been the central subject of large amounts of research and engineering~\cite{kortli2020face}, and their use is widespread in everyday life, ranging from unlocking phones or personal computers to entering buildings or passing through airport security.
Thus, understanding FRMs and their failure modes can constrain how and when to trust FRMs in the real-world.
More importantly, interpreting FRMs can provide guides towards a more responsible and ethical use.

The pervasive vulnerability of DNNs against adversarial attacks calls for a unified methodology to study the robustness of FRMs in realistic settings. 
Specifically, we argue for studying the \textit{semantic} robustness of FRMs, concurring with other works~\cite{Bhattad2020Unrestricted,joshi2019semantic}, which account for semantic considerations in robustness settings.
Towards this objective, some works studied adversarial perturbations to attack~\cite{dong2019efficient,yang2021attacks,kakizaki2019adversarial} and diagnose~\cite{goswami2018unravelling,ruiz2021simulated} FRMs.

Other works criticized the physical and/or semantic realism of traditional adversarial perturbations, and developed sophisticated frameworks to introduce physical~\cite{liu2018beyond,sharif2016accessorize,sharif2019general} or semantic~\cite{kakizaki2019adversarial,qiu2020semanticadv} considerations.
Despite such progress in exploring the vulnerability of FRMs against perturbations, there is still no consensus regarding a methodology for studying the semantic robustness of FRMs.

\begin{figure*}
    \centering
    \includegraphics[width=\textwidth]{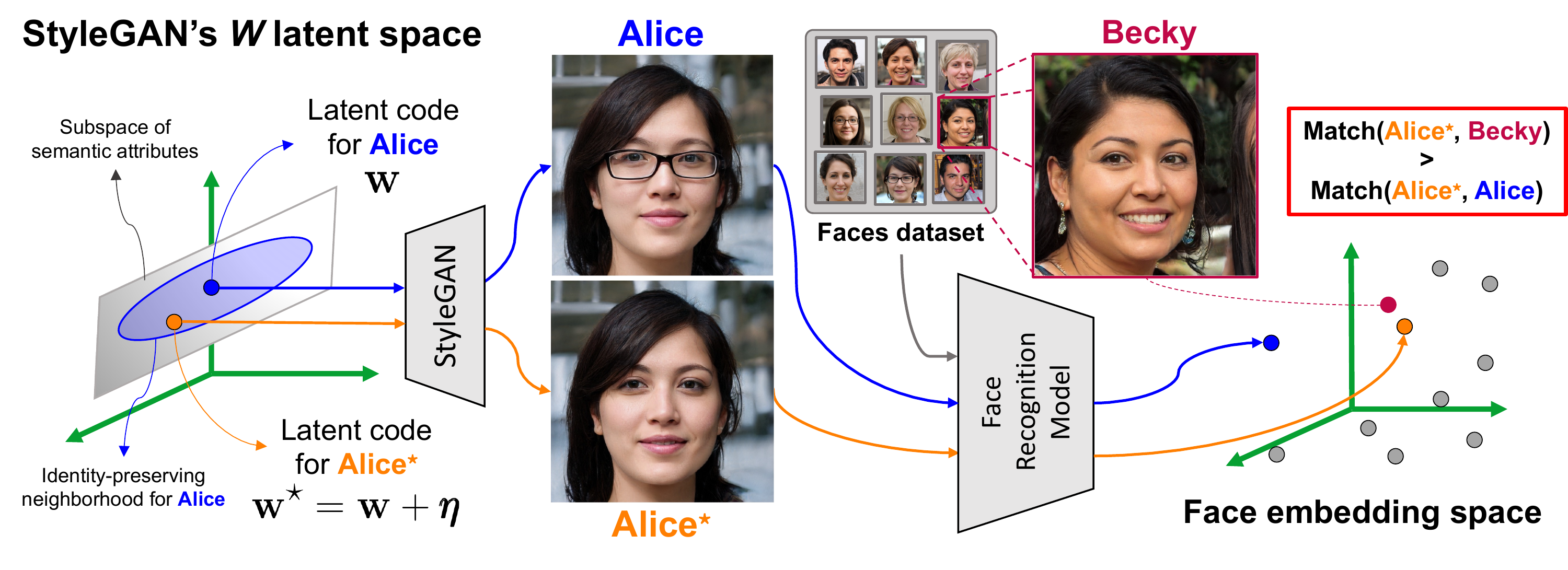}
    \vspace{-0.85cm}
    \caption{\textbf{Searching for identity-preserving modifications via StyleGAN's latent space.} 
    Given \textcolor{blue}{\fontfamily{cmss}\selectfont\textbf{Alice}}'s latent code ($\mathbf{w}$) and a subspace of semantic attributes, we draw an identity-preserving neighborhood around \textcolor{blue}{\fontfamily{cmss}\selectfont\textbf{Alice}}.
    We search for \textcolor{orange}{\fontfamily{cmss}\selectfont\textbf{Alice$^\star$}}, a variant of \textcolor{blue}{\fontfamily{cmss}\selectfont\textbf{Alice}}'s face, whose latent code $\mathbf{w}^\star = \mathbf{w} + \pmb{\eta}$ lies in this neighborhood.
    We generate the images corresponding to both $\mathbf{w}$ and $\mathbf{w}^\star$ via StyleGAN and find that, despite remarkable similarities between the faces, a Face Recognition Model's embedding space may suggest to match \textcolor{orange}{\fontfamily{cmss}\selectfont\textbf{Alice$^\star$}} with \textcolor{purple}{\fontfamily{cmss}\selectfont\textbf{Becky}} rather than with \textcolor{blue}{\fontfamily{cmss}\selectfont\textbf{Alice}}.
    Best viewed in color.
    }
    \label{fig:pull}
    \vspace{-.45cm}
\end{figure*}

In this work, we propose and deploy a methodology for systematically assessing and characterizing the semantic robustness of Face Recognition Models.
Our methodology achieves this objective by modeling identity-preserving semantic modifications via constrained perturbations in the latent space of Generative Adversarial Networks (GANs)~\cite{goodfellow2014generative}, specifically the popular StyleGAN~\cite{karras2019style}.
Please refer to Figure~\ref{fig:pull} for a visual guide through our methodology.
Under this model of identity-preserving modifications, our methodology then connects such modifications with the domain of adversarial robustness~\cite{szegedy2014intriguing,cohen2019certified} to study the semantic robustness of FRMs.

Our methodology models identity-preserving modifications of semantic attributes by introducing constrained perturbations in the latent space of StyleGAN~\cite{karras2019style}.
In particular, we leverage InterFaceGAN~\cite{shen2020interpreting,shen2020interfacegan}, a recent method for interpreting the latent space of StyleGAN for synthetic face generation.
Identity-preserving perturbations are constrained both in direction and magnitude: only the subspace spanned by certain attributes is allowed, and different attributes can be perturbed to different extents.
We adapt adversarial attacks to this model of identity-preserving modifications, and then search for semantic adversarial examples for FRMs by employing constrained- and minimum-perturbation adversarial attacks~\cite{madry2018towards,croce2020minimally}.
We then characterize the semantic robustness of an individual FRM through a statistical procedure that describes the adversarial examples that fool the FRM.
Finally, we show how our methodology can leverage an approach for certified robustness.
Certifying an FRM provides us with \textit{(i)} theoretical guarantees on the FRM's performance and \textit{(ii)} insights into how the FRM may model the notion of face identity, as delivered by a formal description of the extent to which a face's attributes can vary while the FRM's output remains constant. 

\vspace{2pt}\noindent \textbf{Contributions}. Our contributions are three-fold. 
    \textbf{(1)} We propose a methodology for studying the robustness of Face Recognition Models (FRMs) against semantic perturbations.
    For that purpose, we extend widely-used paradigms of adversarial attacks
    to our methodology to search for semantic adversarial examples.
    \textbf{(2)} We propose a procedure for characterizing the semantic robustness of FRMs by statistically describing the semantic adversarial examples we find.
    \textbf{(3)} We show how our methodology can be combined with certification techniques, granting formal guarantees on the performance of an FRM against semantic perturbations and insights regarding how the FRM models identity.
    We provide our PyTorch~\cite{paszke2019pytorch} implementation online\footnote{Available at \url{https://github.com/juancprzs/certifyingFaceRecognition}}.

\section{Related Work}
\vspace{2pt}\noindent\textbf{Adversarial attacks.}
Previous works~\cite{szegedy2014intriguing, goodfellow2015explaining} showed that adversarial examples, \ie images modified by small maliciously-crafted additive perturbations, could deteriorate the impressive recognition performance of DNNs.
This observation led to research on designing procedures, or ``attacks'', to find adversarial examples for DNNs.
Attacks can be dichotomously categorized into two paradigms~\cite{dong2020benchmarking} according to how the underlying optimization problem accounts for the perturbation's magnitude: either as a constraint~\cite{madry2018towards}, known as \textit{constrained-perturbation attacks}, or as the objective itself~\cite{moosavi2016deepfool}, known as \textit{minimum-perturbation attacks}.
In this work, we find semantic adversarial examples for FRMs by adapting adversarial attacks from both paradigms to our methodology and searching in the latent space of StyleGAN.
For constrained-perturbation attacks, we adopt Projected Gradient Descent (PGD) attacks~\cite{madry2018towards}, while for minimum-perturbation attacks, we adopt Fast Adaptive Boundary (FAB) attacks~\cite{croce2020minimally}.
Moreover, we characterize the semantic robustness of a target FRM by proposing a statistical procedure to describe the adversarial examples found by each attack in terms of semantic attributes.

\vspace{2pt}\noindent\textbf{Certified robustness.}
Adversarial attacks can be used to empirically assess the robustness of DNNs~\cite{croce2020reliable,carlini2017towards,carlini2019evaluating}.
However, an attack's inability to find adversarial examples for a DNN does not imply the nonexistence of adversarial examples for this DNN~\cite{carlini2019evaluating,athalye2018obfuscated}.
To address this shortcoming, a line of works studied ``certifiable robustness''~\cite{lecuyer2019certified,li2018second,wong2018provable}.
This field studies models that are provably robust against additive input perturbations of restricted magnitude, thus guaranteeing the nonexistence of adversarial examples at such magnitude.
Randomized smoothing~\cite{cohen2019certified} is one such approach and one of the main certification frameworks that scales to large DNNs and datasets. 
In this work, we extend randomized smoothing to combine it with our methodology.
By certifying FRMs against semantic perturbations, we provide performance guarantees and insights into how FRMs recognize faces and, thus, model the notion of identity.

\vspace{2pt}\noindent\textbf{Adversarial examples for Face Recognition Models.}
Face Recognition Models (FRMs) are computer vision models, whose objective is recognizing human faces.
Modern FRMs leverage DNNs to achieve impressive performance~\cite{schroff2015facenet,deng2019arcface,deng2020sub}.
The discovery of adversarial examples led to a stream of works attacking FRMs.
Some works perturbed the FRM's input in pixel space~\cite{dong2019efficient,goswami2018unravelling}, while others proposed sophisticated attacks~\cite{yang2021attacks,song2018constructing,deb2020advfaces,dabouei2019fast} that accounted for physical~\cite{liu2018beyond,Bhattad2020Unrestricted,sharif2016accessorize,sharif2019general} and semantic~\cite{kakizaki2019adversarial,joshi2019semantic,qiu2020semanticadv} considerations in attacking FRMs in the real-world.
These works showcased the vulnerability of FRMs against adversarial examples, both in pixel space and in more semantically-inclined spaces.
Sharing spirit with our work, Song~\etal~\cite{song2018constructing} trained a class-conditional GAN and conducted attacks in its latent space.
Similarly, Qiu~\etal~\cite{qiu2020semanticadv} interpolated in the latent space of an image-conditional GAN to search for semantic adversarial examples.
Joshi~\etal~\cite{joshi2019semantic} optimized over a Fader~\cite{lample2017fader} network's latent space to fool facial attribute classifiers.
Ruiz~\etal~\cite{ruiz2021simulated} searched for adversarial examples in a simulator's parametric space to detect weaknesses in FRMs.
Most recently, Li~\etal~\cite{li2021exploring} fooled deepfake-detection by searching StyleGAN's latent space for adversarial examples.
While earlier works address FRMs' vulnerability against semantic perturbations, a standard assessment of semantic robustness is still missing. 
Our work fills this gap in the literature, proposing a methodology to assess and characterize an FRM's semantic robustness by searching for identity-preserving examples that fool the FRM.
We search for such examples by modeling semantic (and interpretable) manipulations of facial attributes via direction- and magnitude-constrained perturbations in StyleGAN's latent space.

\vspace{2pt}\noindent\textbf{GANs and interpretation methods.}
The advent of GANs~\cite{goodfellow2014generative} propelled works on generating images of remarkable visual quality~\cite{brock2018large,karras2018progressive}.
The impressive perceptual quality achieved by GANs~\cite{karras2019style,karras2020analyzing} suggested that the representations learnt by these models inherently captured concepts of our visual world.
This observation stimulated research on interpreting the internal features learnt by GANs~\cite{bau2018gan} and the GANs' latent space~\cite{yang2019semantic}.
Recent works showed that this latent space not only encodes semantic concepts, but that such concepts can also be discovered~\cite{harkonen2020ganspace,Tzelepis_2021_ICCV,voynov2020unsupervised,shen2021closedform} and ``controlled''~\cite{shen2020interfacegan,shen2020interpreting}.
Our methodology leverages identity-preserving modifications by \textit{(i)}~building upon StyleGAN's capacity for generating human faces, and \textit{(ii)} controlling facial attributes in StyleGAN's latent space via InterFaceGAN~\cite{shen2020interfacegan,shen2020interpreting}.

\section{Semantic Adversarial Attacks}\label{sec:method}
Adversarial attacks usually fool a recognition model by imperceptibly modifying the pixels of an input image with an additive perturbation.
These attacks find such perturbation by searching for incorrectly-classified images within a set of imperceptible perturbations.
This set is often defined in pixel space as an $\ell_p$-ball with a small radius $\epsilon$, aiming at preserving the image's semantics.
Thus, these attacks leave both the image \textit{and} its semantics mostly unchanged.
While analyzing these perturbations is of interest, here we aim for a more practical class of perturbations that could fool FRMs in the real-world. 
Thus, in this work, we aim to assess the robustness of FRMs against semantic perturbations.

\subsection{Problem Formulation}
Let $f: \mathcal{I} \rightarrow \mathcal P(\mathcal Y)$ be an FRM that maps image~$I \in \mathcal{I}$ into the probability simplex over the set of identities $\mathcal{Y}$. 
Given an image $I$ of identity $y$, an attack aims at constructing $I^\star$, a perturbed version of $I$, considering two goals: \textit{(i)}~image similarity, \ie the distance between the two images $d_{\mathcal{I}}\left(I, I^\star\right)$ is small for some notion of $d_{\mathcal{I}}$, and \textit{(ii)}~fooling the FRM, \ie $I^\star$ is \textit{not} recognized as $y$ such that $\arg\max_i f^i\left(I^\star\right) \neq y$. 
These two goals may be misaligned, affecting the attack's formulation via constrained optimization.
In particular, formulations differ in whether the goal of similarity is used as a constraint---and so the fooling goal is the objective---or vice versa.
These two alternatives give rise to the paradigms of \textit{constrained-perturbation} and \textit{minimum-perturbation} attacks, respectively~\cite{dong2020benchmarking}.

In this work, we find identity-preserving modifications by proposing attacks from both paradigms that model image similarity via distances in StyleGAN's latent space.

\subsection{Identity-preserving Modifications}\label{sec:semantic_perturbs}
A StyleGAN model $G: \mathcal{W} \to \mathcal{I}$ generates images by mapping from latent space to image space. 
We consider a latent code $\mathbf{w} \in \mathcal{W} \subseteq \mathbb{R}^d$, which produces image $I = G(\mathbf{w})$.
We can generate $I^\star$, a perturbed variant of $I$, by injecting a perturbation $\pmb{\eta} \in \mathbb{R}^d$ on $\mathbf{w}$, that is $I^\star = G(\mathbf{w}^\star) = G(\mathbf{w} + \pmb{\eta})$.
However, we are not interested in introducing \textit{any} perturbation, but rather perturbations that produce identity-preserving modifications on $I$.

We remark two observations for these modifications: \textit{(i)}~InterFaceGAN~\cite{shen2020interfacegan} finds directions along which latent codes can be modified to inject semantically-viable modifications, \eg smile or pose directions, and \textit{(ii)} constrained modifications along these directions should not modify the image's identity.
Hence, we model identity-preserving modifications on $I$ by constraining $\pmb{\eta}$'s direction and magnitude.
We next describe how we model each constraint.

\vspace{2pt}\noindent\textbf{Direction constraints.}
InterFaceGAN provides a set of $N$ directions $\{\mathbf{v}_i\}_{i=1}^N$ in StyleGAN's latent space.
Each unit-norm vector $\mathbf{v}_i \in \mathbb{R}^d$ specifies a direction along which a semantic face attribute changes.
If these vectors are stacked into matrix $V \in \mathbb{R}^{N \times d}$, then constraining $\pmb{\eta}$'s direction amounts to constraining $\pmb{\eta}$ to lie in the subspace spanned by $V$'s rows.
We enforce this constraint by substituting $\pmb{\eta} = V^\top\pmb{\delta}$.
The substitution accomplishes our goal while changing the attack's search space from $\mathbb{R}^d \ni \pmb{\eta}$ to $\mathbb{R}^N \ni \pmb{\delta}$. 
This change in search space benefits the attack's efficiency, since most likely $N \ll d = 512$.
\underline{\textit{In practice,}} we derive $V$ by drawing upon the $N = 5$ interpretable directions provided by InterFaceGAN.
Thus, we build matrix $V$ from the directions corresponding to attributes: ``Pose'', ``Age'', ``Gender'', ``Smile'' and ``Eyeglasses''.

\vspace{2pt}\noindent\textbf{Magnitude constraints.}
Given how we enforce the direction constraints, we constrain $\pmb{\eta}$'s magnitude by constraining $\pmb{\delta}$'s magnitude.
While most works in robustness constrain with an $\ell_p$ norm, we argue this scheme is ill-suited for our purposes, since the scale in which semantic attributes vary may be incomparable across attributes. 
We thus introduce a symmetric and Positive-Definite (PD) matrix $M \in \mathbb{R}^{N \times N}$ to induce ``comparability'' across attributes.
Given this matrix, we model $\pmb{\delta}$'s magnitude as the norm induced by $M$.
Formally, we constrain $\sqrt{\pmb{\delta}^\top M\:\pmb{\delta}} = \|\pmb{\delta}\|_{M, 2} \leq 1$\footnote{This formulation still allows bounding $\|\pmb{\delta}\|_{M, 2}$ by any $\epsilon > 0$, as common in adversarial robustness, by redefining $M$ as $M \coloneqq \nicefrac{1}{\epsilon^2}\:M$.}
and so, the $\pmb{\eta}$'s magnitude is controlled solely by $M$.
\underline{\textit{In practice,}} we define $M$ by noting each entry of $\pmb{\delta} \in \mathbb{R}^N$ is associated with one direction from $\{\mathbf{v}_i\}_{i=1}^N$, in turn corresponding to a semantic attribute.
Defining $M$ is thus linked with the \textit{maximum} allowable perturbation along each individual $\mathbf{v}_i$.
Let the scalar $\epsilon_i$ define the maximum perturbation allowed along $\mathbf{v}_i$, then we have the condition $|\pmb{\delta}_i| \leq \epsilon_i$.
However, this condition still leaves $M$'s definition ill-posed.
We resolve this ambiguity by requiring $M$ to enclose the minimum volume possible.
With this requirement, we find that $M$ must be the diagonal matrix $M = \text{diag}(\epsilon_1^{-2},\:\dots, \:\epsilon_N^{-2})$.
We leave the details of this derivation to the \textbf{Appendix}.
\begin{figure}
    \centering
    \includegraphics[width=\columnwidth]{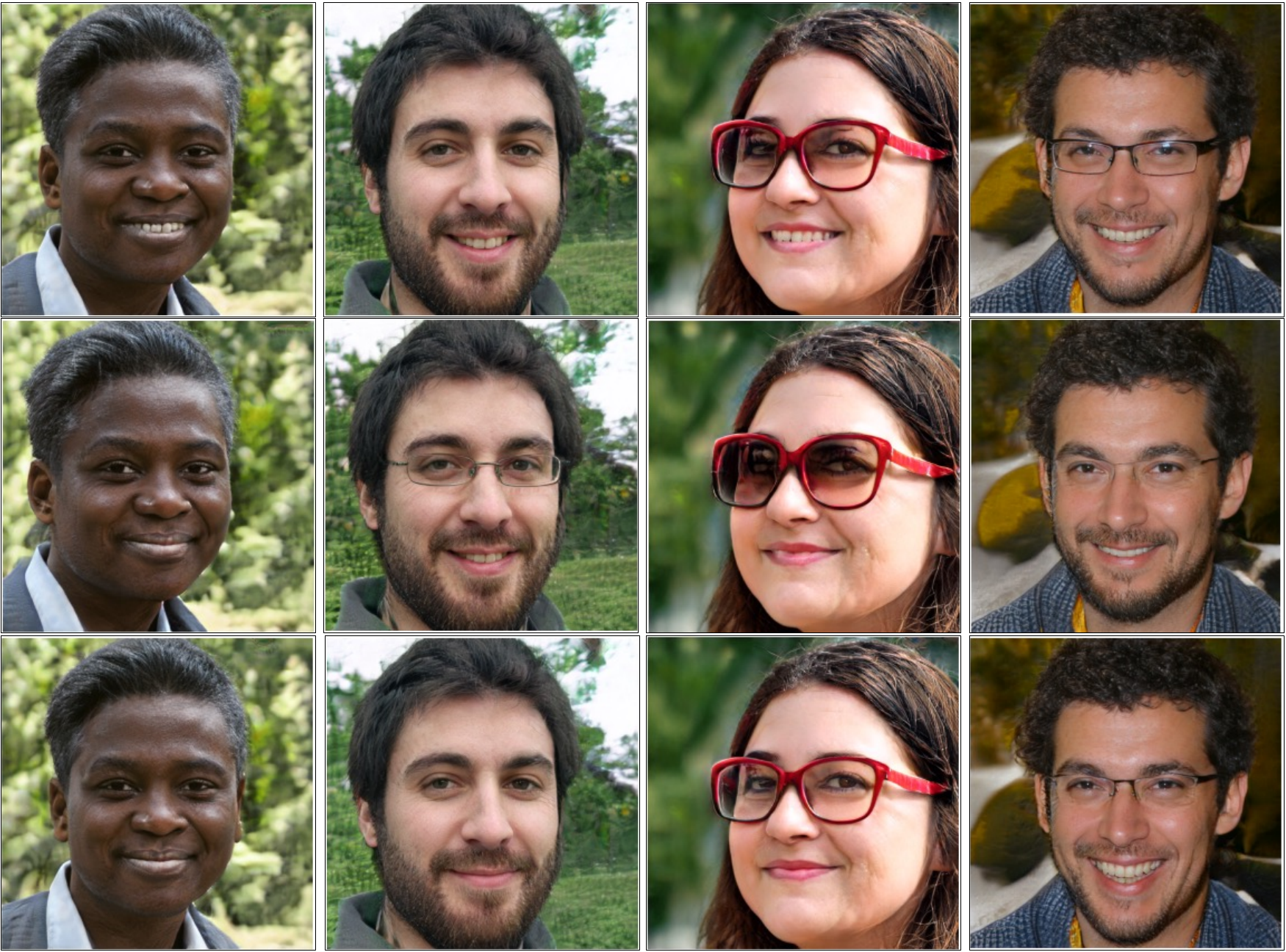}
    \vspace{-.6cm}
    \caption{\textbf{Identity-preserving modifications.} 
    The row column is the original image, and the other rows are random variants within the respective identity-preserving neighborhood.
    Notice simultaneous changes in pose and smile, while eyeglasses change color or appear/disappear.}
    \vspace{-.2cm}
    \label{fig:identity-preserving-perts}
\end{figure}
\vspace{-0.5cm}
\paragraph{\textit{Summary:}} With the direction and magnitude constraints, we define the set of identity-preserving modifications as $\mathcal{S}(V, M) = \{V^\top\pmb{\delta} \: : \: \|\pmb{\delta}\|_{M, 2} \leq 1\}$.
We show examples of these modifications in Figure~\ref{fig:identity-preserving-perts}.

\subsection{Constrained-perturbation Attacks}\label{sec:constrained-attacks}
Based on our formulation of identity-preserving modifications, we outline a constrained-perturbation attack under our framework. 
In particular, for the composition $F(\mathbf{w}) = f(G(\mathbf{w})) : \mathcal{W} \rightarrow \mathcal P(\mathcal{Y})$, an attack constructs an identity-preserving modification $\pmb\delta$ that fools the FRM $f$ by solving:
\begin{equation*}
    \max_{\pmb\delta}~ \mathcal{L}\left(F(\mathbf{w} + V^\top \pmb\delta),\, y\right) \quad \text{s.t.} \quad \|\pmb{\delta}\|_{M, 2} \leq 1,
\end{equation*}
where $\mathcal L$ is a suitable loss function between probability distributions.
This problem can be tackled with Projected Gradient Descent (PGD)~\cite{madry2018towards}, whose steps take the form:
\begin{equation*}
    \pmb{\delta}^{k+1} = \prod_{\|\pmb{\delta}\|_{M, 2}\leq1} \left(\pmb{\delta}^k + \alpha\: \nabla_{\pmb{\delta}}\mathcal L\left(F(\mathbf{w} + V^\top \pmb{\delta}),\, y \right) \Big|_{\pmb{\delta}=\pmb{\delta}^k}\right),
\end{equation*}
where $\alpha$ is the step size and $\prod$ is the projection operator.
While this formulation is similar to the classical PGD, we highlight a key difference: the set onto which updates are projected, that is $\|\pmb{\delta}\|_{M, 2}\leq1$, is no longer an isotropic $\ell_p$-ball, but rather an ellipsoid.
Hence, we derive next an efficient projection procedure on ellipsoids, which is critical for the computational tractability of our iterative attacks.

\paragraph{Projecting to an ellipsoid.} 
Formally, projecting a point $\pmb{\delta}$ to the region defined by $\|\pmb{\delta}\|_{M, 2} \leq 1$ is defined as solving
\begin{equation}\label{eq:project}
\underset{\pmb{\delta}^\star}{\arg\min}~ \frac12\:\left\|\pmb{\delta} - \pmb{\delta}^\star\right\|_2^2,\qquad\text{s.t. } ~~\pmb{\delta}^{\star\top} M\:\pmb{\delta}^\star\leq1.
\end{equation}
If $\pmb{\delta}$ is inside the ellipsoid, then $\pmb{\delta}^\star = \pmb{\delta}$.
Otherwise, we need to solve a variant of Problem~\eqref{eq:project}, where the inequality constraint is replaced by an equality, \ie search for $\pmb{\delta}^\star$ on the ellipsoid's surface. 
Problem~\eqref{eq:project} is convex in $\pmb\delta$ since $M$ is positive definite, so we find $\pmb{\delta}^\star$ with the Lagrangian:
\[
    L(\pmb{\delta}^\star, \lambda) = \frac12\:\|\pmb{\delta}^\star - \pmb{\delta}\|_2^2 + \lambda\:\left(\pmb{\delta}^{\star\top} M\:\pmb{\delta}^{\star\top} - 1\right).
\]
Deriving the KKT conditions yields:
\begin{equation}\label{eq:delta_star}
    \left(\pmb I + \lambda^\star\:M\right)\pmb{\delta}^\star = \pmb{\delta}^{\star},
\end{equation}
where $\pmb I$ is the identity and $\lambda^\star \in \mathbb{R}$ is the root of the function
\begin{equation*}
    h(\lambda) = \pmb{\delta}^{\top}\left(\pmb I + \lambda\:M\right)^{-1}M\left(\pmb I + \lambda\:M\right)^{-1}\pmb{\delta} - 1.
\end{equation*}

Thus, to find $\pmb{\delta}^\star$, we efficiently find $\lambda^\star$ via the bisection method, substitute into Eq.~\eqref{eq:delta_star}, and solve the linear system.

In practice, we define $M$ as a diagonal matrix (Section~\ref{sec:semantic_perturbs}).
This structure implies that $h$ can be evaluated without matrix multiplications nor inversions, and that Eq.~\eqref{eq:delta_star} is a diagonal system that can be efficiently solved.
Thus, our projection step is an inexpensive procedure that makes our attacks computationally tractable.

\subsection{Minimum-perturbation Attacks}
Analogous to constrained-perturbation attacks, we also outline a minimum-perturbation attack under our framework. 
In this paradigm, the attack aims to find the perturbation with the smallest magnitude that fools the FRM. 
Thus, based on our formulation of identity-preserving modifications, an attack that minimally modifies identity seeks to solve the following optimization problem:
\vspace{-0.05cm}
\begin{equation} \label{eq:minimum-norm-semantic-attack}
    \min_{\pmb{\delta}} ~\|\pmb{\delta}\|_{M, 2}  \quad \text{s.t.} \quad \underset{i}{\arg\max}~F^i(\mathbf{w} + V^\top\pmb\delta) \neq y.
\end{equation}
\vspace{-0.05cm}
We adopt the state-of-the-art FAB attack~\cite{croce2020minimally} to solve Problem~\eqref{eq:minimum-norm-semantic-attack}, as detailed in the \textbf{Appendix}.

\subsection{Interpreting Adversarial Examples}\label{sec:interpreting}
Once we attack and find adversarial perturbations, we are interested in interpreting them.
Each perturbation $\pmb{\delta} \in \mathbb{R}^N$ has an associated energy $\|\pmb{\delta}\|_{M, 2}$, and entry $\pmb{\delta}_i$ is related to modifying the $i^{\text{th}}$ attribute.
Hence, we discover trends in how an FRM weighs attributes to recognize faces by finding trends in how $\pmb{\delta}$'s energy is distributed among the attributes.

We thus propose to describe these trends via a ranking of the energy spent by $\pmb{\delta}$ on modifying each attribute.
Therefore, we first compose a candidate ranking by collecting ``votes'' from the $\pmb{\delta}$s that were found, and then validate the ranking by conducting statistical tests.

\vspace{2pt}\noindent\textbf{Composing a candidate ranking.}
The quantities being ranked must consider \textit{(i)} the likely anisotropy of the attribute space and \textit{(ii)} the energy of each perturbation found.
Thus, we consider the \textit{normalized} entries $\hat{\pmb{\delta}}_i = \nicefrac{\pmb{\delta}_i^2}{(\epsilon_i^2\:\|\pmb{\delta}\|_{M, 2})}$.
Based on these entries, each $\pmb{\delta}$ casts weighed votes, which we sort to find a ``winner'' attribute.
Each time a winner is found, we append the attribute to the ranking, and so we complete the ranking by iterating $(N-1)$ times.
We evaluate for significant differences among the remaining attributes with Friedman's test before deciding each winner.

\vspace{2pt}\noindent\textbf{Validating the ranking.}
Once we have a candidate ranking, we validate it with a statistical test.
In particular, we model a ranking of $N$ attributes as $(N-1)$ pair-wise comparisons of adjacent items in the ranking.
Thus, for each such pair of items we run a Wilcoxon signed-rank test.
Hence, for each ranking, we obtain $(N-1)$ \textit{p}-values testing for the local validity of the candidate ranking we propose.


\subsection{Certifying Against Semantic Perturbations}\label{sec:certification}
We also outline a certified robustness approach under our framework.
Hence, we consider composition $F$ from Section~\ref{sec:constrained-attacks} and adopt a certification formulation based on randomized smoothing.
In particular, we specialize the definition of domain-smoothed classifiers~\cite{deformrs} to anisotropically-smooth~\cite{ancer} semantic directions defined by matrix $V$.
\vspace{-0.095cm}
\begin{definition}\label{def:smooth-classifier}
Given a classifier $F(\mathbf{w}): \mathcal{W} \rightarrow \mathcal P(\mathcal{Y})$, we define a semantically-smoothed classifier as:
\[    
g(\mathbf{w}, \mathbf{p}) = \mathbb{E}_{\pmb{\epsilon}\sim \mathcal N(0, \Sigma)}\left[F\left(\mathbf{w} + V^T(\mathbf{p}+\pmb{\epsilon})\right) \right].
\]
\end{definition}

\vspace{-0.095cm}
In a nutshell, $g$'s prediction for the image generated from latent code $\mathbf{w}$ is the expected value of $F$'s predictions for semantic variants of the image, where such variants originate from perturbing $\mathbf{w}$. 
Moreover, $\mathbf{p}$ represents a canonical semantic perturbation of the original image.
The following proposition shows that our smooth classifier $g$ is certifiably robust against semantic perturbations along the directions defined by $V$.
We leave the proof for the \textbf{Appendix}.

\begin{proposition}\label{prop:semantic-perturbation-certification}
Let $g$ assign class $c_A$ for the input pair $(\mathbf{w}, \mathbf{p})$, \ie $\arg\max_c g^c(\mathbf{w}, \mathbf{p}) = c_A$ with:
\[
p_A = g^{c_A}(\mathbf{w}, \mathbf{p}) \quad\text{and}\quad p_B = \max_{c \neq c_A} g^c(\mathbf{w}, \mathbf{p})
\]
then $\arg\max_c~ g^c(\mathbf{w}, \mathbf{p}+\pmb\delta) = c_A \,\, \forall\:\pmb\delta$ such that:
\begin{equation}\label{eq:certification}
    \sqrt{\pmb{\delta}^T \Sigma^{-1} \pmb{\delta}} \leq \frac{1}{2}\left(\Phi^{-1}(p_A) - \Phi^{-1}(p_B) \right).
\end{equation}
\end{proposition}
Here, $\Sigma$ is the Gaussian covariance matrix and $\Phi$ is the Gaussian CDF. 
Proposition~\ref{prop:semantic-perturbation-certification} guarantees the smooth classifier's prediction will be constant for all perturbations within the ellipsoid defined by Eq.~\eqref{eq:certification}.
Note our result is not constrained to directions in a GAN's latent space: the smooth classifier is certifiable w.r.t. directions characterized by any matrix $V$.
When $\Sigma = \sigma \pmb I$, Eq.~\eqref{eq:certification} reduces to isotropic certification as introduced in Randomized Smoothing~\cite{cohen2019certified}; consequently, other choices of $\Sigma$ yield anisotropic certification.

\section{Experiments}
In this section, we assess and characterize the semantic robustness of off-the-shelf FRMs with our methodology.
We first study robustness under a constrained-perturbation attack, \ie PGD.
Then, we study robustness under a minimum-perturbation attack, \ie FAB.
Finally, we run isotropic and anisotropic certifications on the FRMs.

\subsection{Experimental details}
\vspace{2pt}\noindent\textbf{FRMs.} 
We target three renowned off-the-shelf FRMs: \textit{(i)} ArcFace~\cite{deng2019arcface}, \textit{(ii)} a FaceNet~\cite{schroff2015facenet} model trained on CASIA-Webface~\cite{yi2014learning} that we refer to as ``FaceNet$^C$'', and \textit{(iii)} a FaceNet model trained on VGGFace2~\cite{cao2018vggface2} that we refer to as ``FaceNet$^V$''. All models were retrieved from the public implementations \textit{InsightFace} and \textit{facenet-pytorch}.

\vspace{2pt}\noindent\textbf{Attributes' budget.} 
Table~\ref{tab:attribute-budgets} reports the budgets we assign to each attribute, \ie the $\epsilon_i$ defining the maximum extent to which latent codes can be perturbed in the direction of each attribute without changing the identity.
We establish these values by qualitatively and extensively exploring StyleGAN's output.
In particular, we set $\epsilon_i$ values that allowed StyleGAN to generate high-quality faces which, arguably, belong to the original identity.
Figure~\ref{fig:identity-preserving-perts} shows examples of faces following these $\epsilon_i$ attribute budgets.
\begin{table}[]
    \centering
    \caption{\textbf{Budget per attribute.} 
    We report the budget assigned for each attribute, \ie the maximum extent to which a latent code is allowed to vary in each direction while preserving identity.
    \label{tab:attribute-budgets}}
    \vspace{-0.25cm}
    \centering
    \begin{tabular}{ccccc}
    \toprule 
    \multicolumn{5}{c}{$\epsilon_i$ for attribute:} \\
    Pose        & Age       & Gender        & Smile         & Eyeglasses    \\ \hline 
    0.5         & 0.5       & 0.2           & 0.8           & 0.5          \\
    \bottomrule
    \end{tabular}
    \vspace{-0.4cm}
\end{table}

\vspace{2pt}\noindent\textbf{Attacks.}
Unless stated otherwise, we always experiment with a StyleGAN-generated dataset of \texttt{100k} identities (of comparable size to FFHQ~\cite{karras2019style}), from which we extract \texttt{5k} identities to attack.
We consider one image per identity.
\underline{\textit{PGD.}}
We use PGD with 10 iterations and 10 restarts.
\underline{\textit{FAB.}}
This attack has un-targeted and targeted versions.
FAB's un-targeted version is impractical, since its computational cost scales with the number of identities in the dataset.
Thus, in practice, we use FAB's \textit{targeted} version, and refer to it simply as ``FAB''.
We use FAB with 10 iterations, 10 restarts, and 10 target classes.
We ablate PGD's and FAB's hyper-parameters in the \textbf{Appendix}.

\vspace{2pt}\noindent\textbf{Certification.}
Randomized Smoothing (RS) uses Monte Carlo sampling and a statistical test on the predicted class probability.
We use $100$ and $10,000$ samples to determine $c_A$ and $p_a$, respectively, and a significance of $\alpha = 10^{-3}$ for the statistical test.
Due to the computational cost of RS, we follow common practice~\cite{cohen2019certified} and certify 500 identities.

\subsection{Attacks with PGD}\label{sec:pgd_attacks}
Attacking each FRM with PGD reveals the model's semantic robust accuracy, \ie the accuracy achieved by the model when under semantic attacks.
We find the following robust accuracies: 84.9 for ArcFace, 76.9 for FaceNet$^C$, and 71.0 for FaceNet$^V$. 
That is, PGD attacks suggest ArcFace is more robust than FaceNet$^C$, which is, in turn, more robust than FaceNet$^V$. 
We show some of the adversarial examples that fooled ArcFace in Figure~\ref{fig:pgd_results}. 
We note that subtle changes in smiling, pose and, most notably, eyeglasses, cause the FRM to malfunction.
Next, we take a closer look at the adversarial examples found by PGD by conducting the statistical procedure described in Section~\ref{sec:interpreting}.

\begin{figure*}
    \centering
    \includegraphics[width=\textwidth]{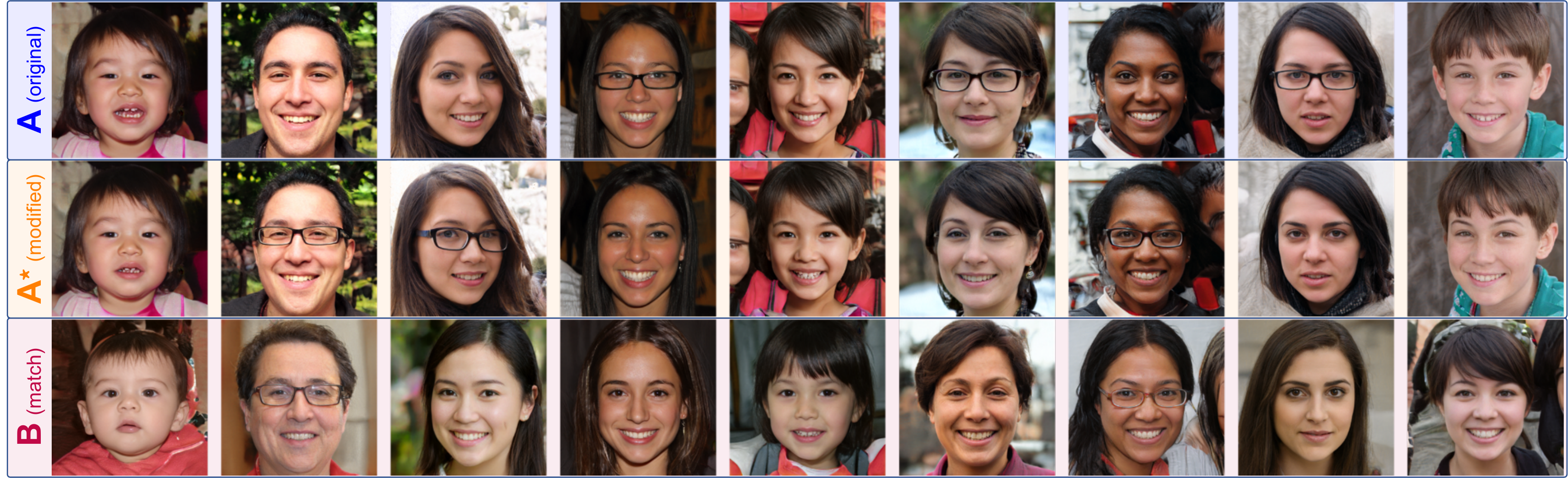}
    \vspace{-0.6cm}
    \caption{\textbf{Attacking Face Recognition Models (FRMs) via PGD.}
    Given face \textcolor{blue}{\fontfamily{cmss}\selectfont\textbf{A}} of an identity, we attack an FRM (ArcFace) and find an identity-preserving modified version \textcolor{orange}{\fontfamily{cmss}\selectfont\textbf{A$^\star$}}, such that the FRM prefers to match \textcolor{orange}{\fontfamily{cmss}\selectfont\textbf{A$^\star$}} with \textcolor{purple}{\fontfamily{cmss}\selectfont\textbf{B}} rather than with \textcolor{blue}{\fontfamily{cmss}\selectfont\textbf{A}}.
    }
    \vspace{-0.4cm}
    \label{fig:pgd_results}
\end{figure*}

\begin{table}[t]
    \centering
    \caption{\textbf{Ranking of PGD's per-attribute energy spent.}
    The rankings suggest each FRM's disproportionate sensitivity against modifications to an attribute (relative to other attributes).
    We denote statistically-significant comparisons with ``$>^\star$'', and the rest with ``$\geq$'' (significance of $0.01$).
    \label{tab:deltas_pgd}
    }
    \vspace{-.22cm}
    \begin{tabular}{c|C{0.1cm}cC{0.1cm}cC{0.1cm}cC{0.1cm}cC{0.1cm}}
        \toprule
        \multirow{2}{*}{Method} & \multicolumn{9}{c}{Ranking}  \\
                                & 1$^{\text{st}}$   &           & 2$^{\text{nd}}$   &           & 3$^{\text{rd}}$   &               & 4$^{\text{th}}$   &           & 5$^{\text{th}}$   \\ 
        \hline 
        ArcFace                 & E                 & $>^\star$ & P                 & $>^\star$ & A                 & $\,\,\,>^\star$     & S                 & $\,\,\,>^\star$ & G                 \\
        FaceNet$^C$             & E                 & $>^\star$ & A                 & $>^\star$ & P                 & $\geq$        & G                 & $\geq$    & S                 \\
        FaceNet$^V$             & E                 & $>^\star$ & A                 & $>^\star$ & P                 & $\geq$        & G                 & $\geq$    & S                 \\
        \bottomrule
        \multicolumn{10}{l}{\footnotesize{Convention: Eyeglasses (E), Pose (P), Age (A), Smile (S), Gender (G)}}
    \end{tabular}
    \vspace{-0.4cm}
\end{table}

\vspace{2pt}\noindent\textbf{Interpreting adversarial PGD examples.} 
We analyze how PGD spends its budget when constructing adversarial examples. 
Since PGD is a constrained-perturbation attack, we argue that the relative energy spent on modifying an attribute is related to \textit{``the FRM's disproportionate sensitivity to modifications on such attribute''}.
We characterize each FRM's semantic robustness by applying the procedure described in Section~\ref{sec:interpreting} on the semantic adversarial examples found by PGD, and report the ranking we obtain\footnote{We leave implementation details to the \textbf{Appendix}.} in Table~\ref{tab:deltas_pgd}.
We make two main observations about the extrema of the rankings, which hold for all FRMs: \textit{(i)} the ``Eyeglasses'' attribute leads the ranking in 1$^{\text{st}}$ position, and \textit{(ii)} the ``Smile'' and ``Gender'' attributes take the last two positions (4$^{\text{th}}$ and 5$^{\text{th}}$).
Next, we discuss these observations.

First, we find of high interest that statistical validation can suggest how the presence/absence of eyeglasses is a strong cue on which FRMs rely, somewhat disproportionately, to recognize faces. 
This finding can be related to previous works~\cite{geirhos2020shortcut,geirhos2018imagenettrained} that observe how DNNs learn ``shortcuts'' to solve tasks, thus hindering generalization.
Moreover, we note that reliance on eyeglasses is not strange to the human visual system: humans also have difficulty recognizing people when glasses are added/removed.
Additionally, our methodology's computation of the position in which eyeglasses rank may prove useful to improve the robustness of FRMs against addition and removal of eyeglasses.

Second, we observe that the smile and gender attributes fall last in the ranking.
Thus, compared to other attributes, \textit{neither} smile nor gender are attributes to which FRMs are disproportionately sensitive.
That is, under the attack's constrained budget, modifying either smile or gender is largely ineffective: altering either \textit{such that} the FRM is fooled would require an expense that exceeds the budget that was given to PGD.
This observation can be read as a pleasant finding: we do not find evidence that FRMs can be fooled by \textit{constrained} changes in smile nor gender.
Lastly, we leave more detailed discussion with a brute-force approach for characterizing semantic robustness to the \textbf{Appendix}.

\vspace{2pt}\noindent\textbf{Robustness \textit{vs.} dataset size.}
An FRM's chances of confusing individuals varies as the dataset size changes.
We thus experiment with this factor and vary the number of identities in the dataset from \texttt{5k} to \texttt{1M} and conduct PGD attacks on the same \texttt{5k} identities as before.
Figure~\ref{fig:all_pgd_exps}\textcolor{red}{a} reports the robust accuracies for each dataset size we considered.
As expected, the robust accuracies of all FRMs drop rapidly as the number of identities increases.
Specifically, performances drop from around 85\% when there are \texttt{5k} identities to around 70\% when there are \texttt{1M} identities.
Our experiments show that an FRM's semantic robustness largely depends on the number of identities it is required to recognize.
Hence, depending on the deployment setting, semantic robustness concerns may vary from negligible to problematic.

\vspace{2pt}\noindent\textbf{Attacking more identities.}
For computational feasibility, we considered a sample of \texttt{5k} out of the \texttt{100k} identities for our attacks.
Here, we test whether this set of identities is a representative sample of the population.
We thus fix the \texttt{100k} identities in the dataset and vary the amount of samples we attack from \texttt{1k} to \texttt{20k} and report the results in Figure~\ref{fig:all_pgd_exps}\textcolor{red}{b}.
We observe that there is virtually no variation in the semantic robustness of any FRM.
These results suggest that our design choice of experimenting with \texttt{5k} samples provides a reasonable sample of the population for assessing the semantic adversarial robustness of FRMs.

\vspace{2pt}\noindent\textbf{Perturbation budget.}
In previous experiments, we searched for adversarial examples within the set of identity-preserving modifications by constraining $\| \pmb{\delta}\|_{M, 2} \leq \epsilon = 1$.
Since our analysis relied on an empirical estimate of the identity-preserving region (\ie $M$), this region might not be the tightest. 
Thus, we test how FRMs behave when this constraint is relaxed/tightened by varying $\epsilon$ from $\nicefrac{1}{4}$ to $8$.
We report results in Figure~\ref{fig:all_pgd_exps}\textcolor{red}{c}.
As expected, the robustness of all FRMs drops rapidly when the semantic perturbation budget increases: ArcFace: $98.3\to4.1$, FaceNet$^C$: $95.4\to12.6$, and FaceNet$^V$: $93.3\to7.3$. 
It is worthwhile to note that allowing semantic perturbation budgets of $\epsilon > 1$ could lead to changing the generated face's identity.

\subsection{FAB attack}\label{sec:fab_attacks}
We also assess each FRM's semantic robustness with FAB attacks.
FAB searches over the subspace of semantic attributes, however, FAB does not guarantee that the adversarial examples it finds fall in the identity-preserving neighborhood.
That is, while FAB may successfully find adversarial examples for \textit{all} the instances it attacks, a human observer may no longer judge the discovered examples as belonging to the same identity.

\begin{figure*}
    \centering
    \includegraphics[width=\textwidth]{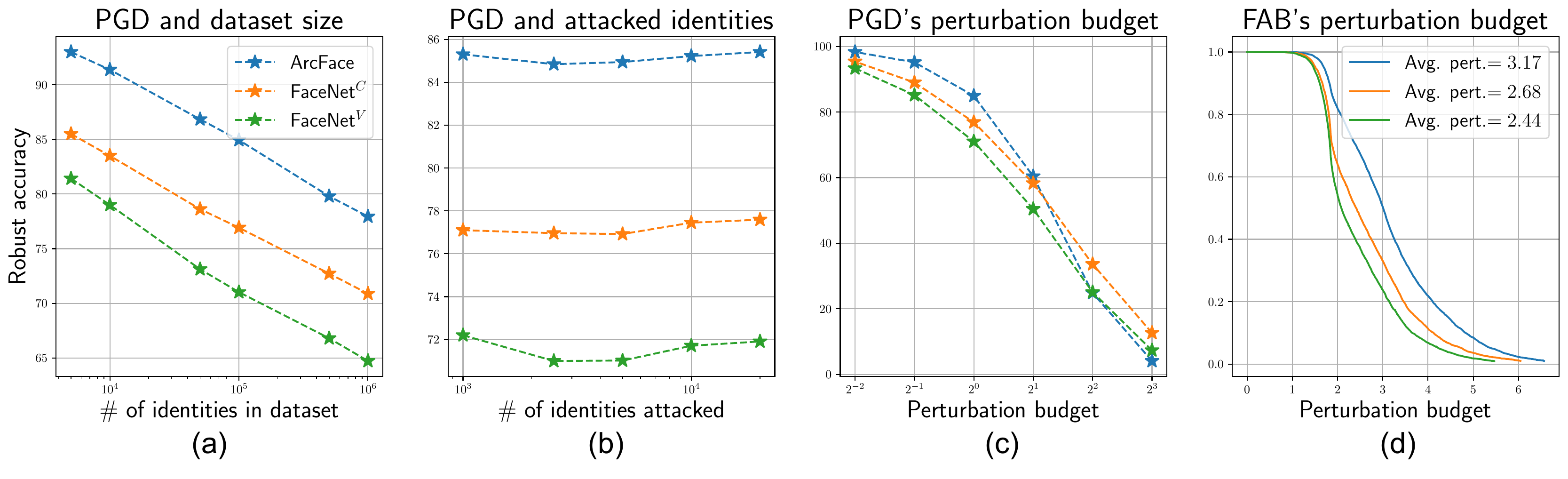}
    \vspace{-0.7cm}
    \caption{\textbf{Assessing semantic robustness via attacks.} 
    We use PGD and FAB to assess the semantic robustness of three Face Recognition Models.
    For PGD, we report how robustness varies with the number of (a) identities in the dataset and (b) identities attacked.
    For PGD (c) and FAB (d), we show how robustness changes w.r.t. perturbation budget. 
    }
    \label{fig:all_pgd_exps}
    \vspace{-0.4cm}
\end{figure*}

We run FAB on each FRM, and find semantic adversarial examples for all the \texttt{5k} images we attack.
The latent code $\mathbf{w}^\star = \mathbf{w} + V^\top\pmb{\delta}$ of each adversarial example found has a perturbation budget $\|\pmb{\delta}\|_{M, 2}$.
FAB finds few adversarial examples with $\| \pmb{\delta}\|_{M, 2} \leq 1$, that is, within the identity-preserving neighborhood; in particular: $3$ for ArcFace, $10$ for FaceNet$^C$ and $13$ for FaceNet$^V$.
Given the uncertainty on $M$'s tightness (due to its empirical estimation), and following common practice in robustness~\cite{dong2020benchmarking},
We plot accuracy \textit{vs.} perturbation budget curves for all FRMs in Figure~\ref{fig:all_pgd_exps}\textcolor{red}{d}.
Adversarial robustness is judged by how rapidly each curve drops as the perturbation budget increases.
Thus, FAB's assessment suggests ArcFace is more robust than FaceNet$^C$, which is more robust than FaceNet$^V$, agreeing with PGD's ranking. 
We leave the interpretation of FAB's adversarial examples (via the procedure from Section~\ref{sec:interpreting} and a brute-force approach) to the \textbf{Appendix}.

\subsection{FRM Certification}
\begin{figure}
    \centering
    \includegraphics[width=1\columnwidth]{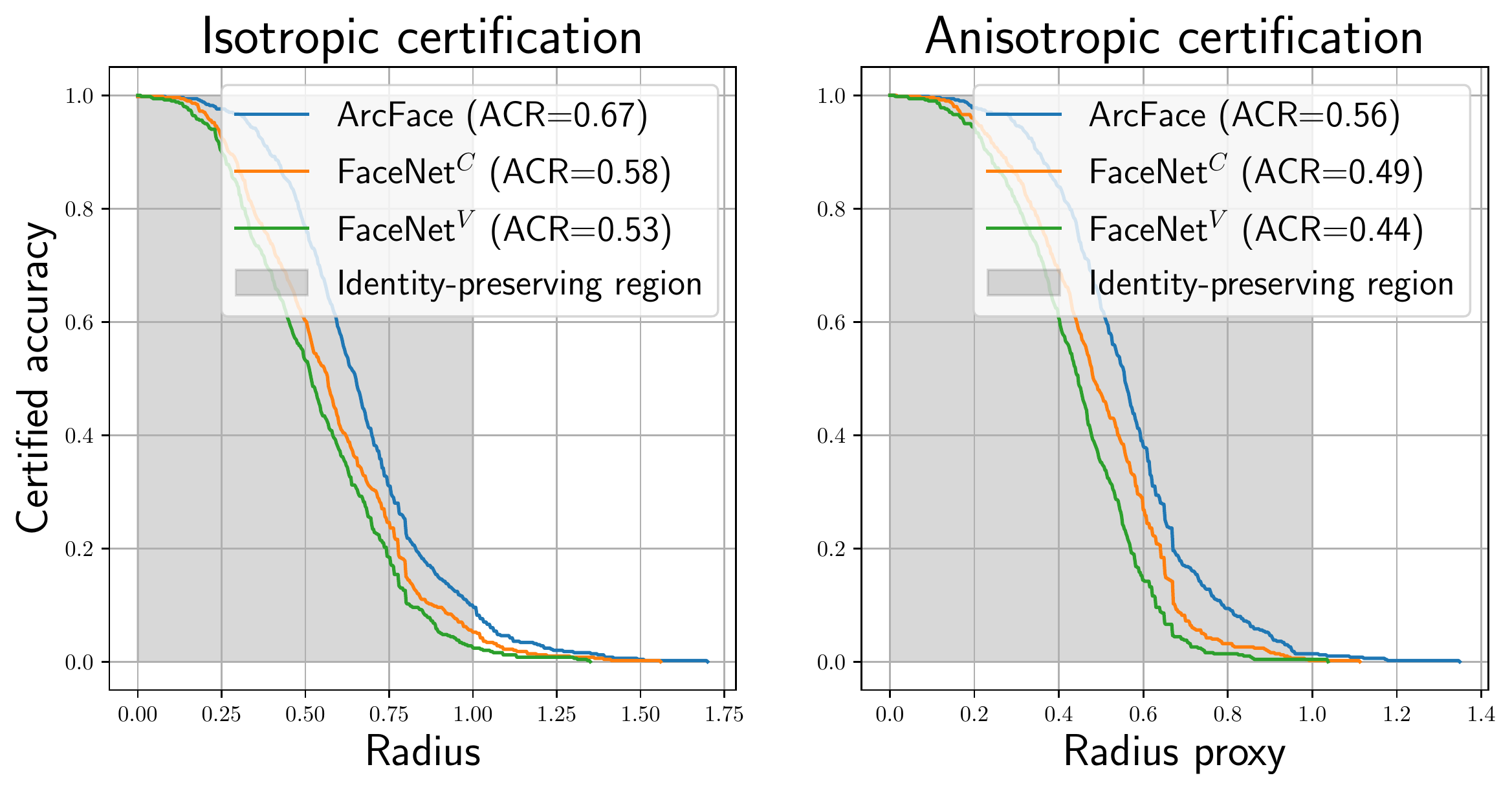}
    \vspace{-0.6cm}
    \caption{\textbf{Certifying Face Recognition Models (FRMs) via Randomized Smoothing.} 
    Envelope curves of all FRMs both for isotropic (left) and anisotropic (right) certification.
    }
    \vspace{-0.4cm}
    \label{fig:certification}
\end{figure}

\vspace{2pt}\noindent\textbf{Isotropic certification.} 
Following the methodology introduced in Section~\ref{sec:certification}, we certify all FRMs with covariance $\Sigma = \sigma^2\pmb I$, and set $\sigma \in \{0.1, 0.25, 0.5, 0.75, 1\}$. 
In this setup, the certified region in Proposition~\ref{prop:semantic-perturbation-certification} is a ball with radius $\|\pmb\delta\|_2 \leq \frac{\sigma}{2}\left(\Phi^{-1}(p_A) - \Phi^{-1}(p_B)\right):=R$, as derived in~\cite{deformrs}.
We denote this quantity as the \textit{certified radius}. 

\vspace{2pt}\noindent\textbf{Anisotropic certification.}
Following our consideration of anisotropic regions for preserving identity, we explore anisotropic certification by drawing upon recent work~\cite{ancer} that extends RS to anisotropic settings.
Thus, we require a sensible candidate for an anisotropic $\Sigma$ in Proposition~\ref{prop:semantic-perturbation-certification}, encoding \textit{a priori} knowledge on the subspace of semantic attributes.
Hence, we set $\Sigma = M^{-1}$, where $M$ is the matrix encoding the magnitude constraints in our approach.
The rationale behind this choice is that the Mahalanobis distance to the distribution $\mathcal{N}(0,M^{-1})$ draws precisely the ellipsoid described by $M$.
For the experiments, we consider $\Sigma = \sigma M^{-1}$ and set $\sigma \in \{0.25, 0.5, 0.75, 1, 2, 2.5\}$.
Since anisotropic regions lack a notion of radius, we follow~\cite{ancer} and compute a \textit{radius proxy}: the radius of a ball whose volume is equivalent to that of the certified region.

\vspace{2pt}\noindent\textbf{Results.} 
We compute the best certificates for each FRM across all $\sigma$ values, and report certified accuracy curves in Figure~\ref{fig:certification} for isotropic (left) and anisotropic (right) certification.
Each point $(x, y)$ in a curve implies that percentage $y\%$ of the dataset is both predicted correctly and has a certified radius of at least $x$. 
Moreover, we adopt common practice~\cite{zhai2020macer} and report the Average Certified Radius (ACR) for each FRM.
We draw the following observations: 
\textit{(i)} The certified accuracy within the identity-preserving region is remarkably low.
That is, while all FRMs displayed substantial robustness against our attacks, certification demonstrates these models \textit{can} be fooled by stronger attacks. 
Hence, we find FRMs are also extremely vulnerable to simple semantic perturbations. 
We argue this vulnerability is expected, as regular DNN training is not designed to resist against adversarial attacks.
\textit{(ii)} The ACRs under the anisotropic setting are smaller than those under the isotropic one. 
This can be a result of a sub-optimal choice of the matrix $\Sigma$.

\section{Conclusions}
We propose a methodology for assessing and characterizing the semantic robustness of Face Recognition Models (FRMs).
Our methodology induces malfunction in FRMs by conducting direction- and magnitude-constrained search in StyleGAN's latent space, such that faces are modified but their identity is preserved.
Under this framework, we attack FRMs, find adversarial examples, and then characterize the semantic robustness of FRMs by statistically describing the examples that lead them to fail.
Finally, we demonstrate how our methodology can leverage a certification technique, allowing us to construct a formal description of what an FRM may conceive as a face's identity.

\section{Limitations}
The main focus of our study is the semantic robustness of a standalone FRM.
However, in practice, we are unable to directly study the FRM, as we introduce a StyleGAN \textit{before} the FRM.
We model semantic directions in StyleGAN's latent space via InterFaceGAN.
Thus, the conclusions we reach are limited by the weaknesses of StyleGAN and InterFaceGAN.
Specifically, we underscore the following weaknesses: \textit{(i)} there are no guarantees for StyleGAN's output, while impressive, to be clean of artifacts, \textit{(ii)} StyleGAN's training data is presumably biased, thus affecting the diversity of generated faces, and \textit{(iii)} the semantic directions found by InterFaceGAN still display some entanglement.


{\small
\bibliographystyle{ieee_fullname}
\bibliography{references}
}

\newpage
\appendix

\onecolumn
{\centering
\Large
\textbf{Towards Assessing and Characterizing the Semantic Robustness of Face Recognition} \\
\vspace{0.5em}Appendix \\
\vspace{1.0em}
}
\appendix

\section{Magnitude Constraints}
Matrix $M$ characterizes the ellipsoid of valid perturbations $\pmb{\delta}$ by controlling the perturbation's norm via $\|\pmb{\delta}\|_{M, 2} \leq 1$.
Each entry of $\pmb{\delta} \in \mathbb{R}^N$ is associated with one of the $N$ directions $\{\mathbf{v}_i\}_{i=1}^N$, which, in turn, corresponds to a semantic attribute.
Defining $M$'s entries, thus, can be associated with the amount of perturbation we allow for each of $\pmb{\delta}$'s entries. 
Our definition of $M$ is thus based on the \textit{maximum} amount of perturbation allowed along each individual direction $\mathbf{v}_i$.
In particular, let the scalar $\epsilon_i$ state the maximum extent to which latent codes can be perturbed in the direction of $\mathbf{v}_i$.
This $\epsilon_i$ is ``maximal'' in the sense that semantically-invalid (and, more importantly, identity-changing) modifications would be introduced if the perturbation was larger than $\epsilon_i$.
Thus, we require $M$ to account for the fact that the magnitude of each entry of $\pmb{\delta}$ must not exceed its corresponding $\epsilon$ or, formally, $|\pmb{\delta}_i| \leq \epsilon_i$.

Furthermore, since $M$ must, by construction, define an ellipsoid, we relate the above requirement with a geometrical reasoning.
Specifically, we note that our requirements for $M$ imply that the associated ellipse must pass through the $N$ points $\mathbf{p}_1 = [\epsilon_1, 0,\:\dots,\:0]^\top$, $\mathbf{p}_2 = [0, \:\epsilon_2, \:\dots,\:0]^\top$, \dots, $\mathbf{p}_N = [0, \:\dots,\:\epsilon_N]^\top$, where $\mathbf{p}_i \in \mathbb{R}^N$.
To account for the absolute value in our requirement, the ellipse must also pass through $-\mathbf{p}_1$, $-\mathbf{p}_2$, \dots, $-\mathbf{p}_N$.
Thus, this $N$-dimensional ellipse must pass through the set of $2N$ points $\mathcal{P} = \{\mathbf{p}_i\}_{i=1}^N \cup \{-\mathbf{p}_i\}_{i=1}^N$, which are a function of $\{\epsilon_i\}_{i=1}^N$.

However, this constraint is insufficient to uniquely define the ellipse: there is an infinite set of ellipses that comply with the constraint.
We note that our requirement of $|\pmb{\delta}_i| \leq \epsilon_i$ also implies that the ellipse must not simply pass through the points in $\mathcal{P}$, but also that these points should be the ellipse's extrema.
Hence, $M$ is defined as a matrix parameterizing the ellipse whose extrema are the points in $\mathcal{P}$.
Equivalently, we are searching for the $M$ that parameterizes an ellipse that passes through these points and encloses the minimum volume, which is equivalent to the $M$ whose determinant is maximal~\cite{boyd2004convex}.

Given a set of points, such $M$ can be found with the algorithm introduced in~\cite{moshtagh2005minimum}.
However, we note that the points in $\mathcal{P}$ enjoy properties that can simplify the search for $M$.
In particular, the set of points $\mathcal{P}$ is composed of scaled versions of the canonical basis for $\mathbb{R}^N$.
This fact implies that the ellipse we are searching for has its semi-axis aligned with the canonical axis of $\mathbb{R}^N$.
Thus, we can simply define $M$ as a diagonal matrix whose elements are the reciprocal of the squares of the $\epsilon$, that is $M = \text{diag}(\epsilon_1^{-2},\:\dots, \:\epsilon_N^{-2})$.

\paragraph{Impact of $M$ on computation.}
Here underscore how we can exploit our definition of matrix $M$ to reduce computation in our methodology.
Note that $M$ is used in the procedure for projecting perturbations to the ellipsoid.
In particular, as stated in Section~\ref{sec:constrained-attacks}, projections are performed by first solving for $h$'s root, \ie $\lambda^\star$, and then using such solution to solve for $\pmb{\delta}^\star$ in Equation~\eqref{eq:delta_star}.
The definition of function $h$ involves $M$ for one inversion and two matrix multiplications.
However, since we defined $M$ to be a diagonal matrix, we can simplify the definition of $h$ to
\[
    h(\lambda) = \sum_{i=1}^N \frac{\pmb{\delta}_i^2\epsilon_i^{-2}}{\left(1 + \lambda\epsilon_i^{-2}\right)^2} - 1,
\]
which does not require matrix multiplication nor inversion.
Moreover, the system in Equation~\eqref{eq:delta_star} can be efficiently solved via inversion since the associated matrix is diagonal.
Thus, in our implementation, we exploit these observations to lower the computational expense of projecting the perturbations at each iteration of PGD.

Similarly, the diagonal property of $M$ is exploited across our implementation when $M$ is the matrix involved in a bilinear form.
Such implementation trick reduces the operations associated with these matrix multiplications.

\section{Adversary Initialization}\label{sec:adv_init}
Previous works on adversarial robustness have shown that the initialization of the adversarial perturbation $\pmb{\delta}$ can have sizable impact in the performance of the attack~\cite{Wong2020Fast}.
Interestingly, Wong \textit{et al.}~\cite{Wong2020Fast} found that initializing the perturbation to random noise within the space of perturbations dramatically increases performance.
By analogy, together with our methodology, we report the procedure for randomly initializing the perturbation, with the modifications induced by our formalization of the space of perturbations.

In particular, we recall that our methodology models the space of perturbations as the volume enclosed by the ellipsoid parameterized by the matrix $M$.
Thus, we are interested in randomly initializing $\pmb{\delta}$ uniformly within such volume.
To achieve this, we first initialize $\pmb{\delta}$ uniformly within the unit ball and then deform the ball to the ellipsoid parameterized by $M$ by performing Cholesky decomposition.

\section{FAB attack Implementation}

\textbf{Main algorithmic modifications.}
FAB's algorithm aims at minimizing the $\ell_p$ norm of a perturbation $\pmb{\delta}$.
There are versions of FAB for $p \in \{1,2,\infty\}$.
We extend FAB to our methodology by modifying the algorithm to consider the norm induced by our $M$ matrix, \ie $\|\pmb{\delta}\|_{M, 2}$.

Such modification amounts to adjusting the $p = 2$ version of the attack for our purposes.
Specifically, we \textit{(i)} replace inner products in the $p$ norm, \ie $\mathbf{u}^\top\mathbf{v}$, by $\mathbf{u}^\top M\mathbf{v}$, \textit{(ii)} replace inner products in the dual norm $q$ by $\mathbf{u}^\top M^{-1}\mathbf{v}$ (where $\nicefrac{1}{p} + \nicefrac{1}{q} = 1$), and \textit{(iii)} re-implement the projection operations required inside FAB's algorithm.
Further, for initializing the adversarial perturbation, we perform random initialization by uniform sampling inside the ellipsoid parameterized by $M$, through the procedure we report in Appendix~\ref{sec:adv_init}.

The code we provide also includes these implementation modifications to FAB.

\section{Proof of Proposition \ref{prop:semantic-perturbation-certification}}
\begin{prop}(restatement)
Let $g$ assign class $c_A$ for the input pair $(\mathbf{w}, \mathbf{p})$, \ie $\arg\max_c g^c(\mathbf{w}, \mathbf{p}) = c_A$ with:
\[
p_A = g^{c_A}(\mathbf{w}, \mathbf{p}) \quad\text{and}\quad p_B = \max_{c \neq c_A} g^c(\mathbf{w}, \mathbf{p})
\]
then $\arg\max_c~ g^c(\mathbf{w}, \mathbf{p}+\pmb\delta) = c_A \,\, \forall\:\pmb\delta$ such that:
\begin{equation}\label{eq:certification}
    \sqrt{\pmb{\delta}^T \Sigma^{-1} \pmb{\delta}} \leq \frac{1}{2}\left(\Phi^{-1}(p_A) - \Phi^{-1}(p_B) \right).
\end{equation}
\end{prop}

\begin{proof}
 This follows from combining Corollary 1 from \cite{ancer} with the certification of domain smooth classifier in Theorem 1 in \cite{deformrs} with setting $\mathcal{D} = \mathcal N(0, \Sigma)$.
\end{proof}

We note that one can alternatively define the semantically-smooth classifier in the following way.
\begin{definition}
Given a classifier $F(\mathbf{w}): \mathbb{R}^d \rightarrow \mathcal P(\mathcal{Y})$ and the semantic direction matrix $V \in \mathbb{R}^{d\times N}$, we define a semantically-smoothed classifier as:
\[    
\hat g(\mathbf{w}) = \mathbb{E}_{\pmb{\epsilon}\sim \mathcal N(0, \Sigma)}\left[F\left(\mathbf{w} + V^T\pmb{\epsilon}\right) \right].
\]
\end{definition}
Thus, one can deploy Corollary 1 from \cite{ancer} directly to obtain the following equivalent result to proposition \ref{prop:semantic-perturbation-certification} for when $V$ is a full rank and invertible matrix.
\begin{prop}
Let $\hat g$ assign class $c_A$ for the input  $\mathbf{w}$, \ie $\arg\max_c \hat g^c(\mathbf{w}) = c_A$ with:
\[
p_A = \hat g^{c_A}(\mathbf{w}) \quad\text{and}\quad p_B = \max_{c \neq c_A} \hat g^c(\mathbf{w})
\]
then $\arg\max_c~ \hat g^c(\mathbf{w} +  V^\top\pmb\delta) = c_A \,\, \forall\:\pmb\delta$ such that:
\begin{equation}\label{eq:certification}
    \sqrt{\pmb{\delta}^T \Sigma^{-1} \pmb{\delta}} \leq \frac{1}{2}\left(\Phi^{-1}(p_A) - \Phi^{-1}(p_B) \right).
\end{equation}
\end{prop}
\begin{proof}
 Note that the smooth classifier in Definition 2 is equivalent to the following:
 \[ \hat g(\mathbf{w}) = \mathbb{E}_{\pmb{\epsilon}\sim \mathcal N(0, V^\top \Sigma V)}\left[F\left(\mathbf{w} + \pmb{\epsilon}\right) \right].
 \]
 Therefore, and based on Corollary 1 in \cite{ancer}, we have $\arg\max_c~ \hat g^c(\mathbf{w} + \pmb \eta) = c_A$ with
 \[ \sqrt{\pmb{\eta}^T V^{-1}\Sigma^{-1}V^{-\top} \pmb{\eta}} \leq \frac{1}{2}\left(\Phi^{-1}(p_A) - \Phi^{-1}(p_B) \right)
 \]
 Therefore, rewriting $\pmb \delta = V^{-\top} \pmb{\eta}$ results in $\pmb \eta = V^\top\pmb \delta$, completing the proof.
\end{proof}

\section{Ablations on PGD and FAB}
Our experiments used both PGD and FAB attacks.
The PGD attack has two hyper-parameters: \textit{(i)} the number of iterations and \textit{(ii)} the number of restarts.
The FAB attack has three hyper-parameters: \textit{(i)} the number of iterations, \textit{(ii)} the number of restarts, and \textit{(iii)} the number of target classes.
For all these hyper-parameters, we run a small grid search to showcase how an FRM's robustness behaves when these parameters vary.
The grid search consists of varying each parameter in $\{1,5,10,20\}$.
We select the most robust model we studied (\ie ArcFace, according to our experiments) as the target FRM for these experiments, conduct each attack, and record the robust accuracy of the FRM.
Following the main paper's results, these attacks are run on a dataset that has \texttt{100k} identities (\ie faces), from which \texttt{5k} identities are attacked.

For PGD, Table~\ref{tab:ablate_pgd} reports the robust accuracies we find.
We observe that both the number of iterations and the number of restarts have an impact on the robust accuracy.
In particular, increasing either parameter affects PGD's success, \ie the FRM's robust accuracy is lower.
The robust accuracies we find range from $99.2\%$ (a single restart and a single iteration) to $83.0\%$ (20 restarts and 20 iterations).
We speculate that a more extensive experimentation with optimization hyper-parameters may yield better success rates for the attack.

FAB can find adversarial examples for $100\%$ of the attacked identities.
Thus, for FAB's ablations we do not focus on the FRM's robust accuracy, but rather on the energy of the adversarial perturbation that FAB finds, \ie $\|\pmb{\delta}\|_{M, 2}$.
Table~\ref{tab:ablate_fab} reports the average energies when varying FAB's restarts and iterations hyper-parameters (where the number of target classes was set to $5$).
Analogously, Table~\ref{tab:ablate_fab_targetclasses} reports these values for FAB's restarts and target classes hyper-parameters (where the number of iterations was set to $5$).
Within these tables, we see that the average energy can range from very large ($22.1$ in Table~\ref{tab:ablate_fab_targetclasses} when using 1 restart, 1 target class and 5 iterations), to very small ($2.6$ also in Table~\ref{tab:ablate_fab_targetclasses} when using 20 restarts, 20 target classes and 5 iterations).
We also observe that increasing iterations has a rather marginal impact on reducing the average energy of the perturbation.
This phenomenon most likely implies that the steps conducted by our modified version of FAB are sub-optimal.
Hence, a reformulation of FAB that accounts for the idiosyncratic properties of FRMs may yield valuable gains for FAB's success rate.

\begin{table}[]
\centering
\caption{\textbf{Ablating PGD's hyper-parameters.} We search over the number of restarts and iterations given to PGD, and record ArcFace's robust accuracy.
\label{tab:ablate_pgd}}
\centering
\begin{tabular}{cc|cccc}
\toprule
\multicolumn{1}{l}{}        & \multicolumn{1}{l|}{} & \multicolumn{4}{c}{Restarts} \\
\multicolumn{1}{l}{}                                                    &       & 1     & 5     & 10    & 20        \\ \hline
\parbox[t]{2mm}{\multirow{4}{*}{\rotatebox[origin=c]{90}{Iterations}}}  & 1      & 99.2	& 96.8	& 95.0	& 93.2  \\
                                                                        & 5      & 95.0	& 87.8	& 85.8	& 84.3 \\
                                                                        & 10      & 94.1	& 86.1	& 84.5	& 83.4 \\
                                                                        & 20      & 93.5	& 85.2	& 83.8	& \textbf{83.0}\\\bottomrule
                                                                        
\end{tabular}
\end{table}

\begin{table}[]
\centering
\caption{\textbf{Ablating FAB's \textit{restarts} and \textit{iterations} hyper-parameters.}
We search over the number of restarts and iterations, and record the energy $\|\pmb{\delta}\|_{M, 2}$ of the adversarial perturbations found by FAB.
\label{tab:ablate_fab}}
\centering
\begin{tabular}{cc|cccc}
\toprule
\multicolumn{1}{l}{}        & \multicolumn{1}{l|}{} & \multicolumn{4}{c}{Restarts} \\
\multicolumn{1}{l}{}                                                    &       & 1     & 5     & 10    & 20        \\ \hline
\parbox[t]{2mm}{\multirow{4}{*}{\rotatebox[origin=c]{90}{Iterations}}}  & 1 & 8.9	& 4.7	& 4.0	& 3.4 \\
                                                                        & 5 & 7.7	& 4.2	& 3.6	& \textbf{3.2} \\
                                                                        & 10 & 7.7	& 4.3	& 3.6	& \textbf{3.2} \\
                                                                        & 20 & 7.7	& 4.3	& 3.6	& \textbf{3.2} \\\bottomrule
\end{tabular}
\end{table}

\begin{table}[H]
\centering
\caption{\textbf{Ablating FAB's \textit{restarts} and \textit{target classes} hyper-parameters.}
We search over the number of restarts and iterations, and record the energy $\|\pmb{\delta}\|_{M, 2}$ of the adversarial perturbations found by FAB.
\label{tab:ablate_fab_targetclasses}}
\centering
\begin{tabular}{cc|cccc}
\toprule
\multicolumn{1}{l}{}        & \multicolumn{1}{l|}{} & \multicolumn{4}{c}{Restarts} \\
\multicolumn{1}{l}{}                                                        &       & 1     & 5     & 10    & 20        \\ \hline
\parbox[t]{2mm}{\multirow{4}{*}{\rotatebox[origin=c]{90}{Targ. classes}}}  & 1 & 22.1	& 7.6	& 5.6	& 4.5\\
                                                                            & 5 & 7.7	& 4.2	& 3.6	& 3.2\\
                                                                            & 10 & 5.7	& 3.6	& 3.2	& 2.8\\
                                                                            & 20 & 4.6	& 3.2	& 2.8	& \textbf{2.6}\\\bottomrule
\end{tabular}
\end{table}

\section{PGD Qualitative Results}
We report randomly-selected qualitative samples for the adversarial examples found by PGD, when attacking ArcFace (the most robust method, according to our assessment), in Figures~\ref{fig:pgd_qual1}--\ref{fig:pgd_qual4}.
While most of these examples elucidate the identity-preserving property we aim for in our methodology, some examples are failure cases.
From such failures, we can identify two common patterns: \textit{(i)} when the GAN introduced artifacts in the original face (\eg the very first sample in Figure~\ref{fig:pgd_qual1}), and \textit{(ii)} when there are faces of children involved.

Furthermore, these samples clearly show how the eyeglasses attribute is an attribute that, upon modification, is highly effective for fooling FRMs.
In particular, \textit{many} of the samples we show here display the addition/removal of eyeglasses.

\twocolumn
\begin{figure}
    \centering
    \includegraphics[trim=0cm 0.3cm 0cm 0.7cm,clip,width=\columnwidth]{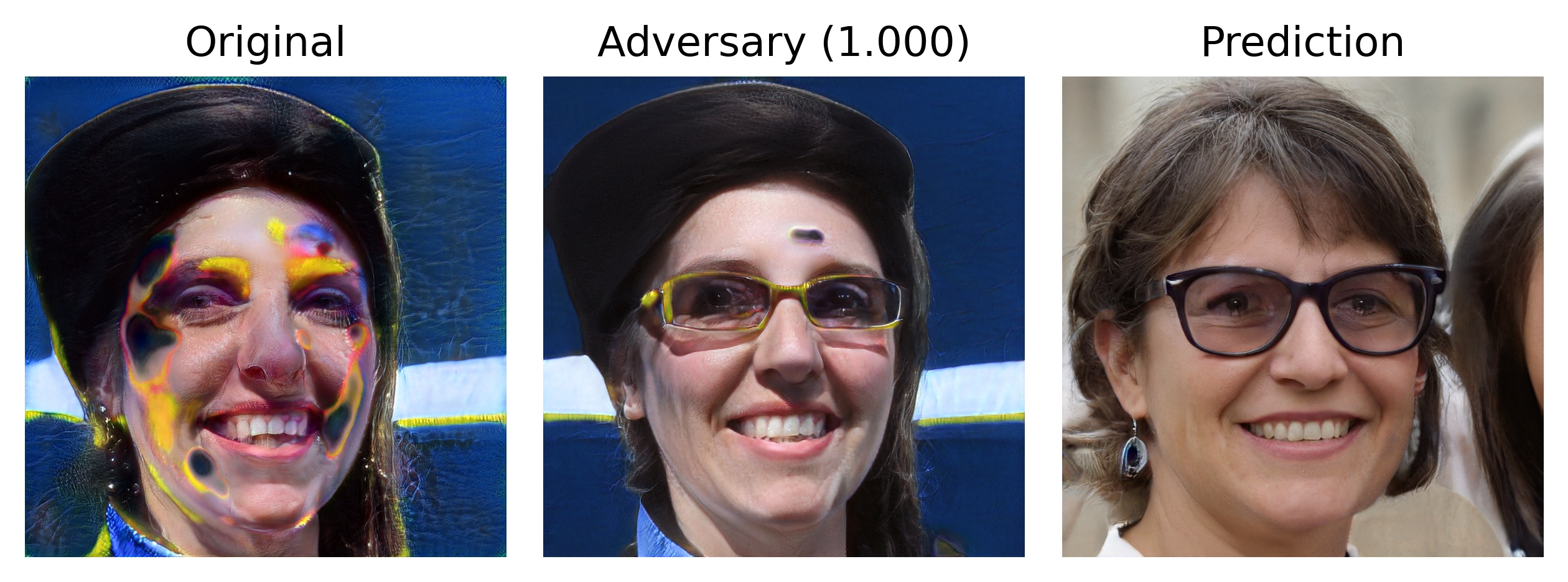}\\
    \includegraphics[trim=0cm 0.3cm 0cm 0.7cm,clip,width=\columnwidth]{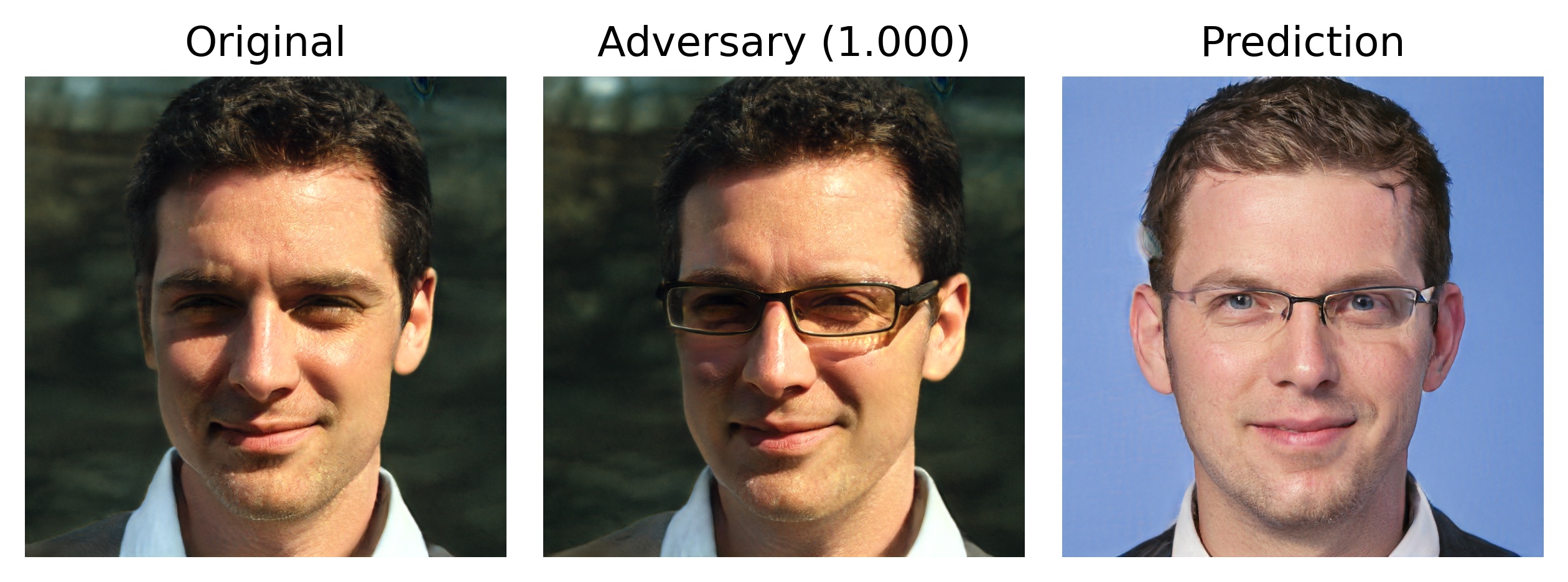}\\
    \includegraphics[trim=0cm 0.3cm 0cm 0.7cm,clip,width=\columnwidth]{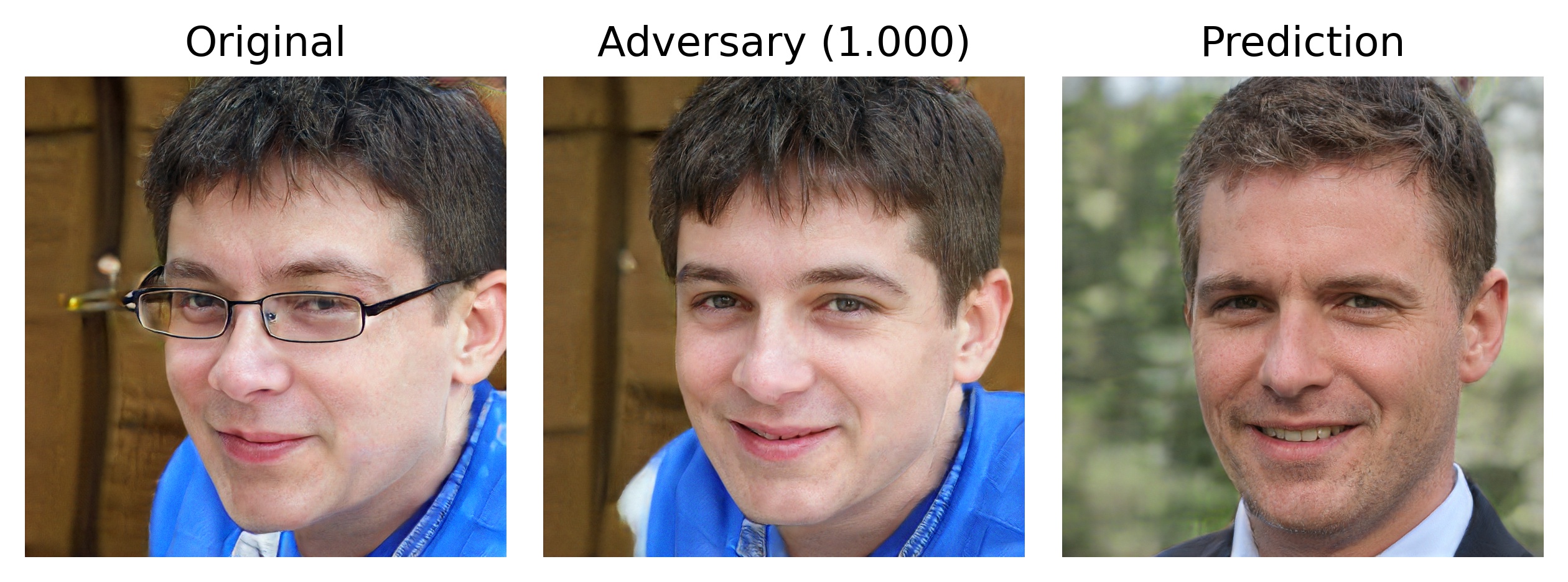}\\
    \includegraphics[trim=0cm 0.3cm 0cm 0.7cm,clip,width=\columnwidth]{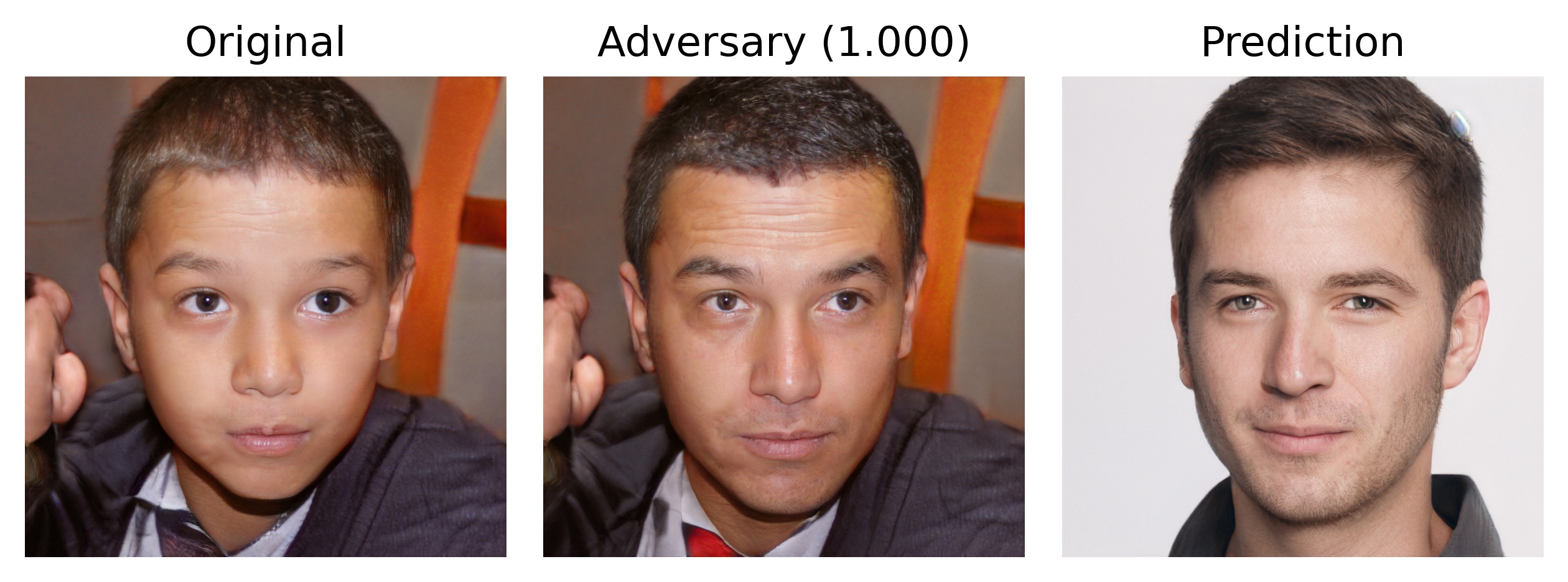}\\
    \includegraphics[trim=0cm 0.3cm 0cm 0.7cm,clip,width=\columnwidth]{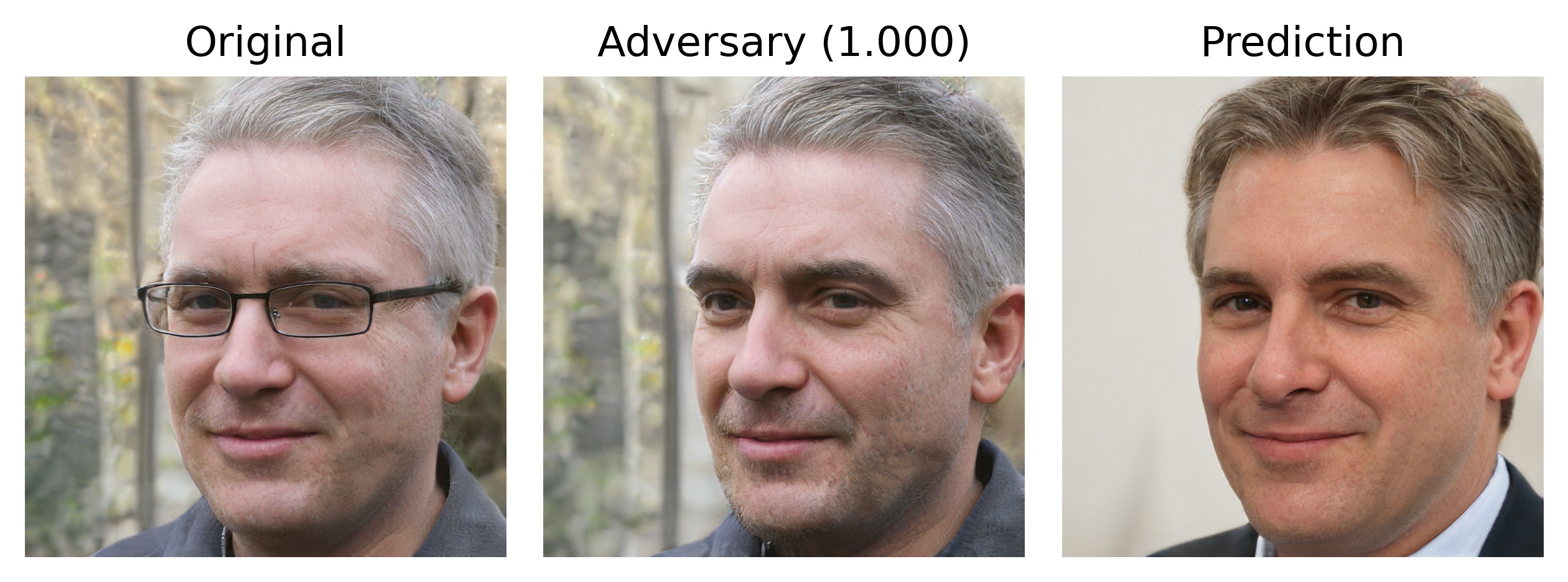}\\
    \includegraphics[trim=0cm 0.3cm 0cm 0.7cm,clip,width=\columnwidth]{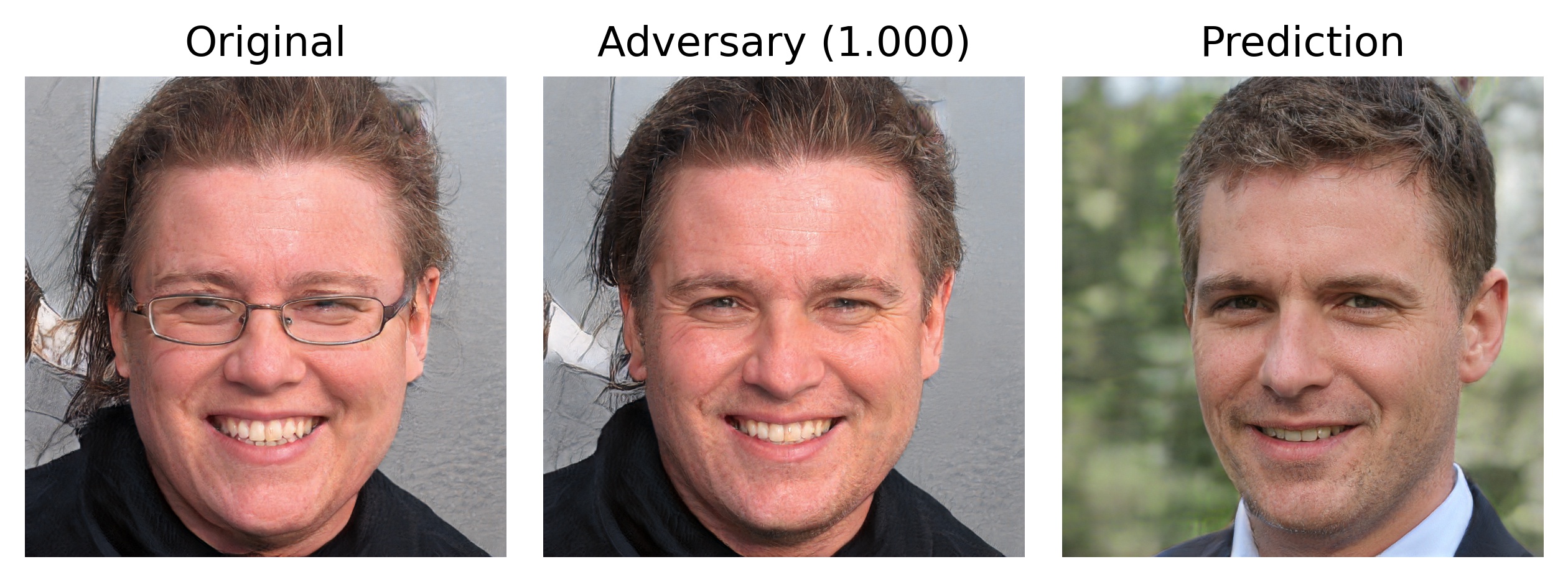}\\
    \includegraphics[trim=0cm 0.3cm 0cm 0.7cm,clip,width=\columnwidth]{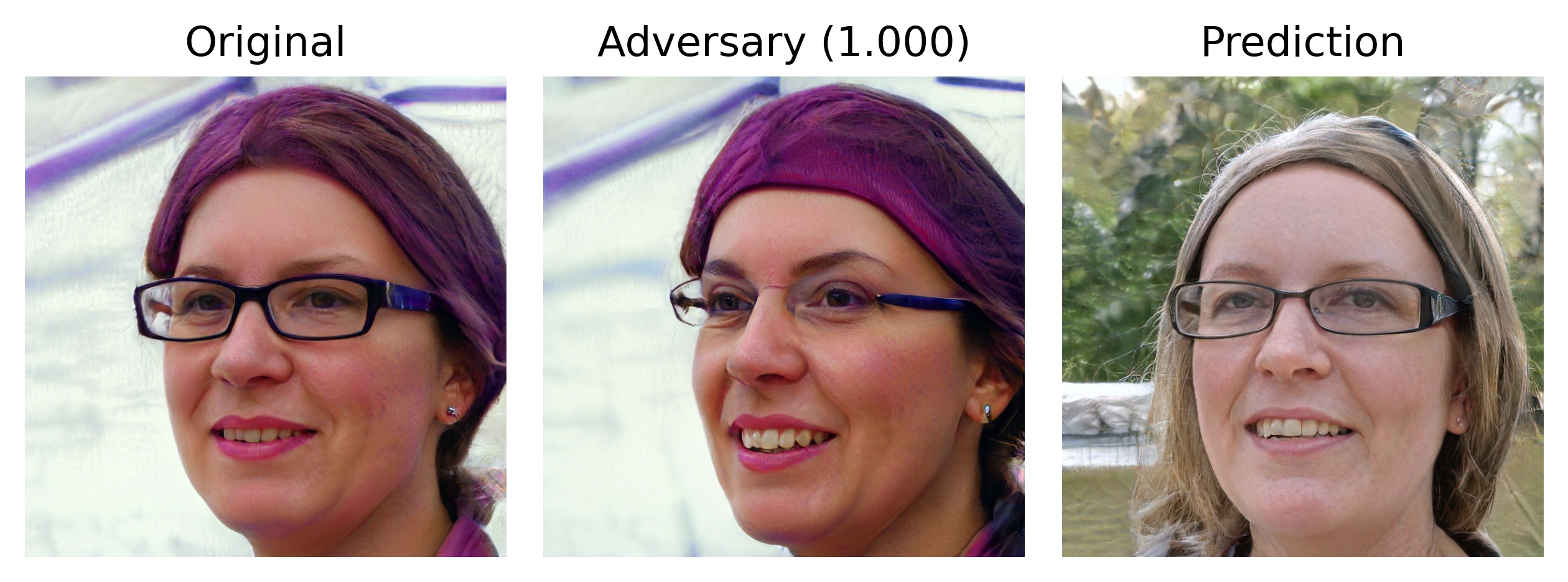}\\
    \includegraphics[trim=0cm 0.3cm 0cm 0.7cm,clip,width=\columnwidth]{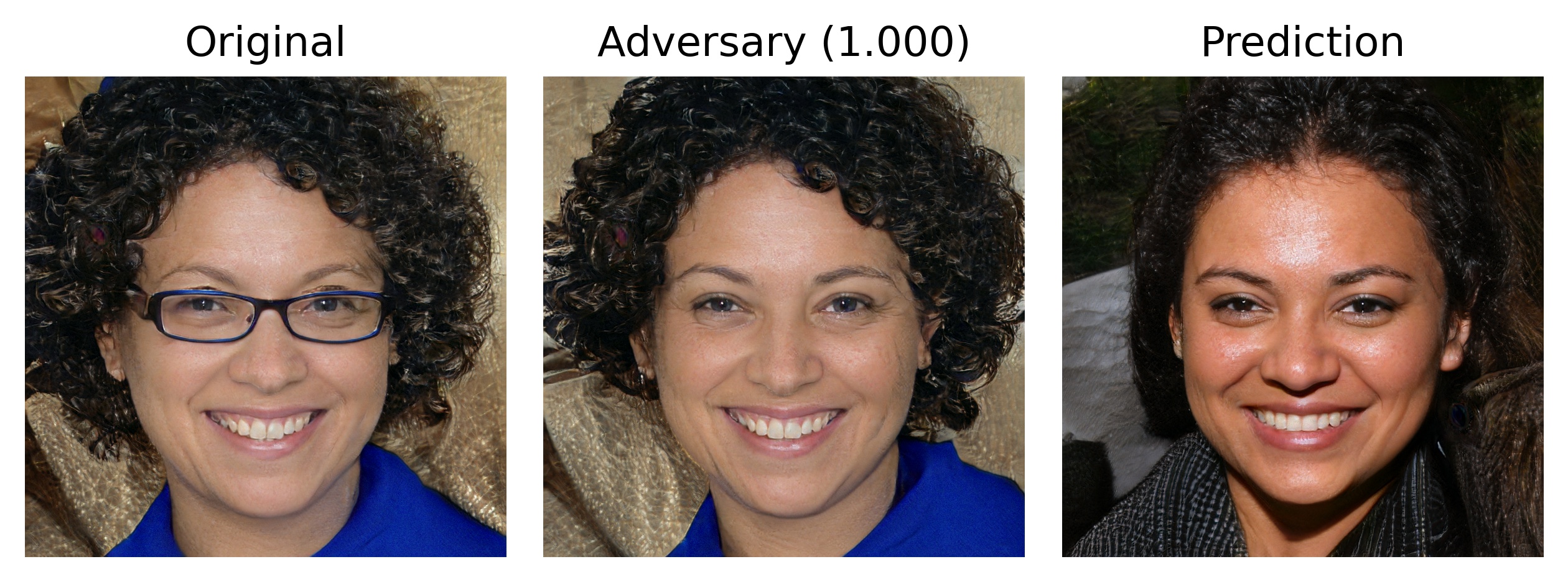}\\
    \caption{\textbf{Adversarial examples found by PGD.} Each row is a different identity. \textit{Left:} original face \textcolor{blue}{\fontfamily{cmss}\selectfont\textbf{A}}, \textit{middle:} modified face \textcolor{orange}{\fontfamily{cmss}\selectfont\textbf{A$^\star$}}, \textit{right:} match \textcolor{purple}{\fontfamily{cmss}\selectfont\textbf{B}}. The FRM prefers to match \textcolor{orange}{\fontfamily{cmss}\selectfont\textbf{A$^\star$}} with \textcolor{purple}{\fontfamily{cmss}\selectfont\textbf{B}} rather than with \textcolor{blue}{\fontfamily{cmss}\selectfont\textbf{A}}.}
    \label{fig:pgd_qual1}
\end{figure}

\begin{figure}
    \centering
    \includegraphics[trim=0cm 0.3cm 0cm 0.7cm,clip,width=\columnwidth]{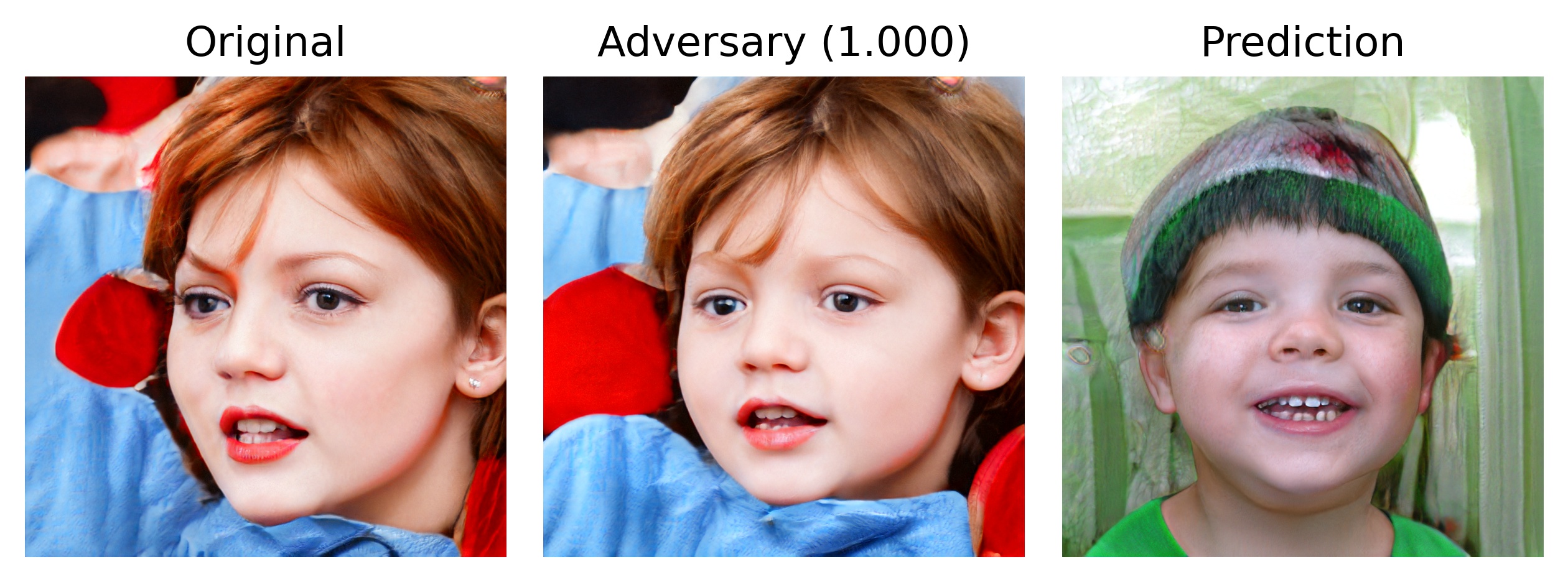}\\
    \includegraphics[trim=0cm 0.3cm 0cm 0.7cm,clip,width=\columnwidth]{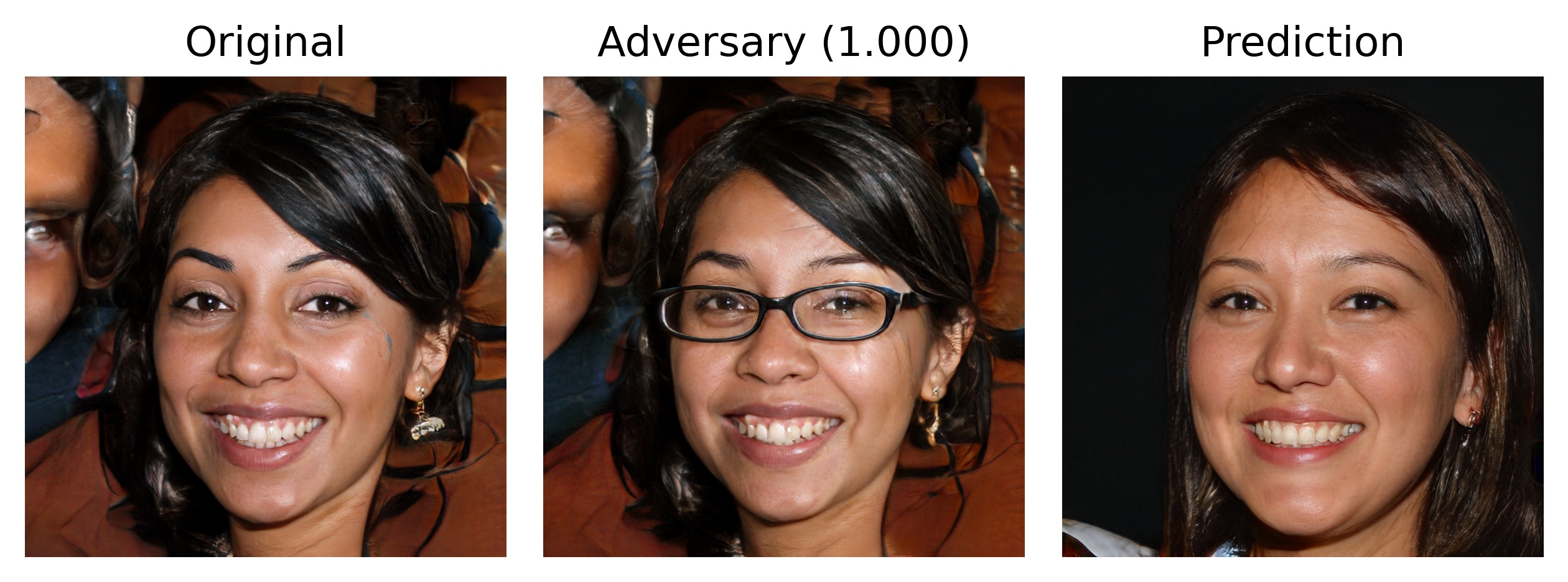}\\
    \includegraphics[trim=0cm 0.3cm 0cm 0.7cm,clip,width=\columnwidth]{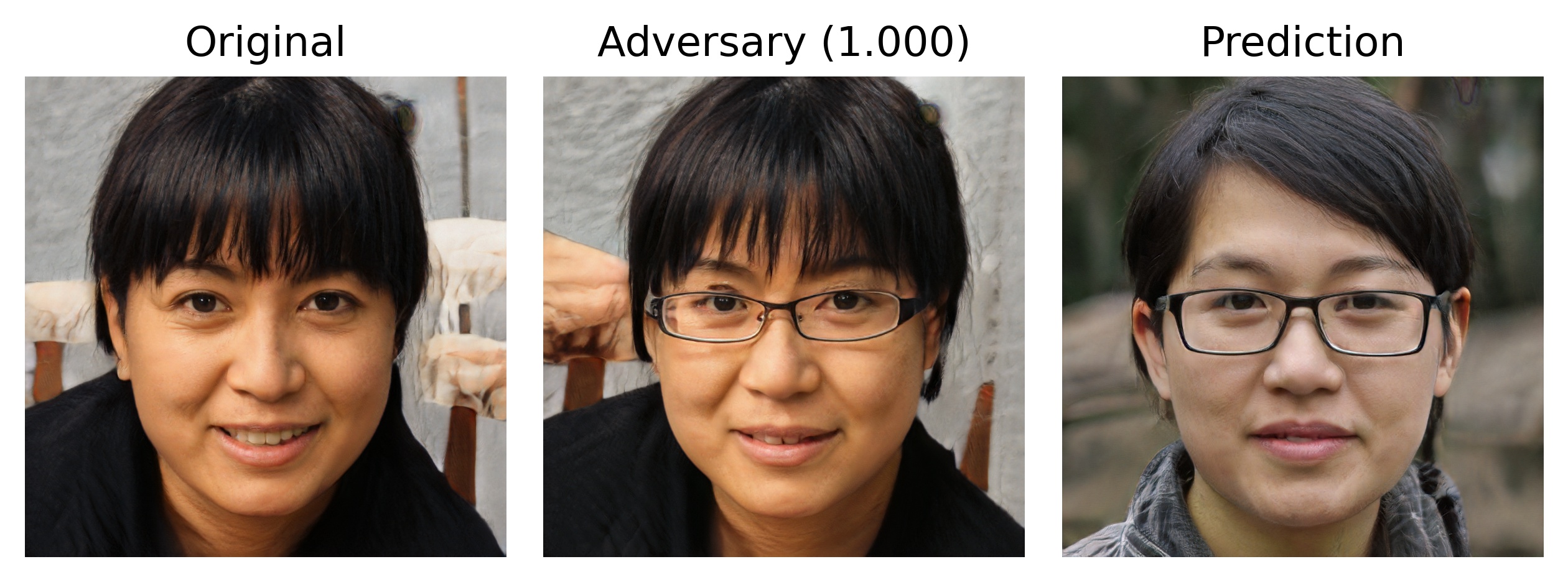}\\
    \includegraphics[trim=0cm 0.3cm 0cm 0.7cm,clip,width=\columnwidth]{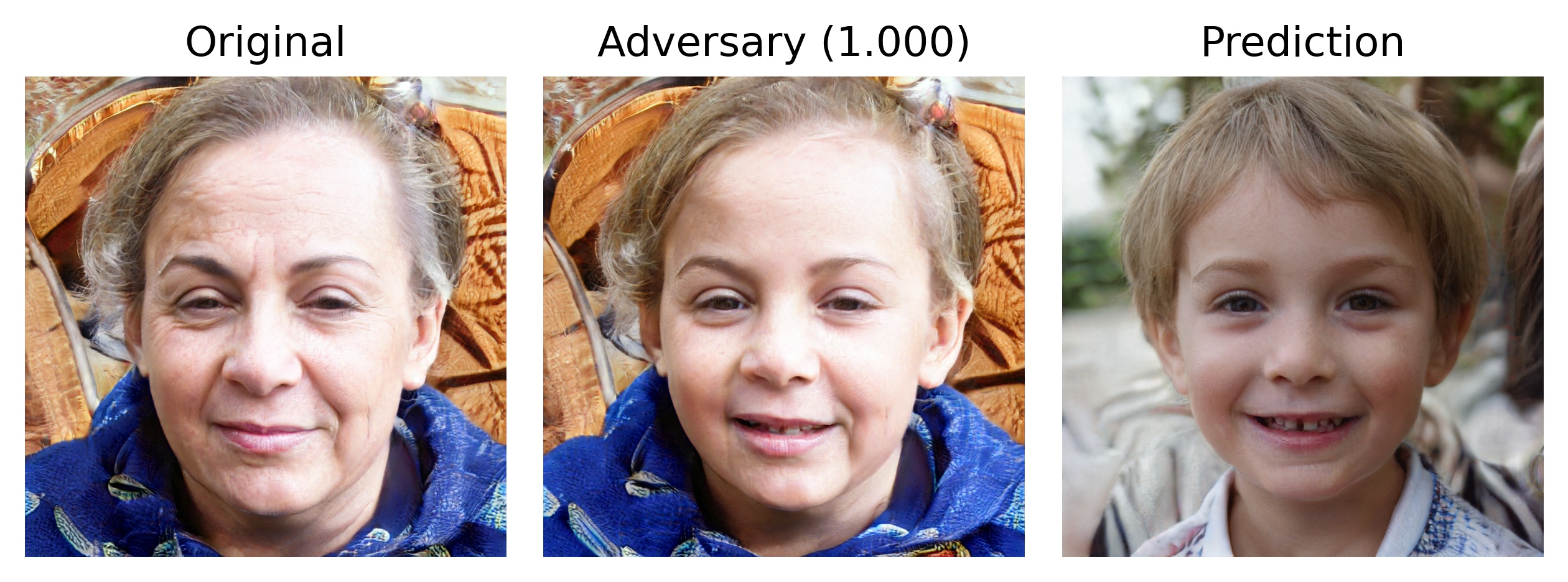}\\
    \includegraphics[trim=0cm 0.3cm 0cm 0.7cm,clip,width=\columnwidth]{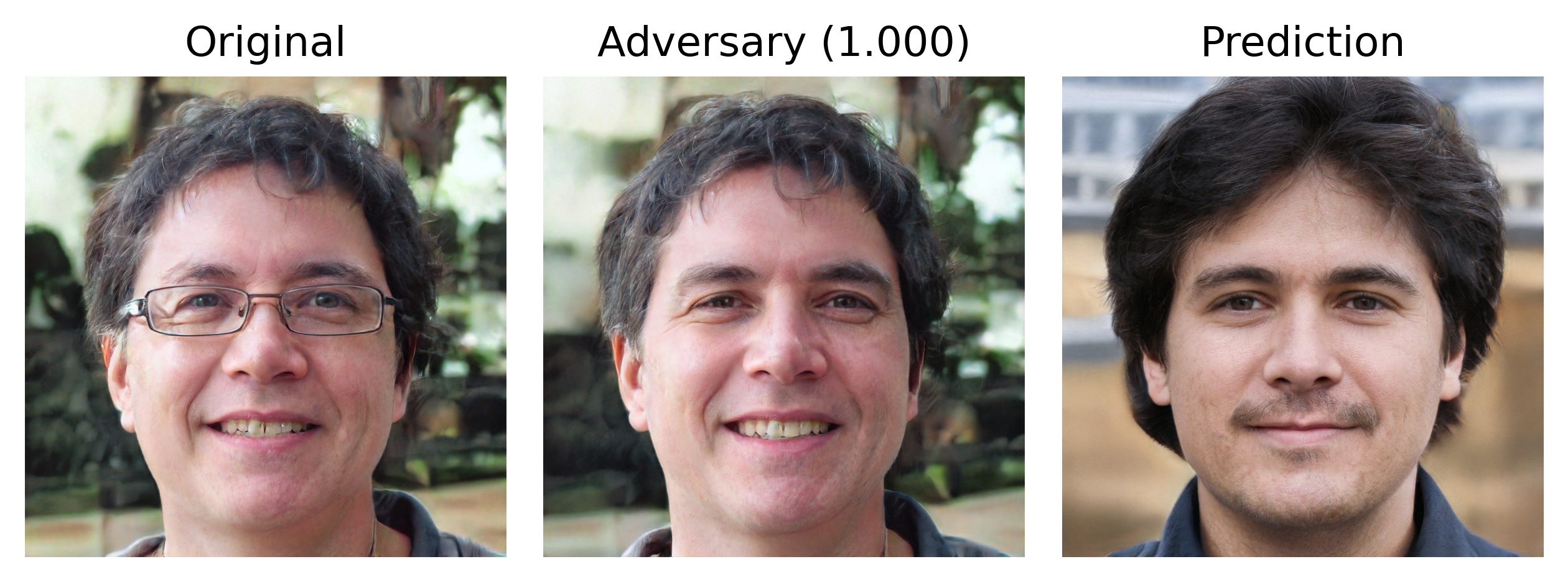}\\
    \includegraphics[trim=0cm 0.3cm 0cm 0.7cm,clip,width=\columnwidth]{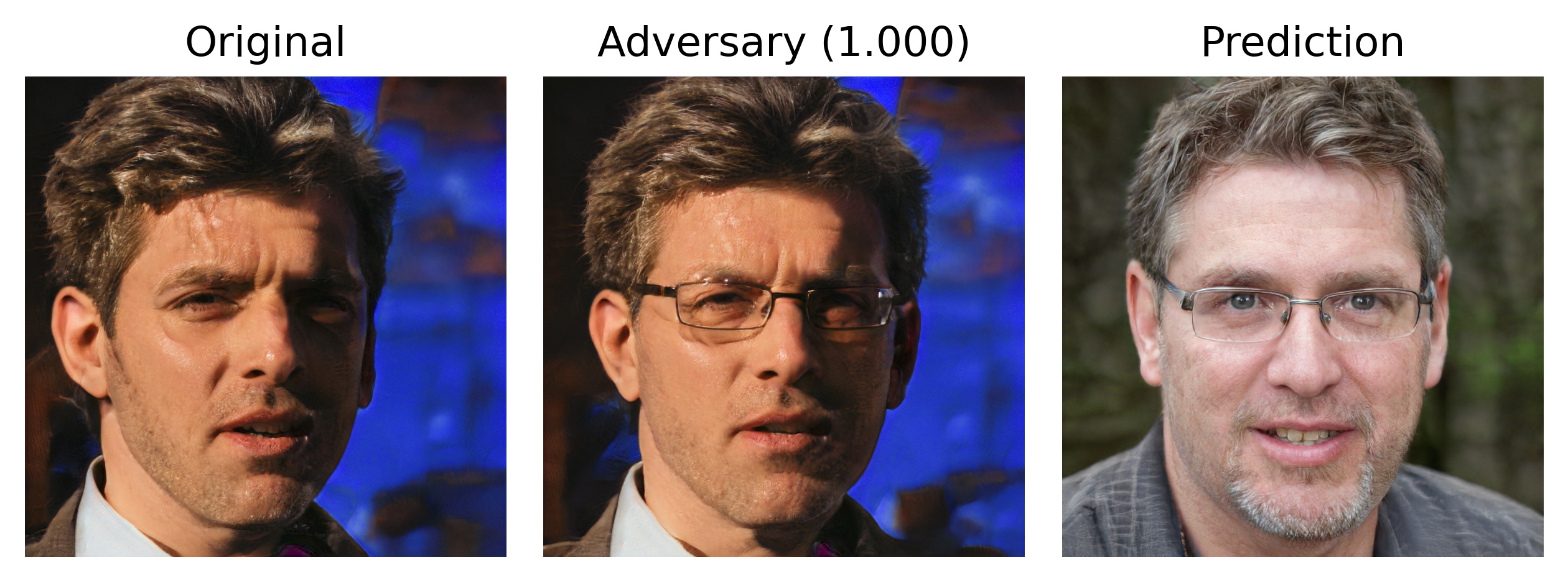}\\
    \includegraphics[trim=0cm 0.3cm 0cm 0.7cm,clip,width=\columnwidth]{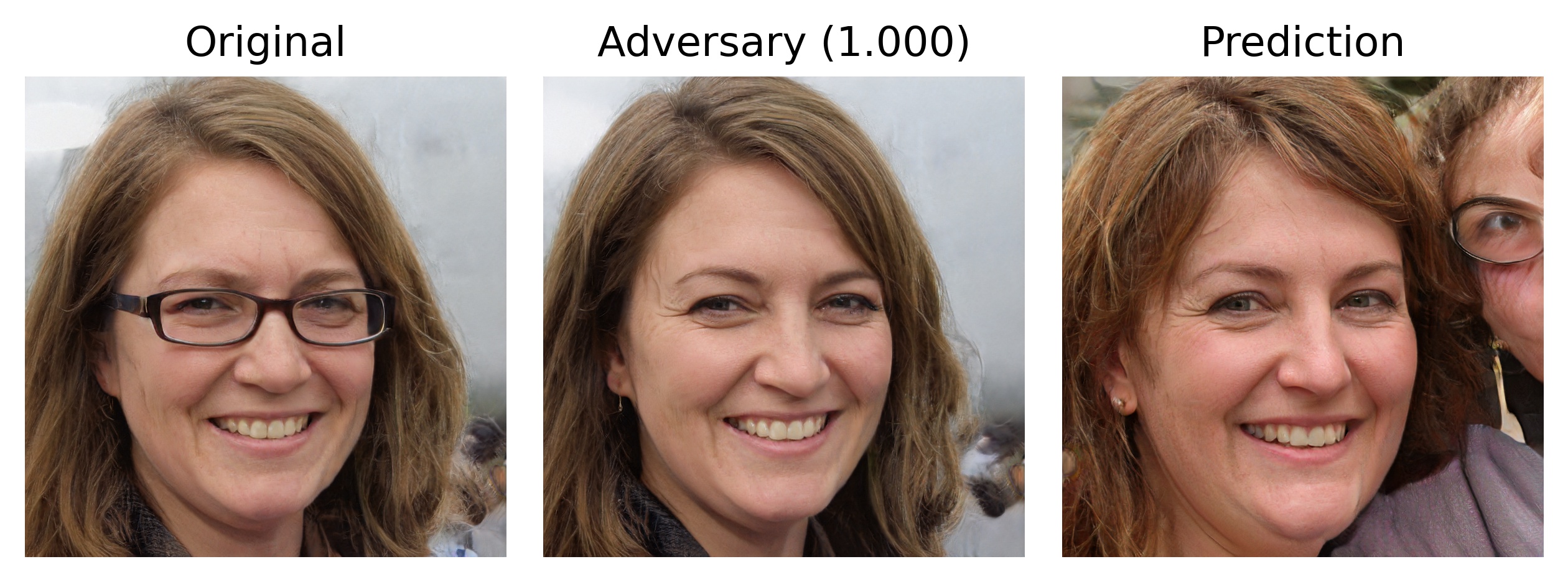}\\
    \includegraphics[trim=0cm 0.3cm 0cm 0.7cm,clip,width=\columnwidth]{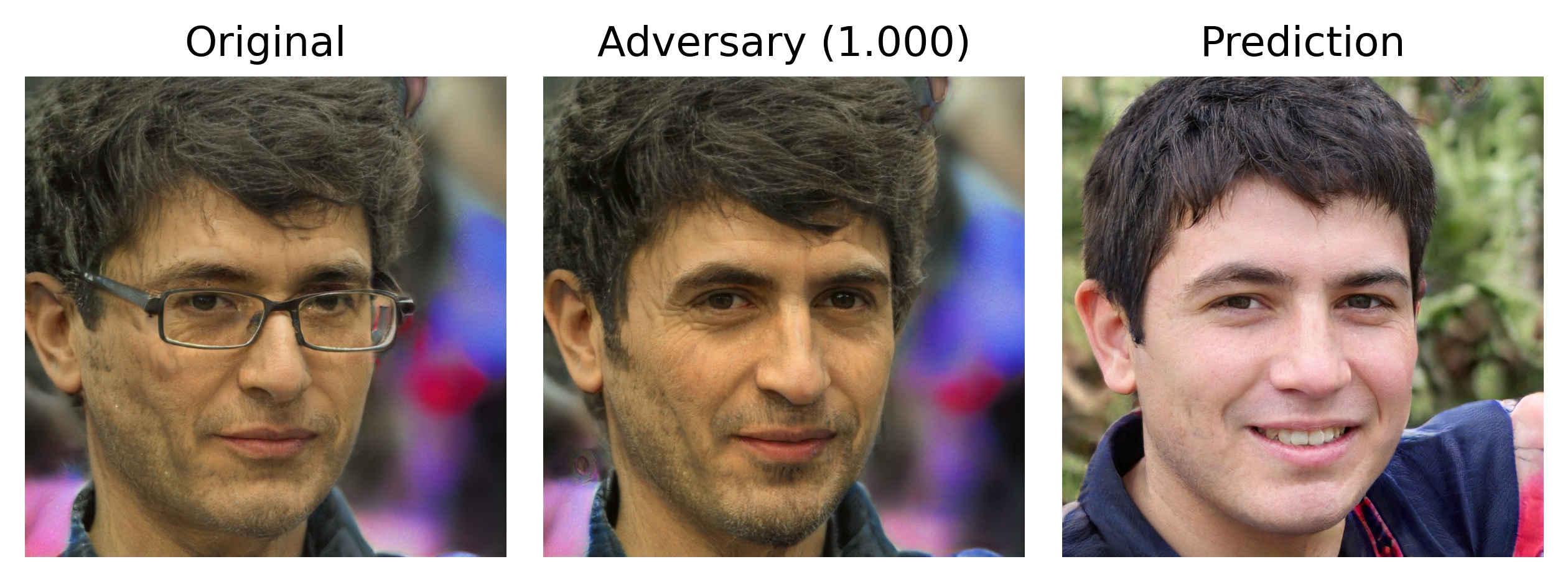}\\
    \caption{\textbf{Adversarial examples found by PGD.} Each row is a different identity. \textit{Left:} original face \textcolor{blue}{\fontfamily{cmss}\selectfont\textbf{A}}, \textit{middle:} modified face \textcolor{orange}{\fontfamily{cmss}\selectfont\textbf{A$^\star$}}, \textit{right:} match \textcolor{purple}{\fontfamily{cmss}\selectfont\textbf{B}}. The FRM prefers to match \textcolor{orange}{\fontfamily{cmss}\selectfont\textbf{A$^\star$}} with \textcolor{purple}{\fontfamily{cmss}\selectfont\textbf{B}} rather than with \textcolor{blue}{\fontfamily{cmss}\selectfont\textbf{A}}.}
    \label{fig:pgd_qual2}
\end{figure}

\begin{figure}
    \centering
    \includegraphics[trim=0cm 0.3cm 0cm 0.7cm,clip,width=\columnwidth]{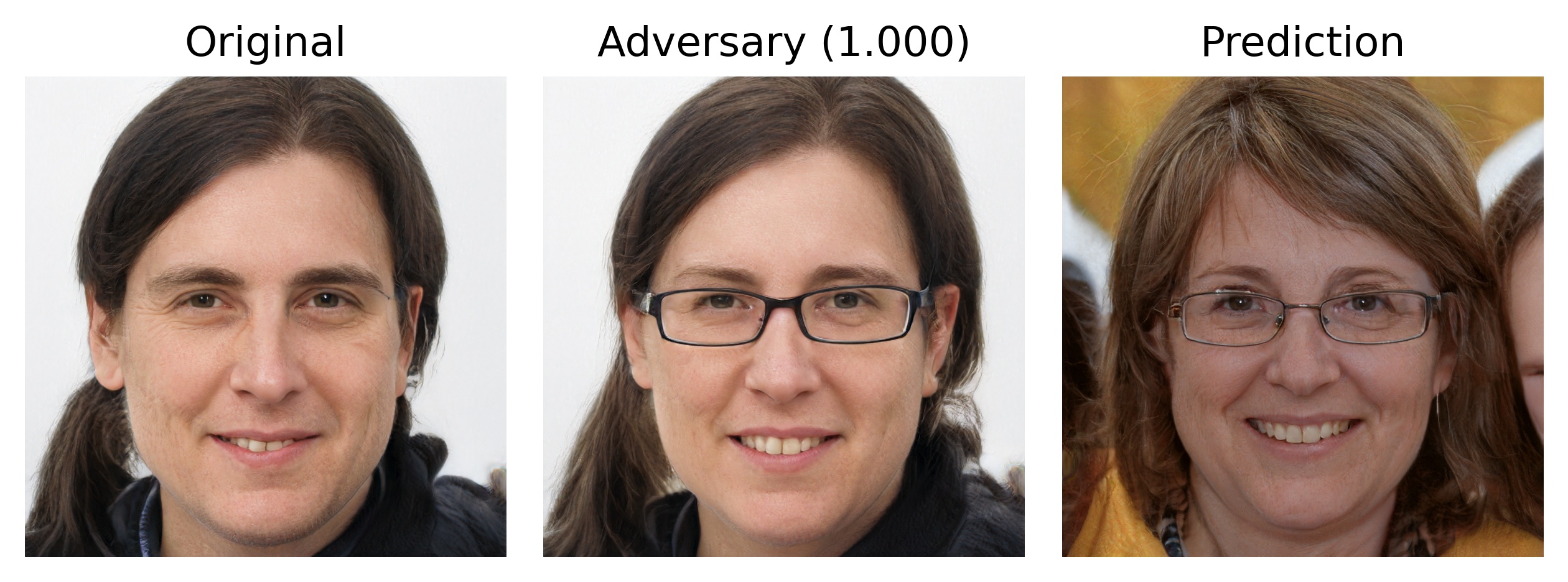}\\
    \includegraphics[trim=0cm 0.3cm 0cm 0.7cm,clip,width=\columnwidth]{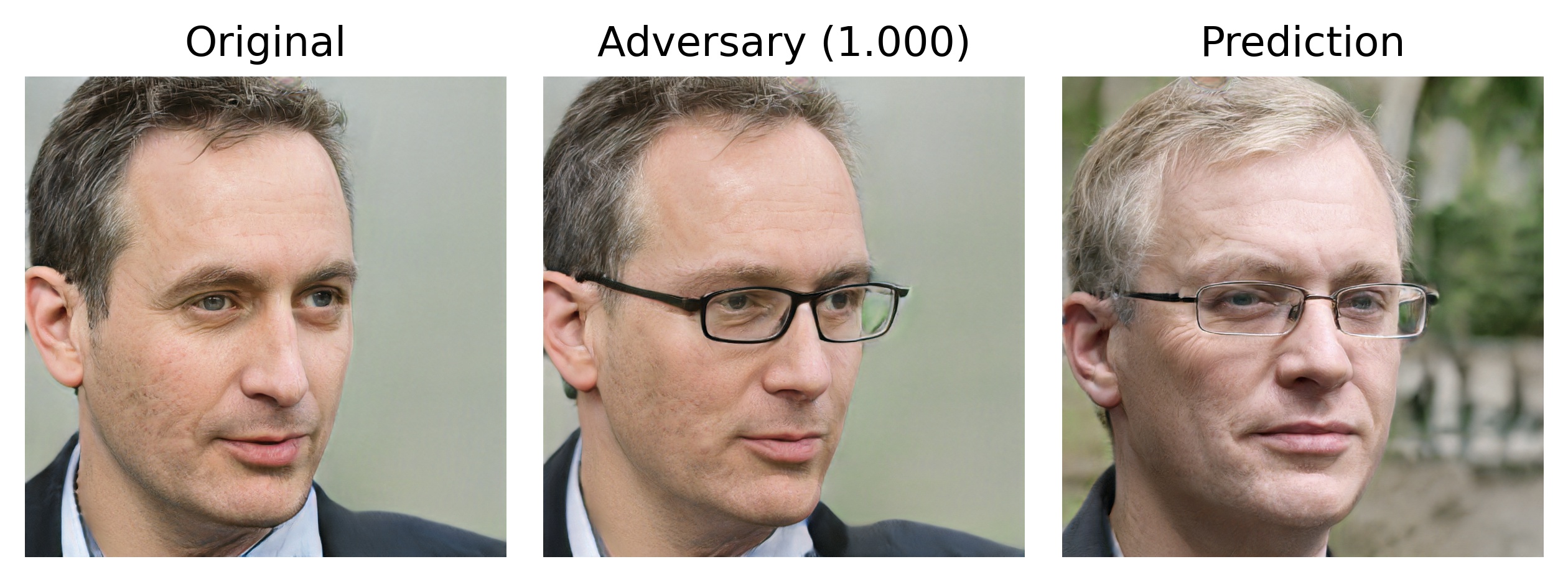}\\
    \includegraphics[trim=0cm 0.3cm 0cm 0.7cm,clip,width=\columnwidth]{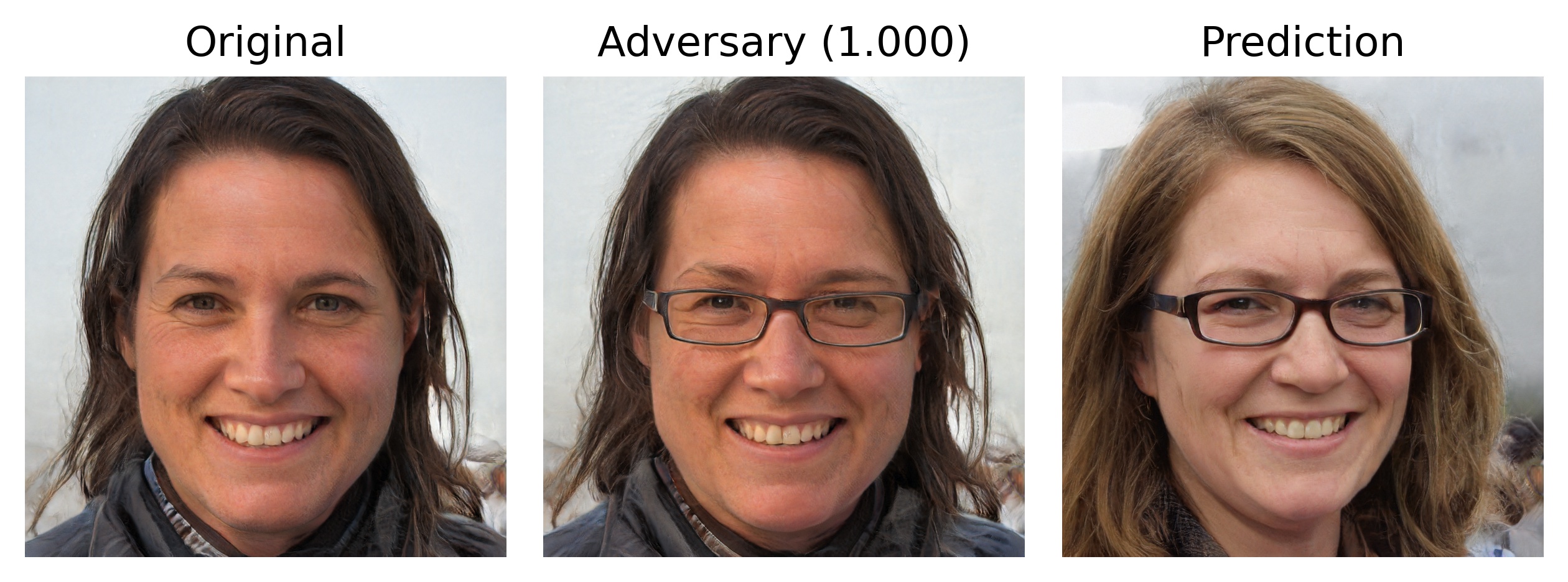}\\
    \includegraphics[trim=0cm 0.3cm 0cm 0.7cm,clip,width=\columnwidth]{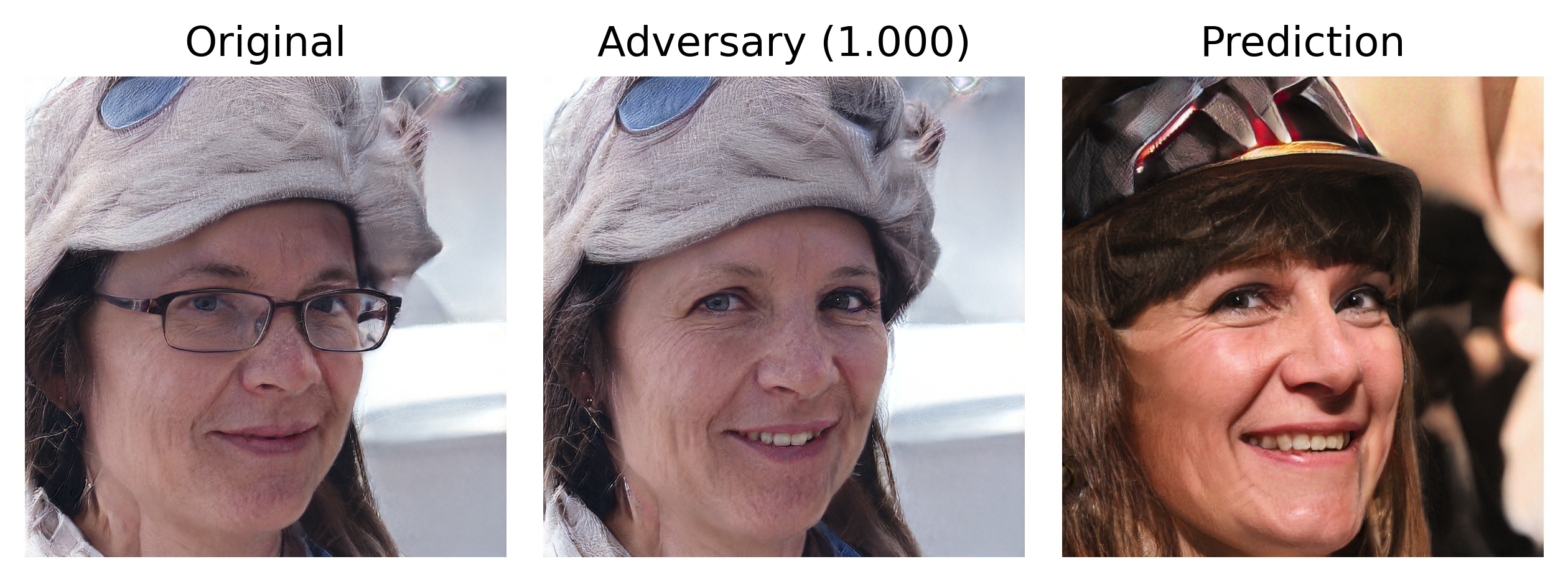}\\
    \includegraphics[trim=0cm 0.3cm 0cm 0.7cm,clip,width=\columnwidth]{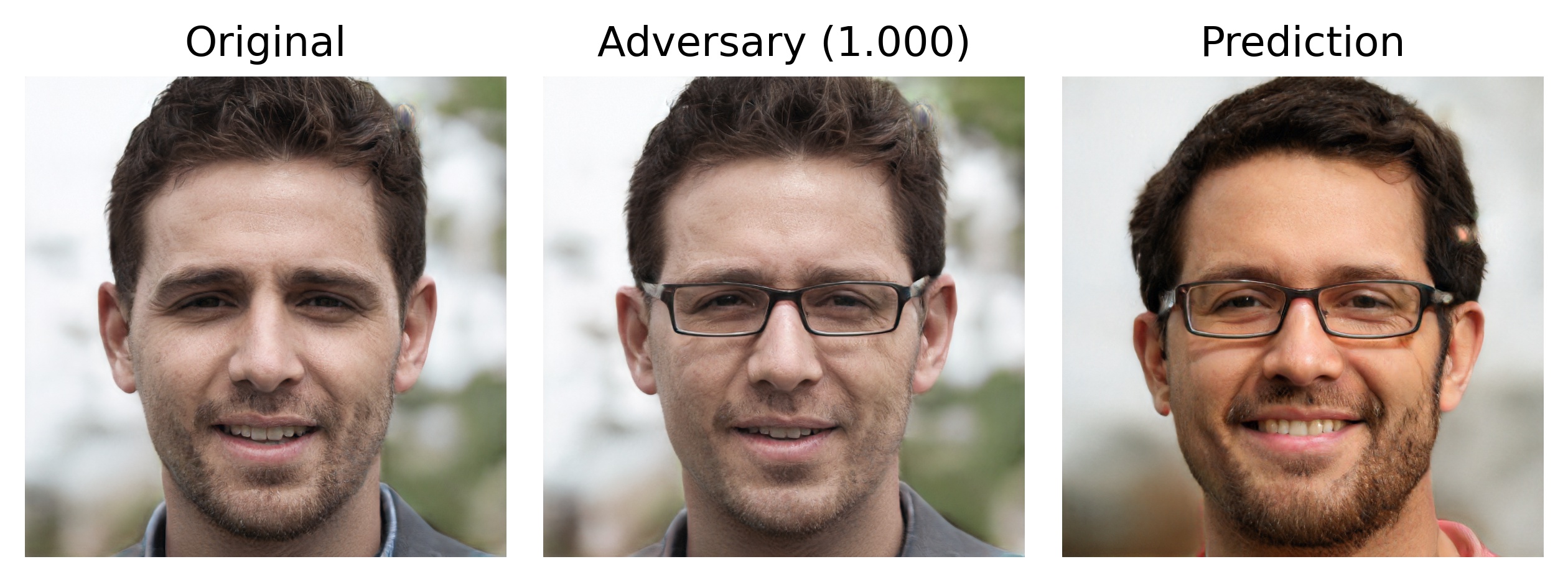}\\
    \includegraphics[trim=0cm 0.3cm 0cm 0.7cm,clip,width=\columnwidth]{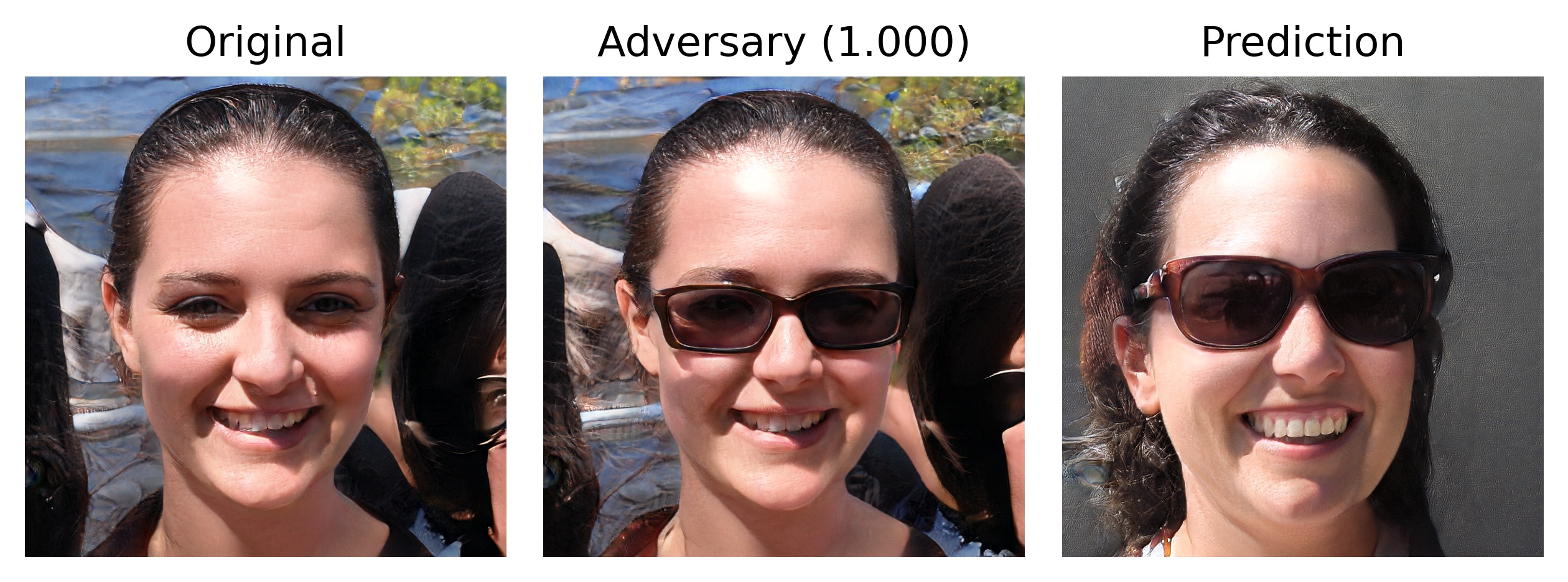}\\
    \includegraphics[trim=0cm 0.3cm 0cm 0.7cm,clip,width=\columnwidth]{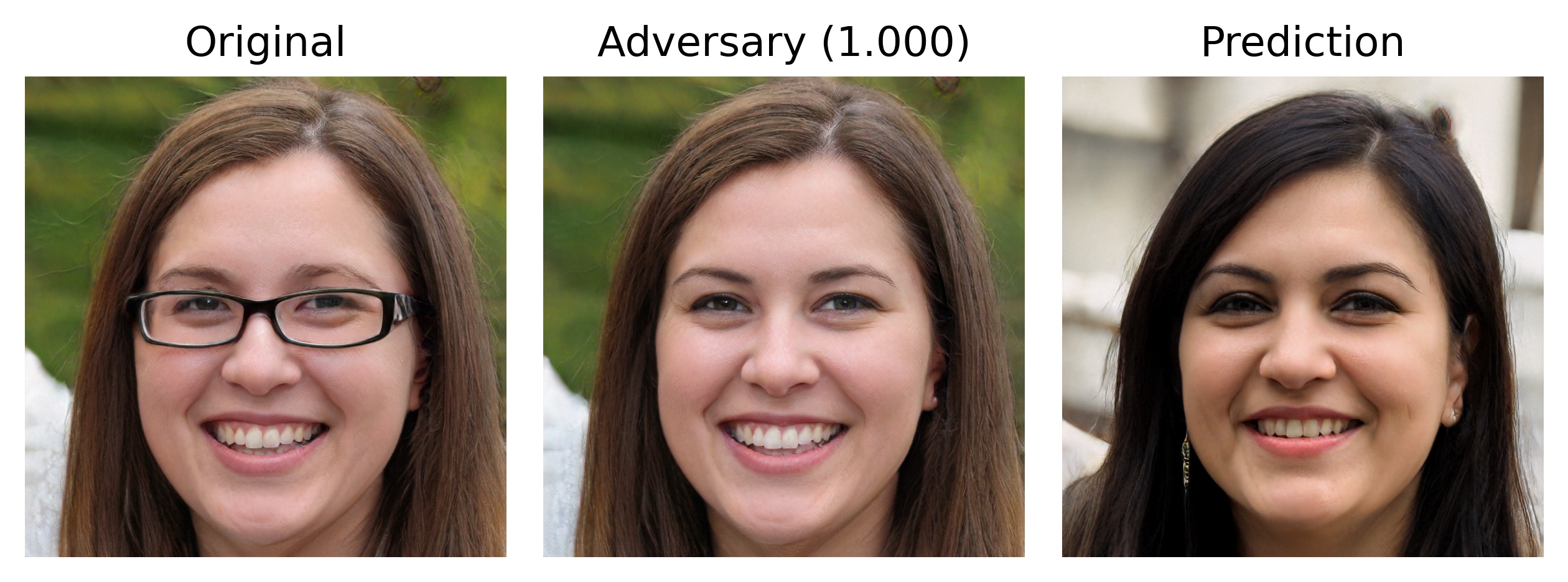}\\
    \includegraphics[trim=0cm 0.3cm 0cm 0.7cm,clip,width=\columnwidth]{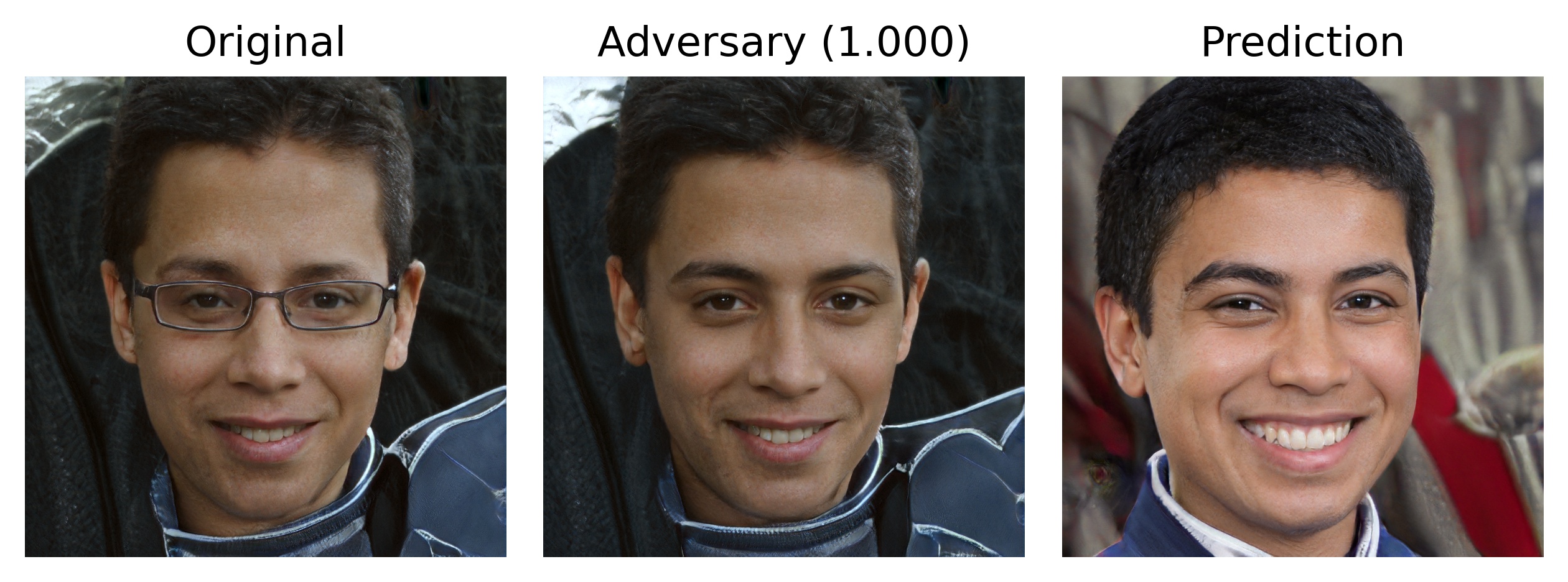}\\
    \caption{\textbf{Adversarial examples found by PGD.} Each row is a different identity. \textit{Left:} original face \textcolor{blue}{\fontfamily{cmss}\selectfont\textbf{A}}, \textit{middle:} modified face \textcolor{orange}{\fontfamily{cmss}\selectfont\textbf{A$^\star$}}, \textit{right:} match \textcolor{purple}{\fontfamily{cmss}\selectfont\textbf{B}}. The FRM prefers to match \textcolor{orange}{\fontfamily{cmss}\selectfont\textbf{A$^\star$}} with \textcolor{purple}{\fontfamily{cmss}\selectfont\textbf{B}} rather than with \textcolor{blue}{\fontfamily{cmss}\selectfont\textbf{A}}.}
    \label{fig:pgd_qual3}
\end{figure}

\begin{figure}
    \centering
    \includegraphics[trim=0cm 0.3cm 0cm 0.7cm,clip,width=\columnwidth]{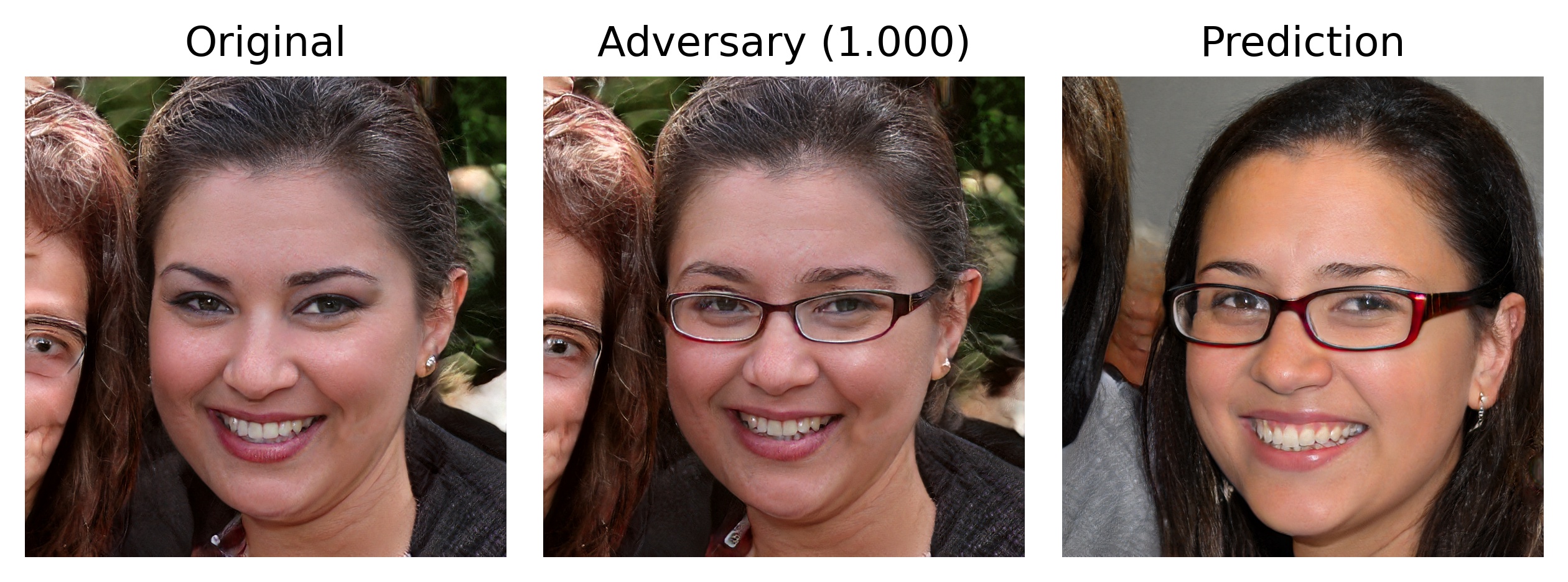}\\
    \includegraphics[trim=0cm 0.3cm 0cm 0.7cm,clip,width=\columnwidth]{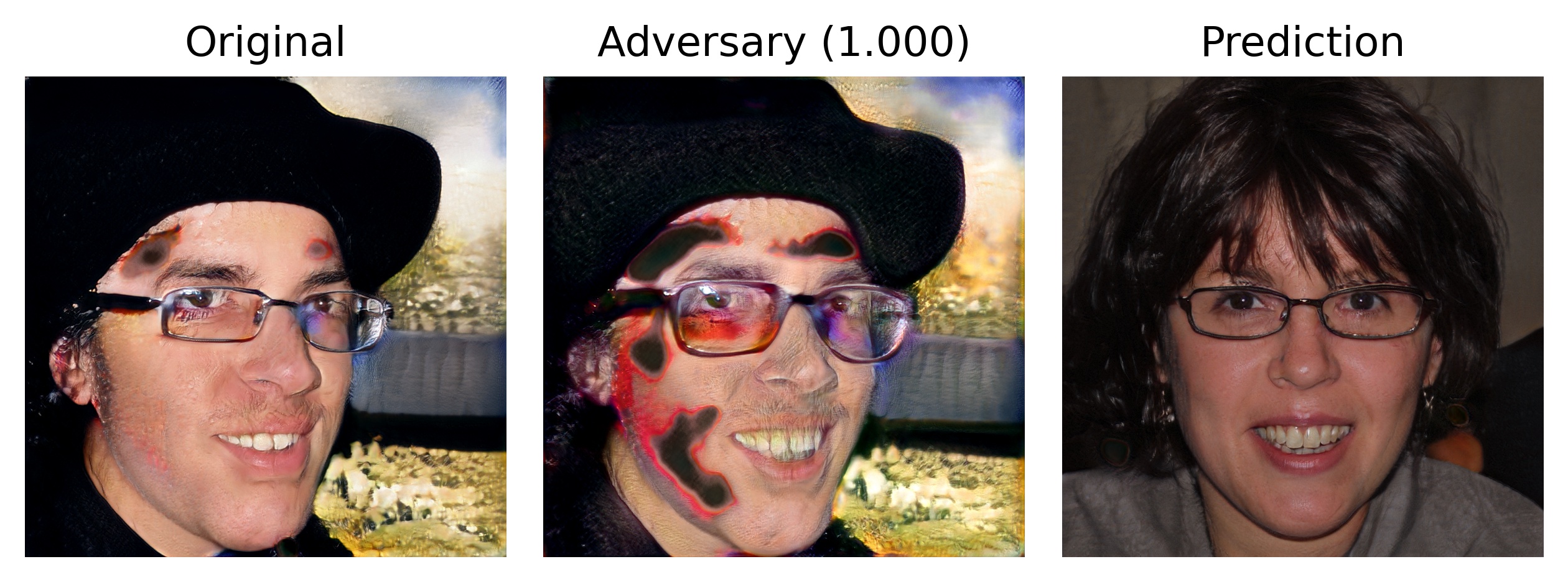}\\
    \includegraphics[trim=0cm 0.3cm 0cm 0.7cm,clip,width=\columnwidth]{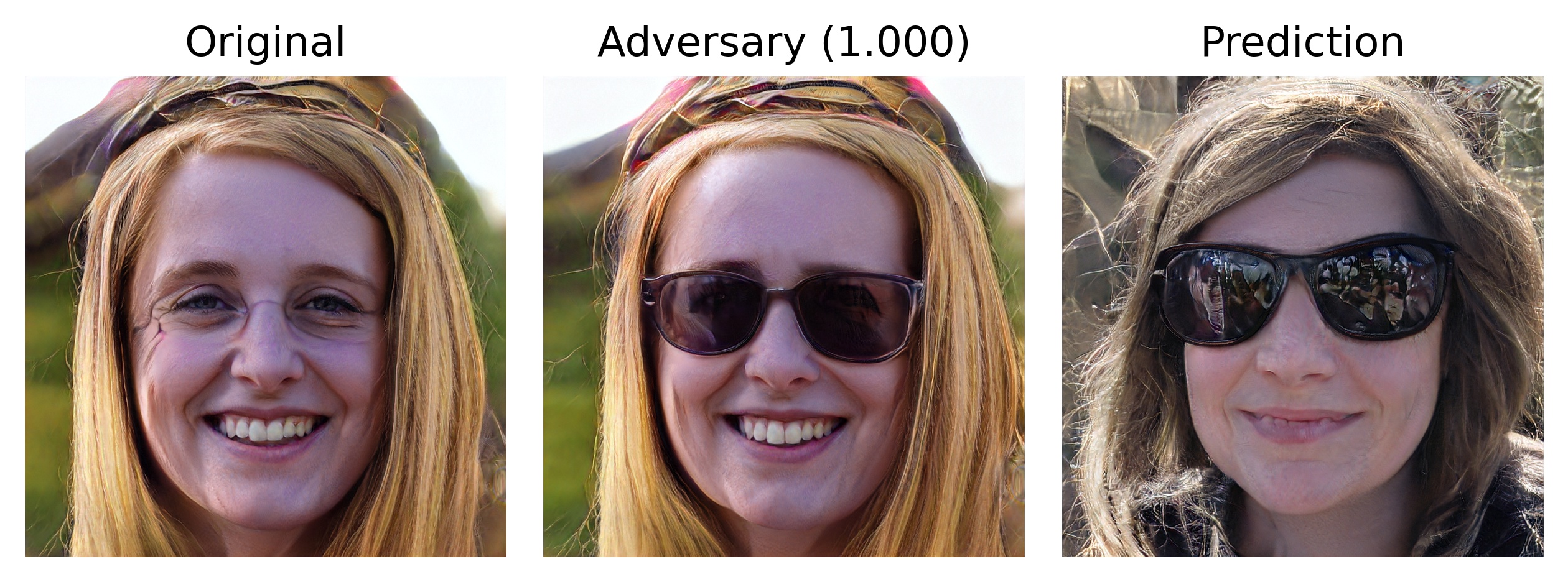}\\
    \includegraphics[trim=0cm 0.3cm 0cm 0.7cm,clip,width=\columnwidth]{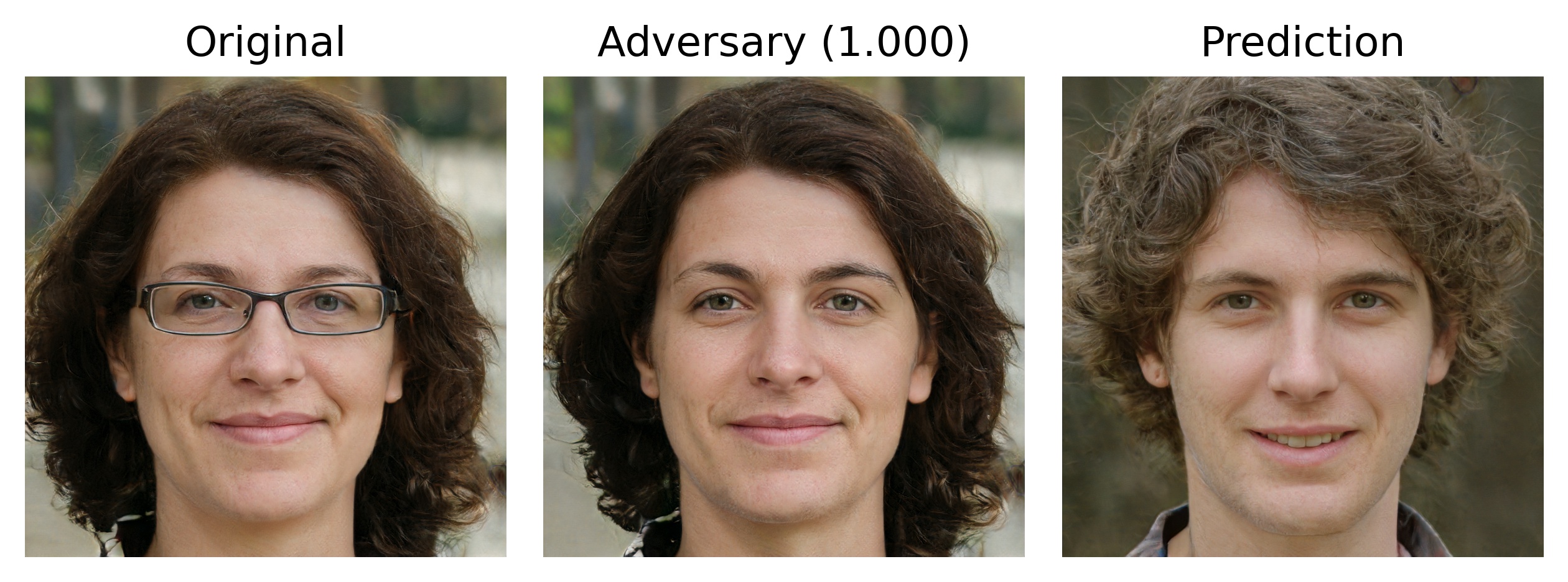}\\
    \includegraphics[trim=0cm 0.3cm 0cm 0.7cm,clip,width=\columnwidth]{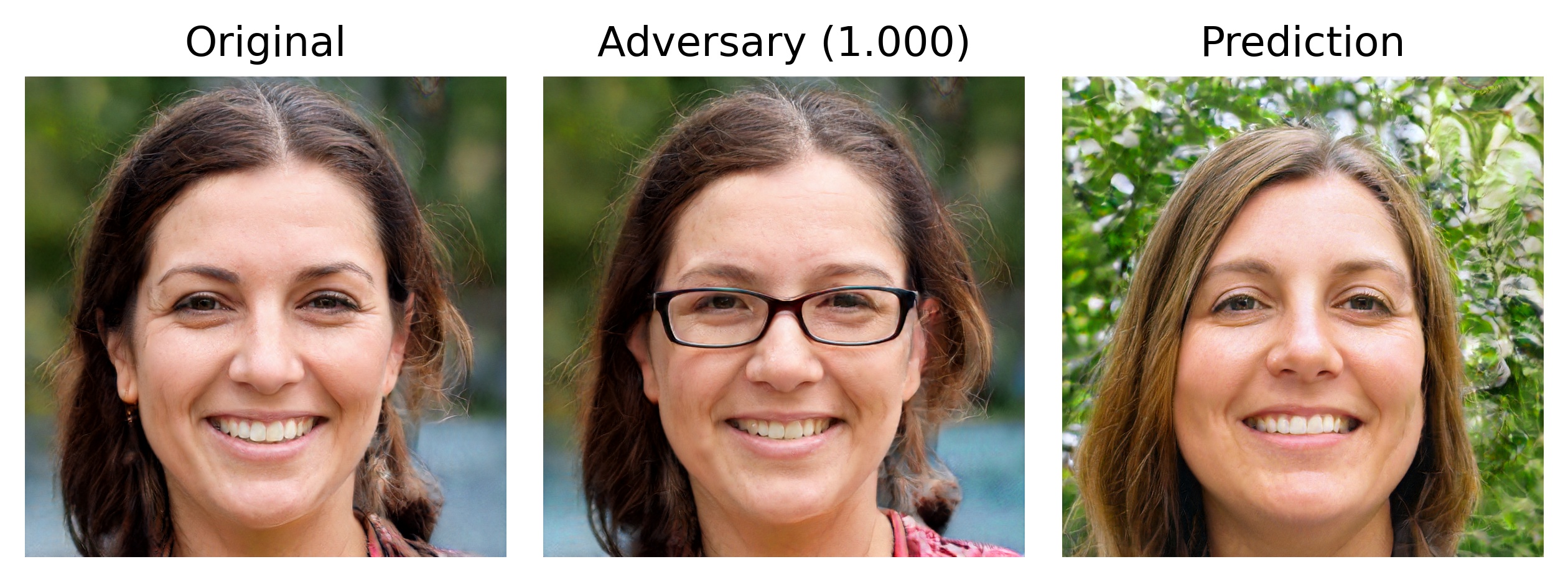}\\
    \includegraphics[trim=0cm 0.3cm 0cm 0.7cm,clip,width=\columnwidth]{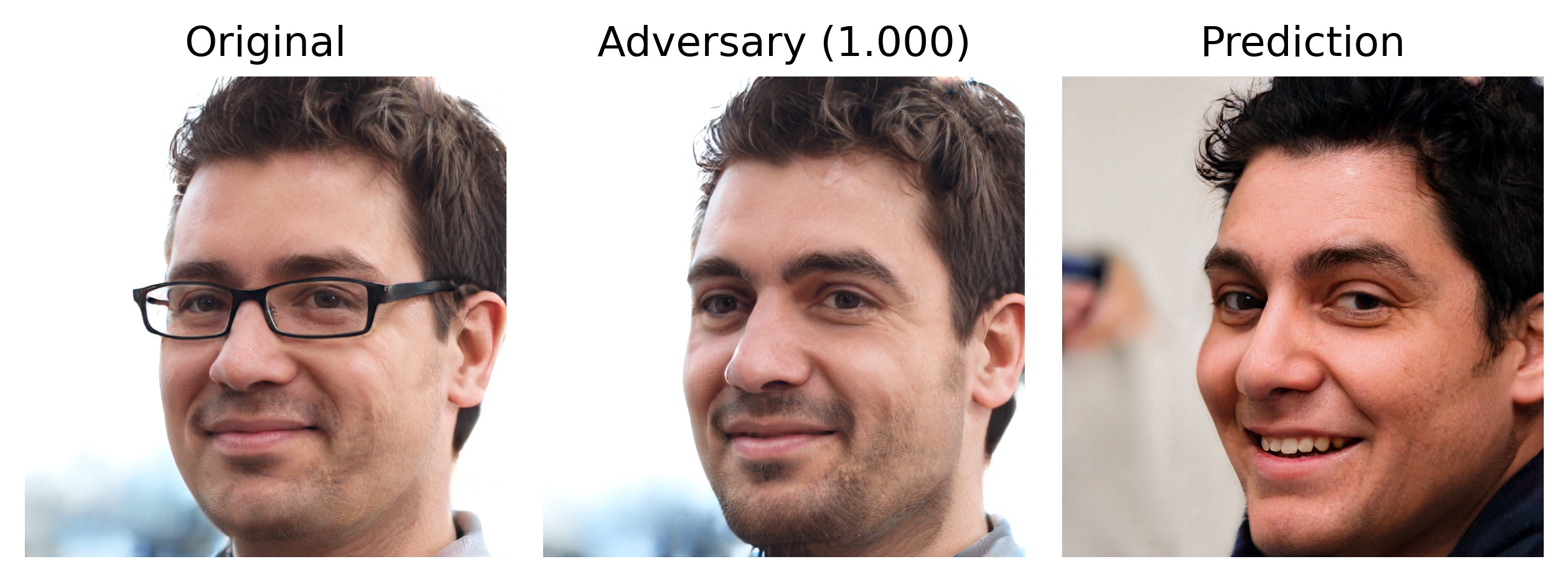}\\
    \includegraphics[trim=0cm 0.3cm 0cm 0.7cm,clip,width=\columnwidth]{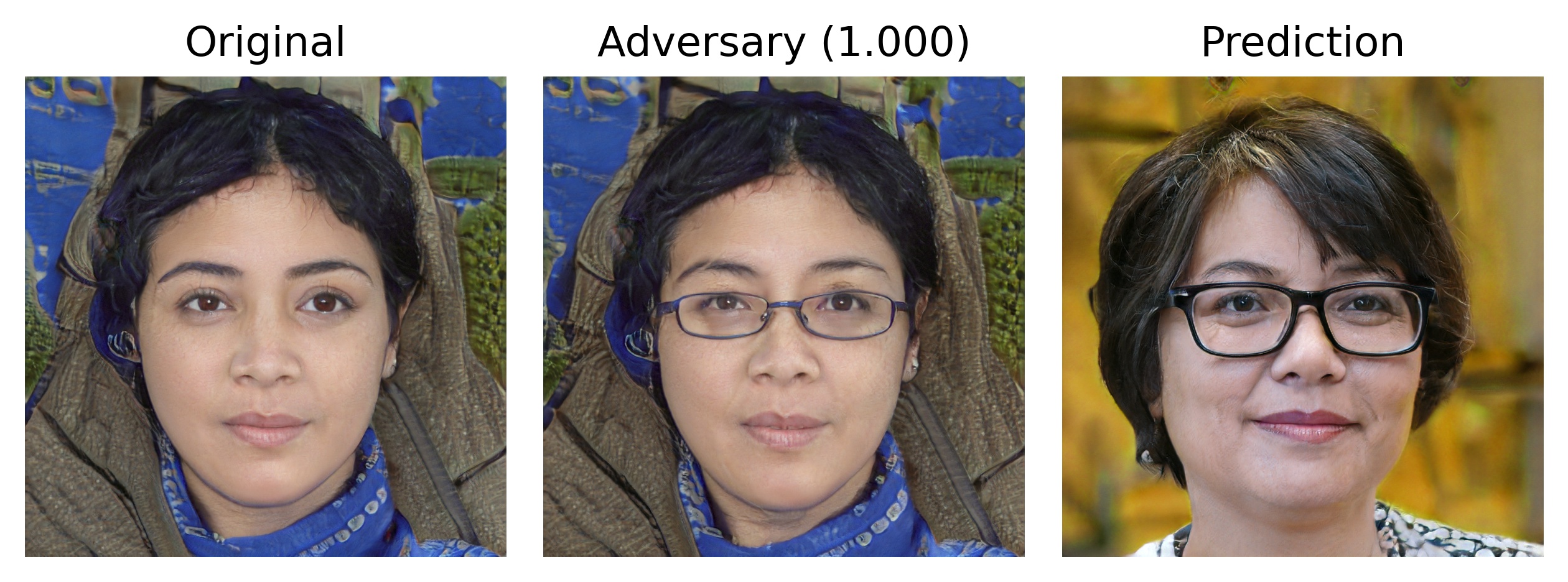}\\
    \includegraphics[trim=0cm 0.3cm 0cm 0.7cm,clip,width=\columnwidth]{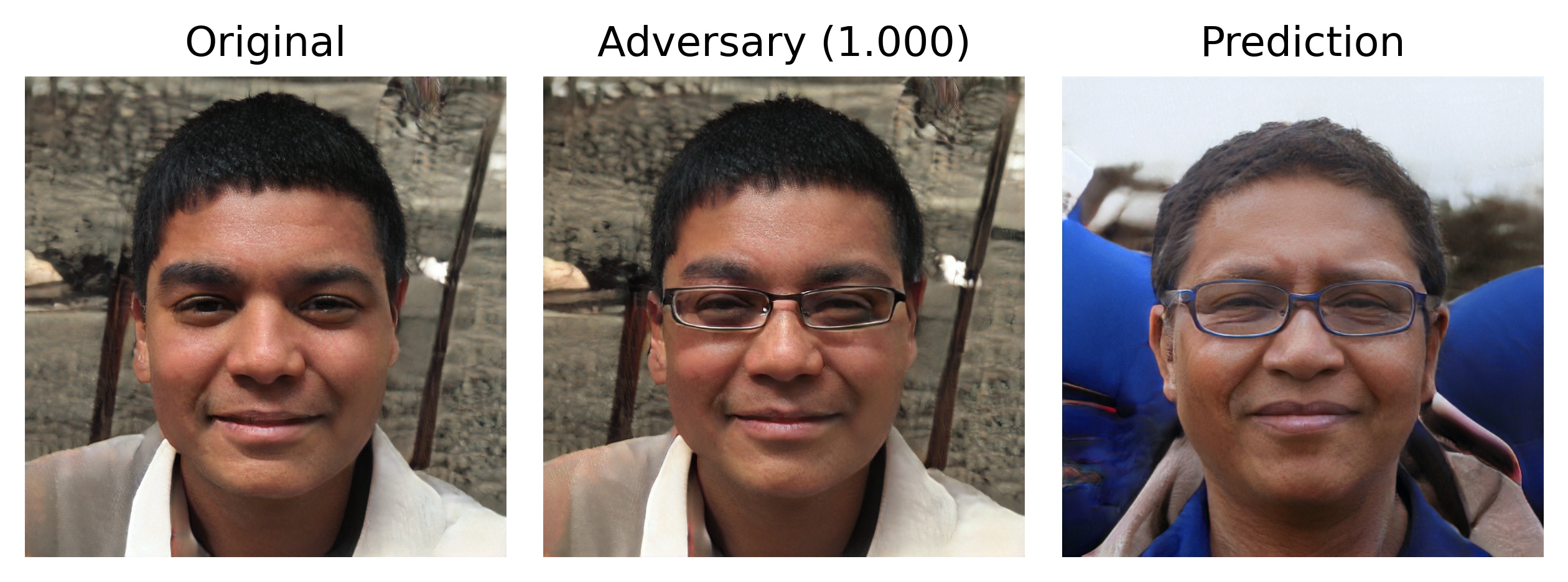}\\
    \caption{\textbf{Adversarial examples found by PGD.} Each row is a different identity. \textit{Left:} original face \textcolor{blue}{\fontfamily{cmss}\selectfont\textbf{A}}, \textit{middle:} modified face \textcolor{orange}{\fontfamily{cmss}\selectfont\textbf{A$^\star$}}, \textit{right:} match \textcolor{purple}{\fontfamily{cmss}\selectfont\textbf{B}}. The FRM prefers to match \textcolor{orange}{\fontfamily{cmss}\selectfont\textbf{A$^\star$}} with \textcolor{purple}{\fontfamily{cmss}\selectfont\textbf{B}} rather than with \textcolor{blue}{\fontfamily{cmss}\selectfont\textbf{A}}.}
    \label{fig:pgd_qual4}
\end{figure}
\onecolumn

\section{PGD Attribute interpretation}\label{sec:app_pgd_attr}
Section~\ref{sec:pgd_attacks} presented attacks with PGD.
The adversarial examples found by these attacks were then interpreted via the procedure we introduced in Section~\ref{sec:interpreting}. 
Here we provide experimental details into the analysis we conducted, a comprehensive explanation of how, we argue, these results should be understood, and, additionally, an alternative brute-force approach to assess the FRMs sensitivity to attribute-constrained modifications.

\paragraph{Quantitative results.} 
Table~\ref{tab:deltas_pvals} is an extended version of Table~\ref{tab:deltas_pgd}, reporting the statistical details of our analysis.
In particular, we show the \textit{p}-values corresponding to each pair-wise comparison in each ranking, and the number of samples on which the comparisons were run (\ie $n$, equivalent to the number of adversarial examples found by PGD).
With a significance of $\alpha = 0.05$, we found most comparisons to be statistically significant.
We mark these statistically-significant comparisons with ``\cmark'', and mark the rest with ``\xmark''.

\paragraph{Interpretation.}
Under the assumption that the induced perturbations indeed result in identity-preserving modifications, an adversarial example for person \underline{A} can be interpreted as \textit{``a variant of \underline{A}'s face that the FRM fails at recognizing as \underline{A}''}.
Since these examples were found via a constrained-perturbation attack, the relative energy spent by the attack on modifying each attribute relates to the FRM's sensitivity to such attribute.
Thus, we argue that the relative energy spent on modifying an attribute can be interpreted as related to \textit{``the FRM's disproportionate sensitivity to modifications on such attribute''}.

From this ranking, we mark three main observations that hold for all FRMs: \textit{(i)} the ``Eyeglasses'' attribute leads the ranking in 1$^{\text{st}}$ position, \textit{(ii)} the ``Pose'' and ``Age'' attributes take either the 2$^{\text{nd}}$ or 3$^{\text{rd}}$ position, and \textit{(iii)} the ``Smile'' and ``Gender'' attributes take the last two positions (4$^{\text{th}}$ and 5$^{\text{th}}$).

We next interpret these observations.
The presence/absence of eyeglasses is a strong cue on which FRMs rely on (somewhat disproportionately) for recognizing individuals.
This phenomenon is inconvenient and suggests avenues for improving FRMs.
However, we also relate the FRMs' sensitivity to eyeglasses with the difficulty humans experience when recognizing a person who recently started/stopped wearing eyeglasses.
Since humans are a strong baseline for FRMs, we argue that requesting FRMs to not rely on such cue may prove unreasonable.
Most likely, thus, the most practical solutions to this problem are to either \textit{(i)} store images of the person's face with and without eyeglasses in the database, or \textit{(ii)} always ask the person to remove eyeglasses before using an FRM.
The inconvenience introduced by either solution calls for more sophisticated ways of handling the FRMs' sensitivity to eyeglasses.

The FRMs' large sensitivity to the next two attributes (pose and age) is interesting.
For age, analogous to our previous comment on eyeglasses, humans also have trouble recognizing faces when the person's age strongly varies between the face they were originally presented with and the new face (\eg, when seeing someone's old picture or when seeing them after a large amount of time).
That is, age can introduce salient changes to faces that hinder the human capacity to recognize them.
We further found, as judged by our qualitative inspection, that while GANs provide remarkably plausible examples of how a person might look younger/older, the quality of these examples degrades significantly when the person's apparent age is close to that of a child.
This phenomenon, we argue, stems from the fact that facial cues vary drastically during the transition from child to adult.
Regarding pose, we find that indeed some FRMs can be fooled by almost exclusively modifying the face's pose.
However, we also observe that the FRMs' failure in some of these cases can be mostly attributed to a malfunction of the GAN, which introduces evident artifacts when queried with generating ``extreme'' poses.

Finally, the smile and gender attributes fall last in the ranking.
Thus, compared to other attributes, \textit{neither} smile nor gender are attributes to which FRMs are disproportionately sensitive.
This result can be understood as a satisfying result.
Under the constrained budget we allow for the attack, modifying either smile or gender is largely ineffective: altering either attribute \textit{such that} the FRM is fooled would require a magnitude that is simply not attainable given the constraint.
Changing a person's smile such that an FRM changes its prediction is impractical.
Similarly, changing a person's gender under a constrained budget is virtually unfeasible.

\begin{table*}[]
\centering
\caption{\textbf{Ranking of PGD's per-attribute energy spent.}
For each method, we report the attribute ranking we obtain.
We also report the \textit{p}-value associated with each pair-wise comparison of neighboring attributes (\ie the 1$^\text{st}$ with the 2$^\text{nd}$, the 2$^\text{nd}$ with the 3$^\text{rd}$, \textit{etc.}), and the number of samples, $n$, on which the tests were run.
\label{tab:deltas_pvals}}
\begin{tabular}{l|clllllllll|c}
\toprule
\multirow{2}{*}{Method} & \multicolumn{10}{c|}{Ranking} &  \\
                                              & \multicolumn{2}{C{2cm}}{1$^{\text{st}}$} & \multicolumn{2}{C{2cm}}{2$^{\text{nd}}$} & \multicolumn{2}{C{2cm}}{3$^{\text{rd}}$} & \multicolumn{2}{C{2cm}}{4$^{\text{th}}$} & \multicolumn{2}{C{2cm}}{5$^{\text{th}}$} & \multicolumn{1}{|c}{$n$}      \\ \hline
\multicolumn{1}{l|}{\multirow{2}{*}{ArcFace}} &  \multicolumn{1}{C{1cm}}{} & \multicolumn{2}{C{2cm}}{\cmark\tiny{$2.61\times10^{-10}$}}  & \multicolumn{2}{C{2cm}}{\cmark\tiny{$5.50\times10^{-3}$}}  & \multicolumn{2}{C{2cm}}{\cmark\tiny{$1.93\times10^{-13}$}}  & \multicolumn{2}{C{2cm}}{\cmark\tiny{$4.47\times10^{-21}$}} & \multicolumn{1}{C{1cm}}{}        & \multicolumn{1}{|c}{\multirow{2}{*}{753}} \\
\multicolumn{1}{l|}{}                         & \multicolumn{2}{c}{Eyeglasses}           & \multicolumn{2}{c}{Pose} & \multicolumn{2}{c}{Age} & \multicolumn{2}{c}{Smile} & \multicolumn{2}{c}{Gender}                                                                                                                                      & \multicolumn{1}{|c}{} \\\hline
\multicolumn{1}{l|}{\multirow{2}{*}{FaceNet$^C$}} &  \multicolumn{1}{C{1cm}}{} & \multicolumn{2}{C{2cm}}{\cmark\tiny{$3.79\times10^{-28}$}}  & \multicolumn{2}{C{2cm}}{\cmark\tiny{$1.12\times10^{-4}$}}  & \multicolumn{2}{C{2cm}}{\xmark\tiny{$1.80\times10^{-1}$}}  & \multicolumn{2}{C{2cm}}{\xmark\tiny{$2.99\times10^{-1}$}} & \multicolumn{1}{C{1cm}}{}& \multicolumn{1}{|c}{\multirow{2}{*}{1154}} \\
\multicolumn{1}{l|}{}                         & \multicolumn{2}{c}{Eyeglasses}           & \multicolumn{2}{c}{Age} & \multicolumn{2}{c}{Pose} & \multicolumn{2}{c}{Gender} & \multicolumn{2}{c}{Smile}                                                                                                                                      & \multicolumn{1}{|c}{} \\\hline
\multicolumn{1}{l|}{\multirow{2}{*}{FaceNet$^V$}} &  \multicolumn{1}{C{1cm}}{} & \multicolumn{2}{C{2cm}}{\cmark\tiny{$7.73\times10^{-23}$}}  & \multicolumn{2}{C{2cm}}{\cmark\tiny{$1.65\times10^{-3}$}}  & \multicolumn{2}{C{2cm}}{\xmark\tiny{$1.19\times10^{-1}$}}  & \multicolumn{2}{C{2cm}}{\xmark\tiny{$1.15\times10^{-1}$}} & \multicolumn{1}{C{1cm}}{}& \multicolumn{1}{|c}{\multirow{2}{*}{1449}} \\
\multicolumn{1}{l|}{}                         & \multicolumn{2}{c}{Eyeglasses}           & \multicolumn{2}{c}{Age} & \multicolumn{2}{c}{Pose} & \multicolumn{2}{c}{Gender} & \multicolumn{2}{c}{Smile}                                                                                                                                      & \multicolumn{1}{|c}{} \\\bottomrule
\end{tabular}
\end{table*}

\paragraph{Brute-force approach with PGD.}
For the purpose of assessing the importance of attributes, we can also consider a brute-force approach.
This approach consists of giving PGD access to modifying a \textit{single} attribute, \ie a single direction, conducting the attack, and recording the FRM's robust accuracy.
The resulting robust accuracy for each attribute should be thus related to how easy PGD finds adversarial examples just by manipulating such attribute.

Table~\ref{tab:pgd_sing_attr}  reports the results of this experiment for each of the FRMs we considered.
From these results we make the following remarks: \textit{(i)} the overall robust accuracies again find a ranking of FRMs by which ArcFace $>$ FaceNet$^C$ $>$ FaceNet$^V$, \textit{(ii)} robustness can go from as high as 100\% (ArcFace when attacking only gender), to as low as 86.2\% (FaceNet$^V$ when attacking eyeglasses), \textit{(iii)} if we rank attributes according to robustness (from low to high), we get the following attribute ranking for all FRMs: Eyeglasses $>$ Age $>$ Pose $>$ Smile $>$ Gender.

The ranking we find via this brute-force procedure shares similarities with the rankings from Table~\ref{tab:deltas_pgd} (and equivalently Table~\ref{tab:deltas_pvals}).
In particular, we highlight how the eyeglasses attribute is in the first place, while the smile and gender attributes are in the last two positions.
Thus, the results from this brute-force approach support our findings from conducting the procedure we presented in Section~\ref{sec:interpreting} and, further, to the ranking's interpretation we provided in Section~\ref{sec:pgd_attacks}.

\begin{table}[]
\centering
\caption{\textbf{Attribute-restricted search for adversaries through PGD.} 
We restrict PGD's attribute-perturbing capacity to each single attribute, and report each FRM's robustness.
\label{tab:pgd_sing_attr}}
\centering
\begin{tabular}{c|ccccc}
\toprule
\multirow{2}{*}{Method} & \multicolumn{5}{c}{Robustness when only attacking} \\
                        & Pose      & Age       & Gender    & Smile     & Eyeglasses    \\ \hline
ArcFace                 & 99.4      & 98.9      & 100       & 99.9      & \textbf{95.6}          \\
FaceNet$^C$      & 99.1      & 95.9      & 99.3      & 99.2      & \textbf{88.7}           \\
FaceNet$^V$      & 98.4      & 93.8      & 98.7      & 98.6      & \textbf{86.2}          \\\bottomrule
\end{tabular}
\end{table}

\section{FAB Qualitative Results}
We report randomly-selected qualitative samples for the adversarial examples found by FAB, when attacking ArcFace (the most robust method, according to our assessment), in Figures~\ref{fig:fab_qual1}--\ref{fig:fab_qual4}.
While the adversarial examples found by FAB are no longer required to preserve identity, we argue there are interesting observations that can be made from these qualitative results.
We next enumerate some of these observations.

\underline{First}, we again see the presence of some artifacts that are introduced by the GAN, either in the original face or in its modified version.
This introduction of artifacts seems to be pervasive, and there is still no standard procedure to remove such artifacts while preserving important semantic content in the image.
\underline{Second}, artifacts associated with children-like faces also surface in this setup.
The age attribute seems to be rather difficult to control without introducing several (and severe) changes in other face attributes.
\underline{Third}, the addition/removal of eyeglasses is also present in many of the adversarial examples, again insisting in how FRMs rely (somewhat disproportionately) on this characteristic.
\underline{Fourth}, most modified images largely preserve the background of the original image; however, this fact does not seem to be taken into account by the FRM.
\underline{Fifth}, there are modified faces that share a remarkable amount of features with the original face, yet have a large perturbation energy; on the other hand, there are modified faces that do not share many features with the original face, yet have small perturbation energy.
These two facts combined suggest that, while our approach to quantifying identity (dis)similarity in StyleGAN's latent space is promising, it may be sub-optimal.
\underline{Sixth}, many of these examples show the intrinsic difficulty of defining what is expected from FRMs: arguably, humans may make the same ``mistakes'' of the FRM, or even make mistakes the FRM would not make.
That is to say, defining when an FRM should change its prediction is a hard task.
\underline{Seventh}, the results shown here also demonstrate that minimum-perturbation attacks, like FAB, may be a viable alternative for diagnosing systems: we can indeed find images that embody unexpected behaviors in FRMs.


\twocolumn
\begin{figure}
    \centering
    \raisebox{0.35in}{\rotatebox[origin=t]{90}{3.13}}\includegraphics[trim=0cm 0.3cm 0cm 0.7cm,clip,width=0.96\columnwidth]{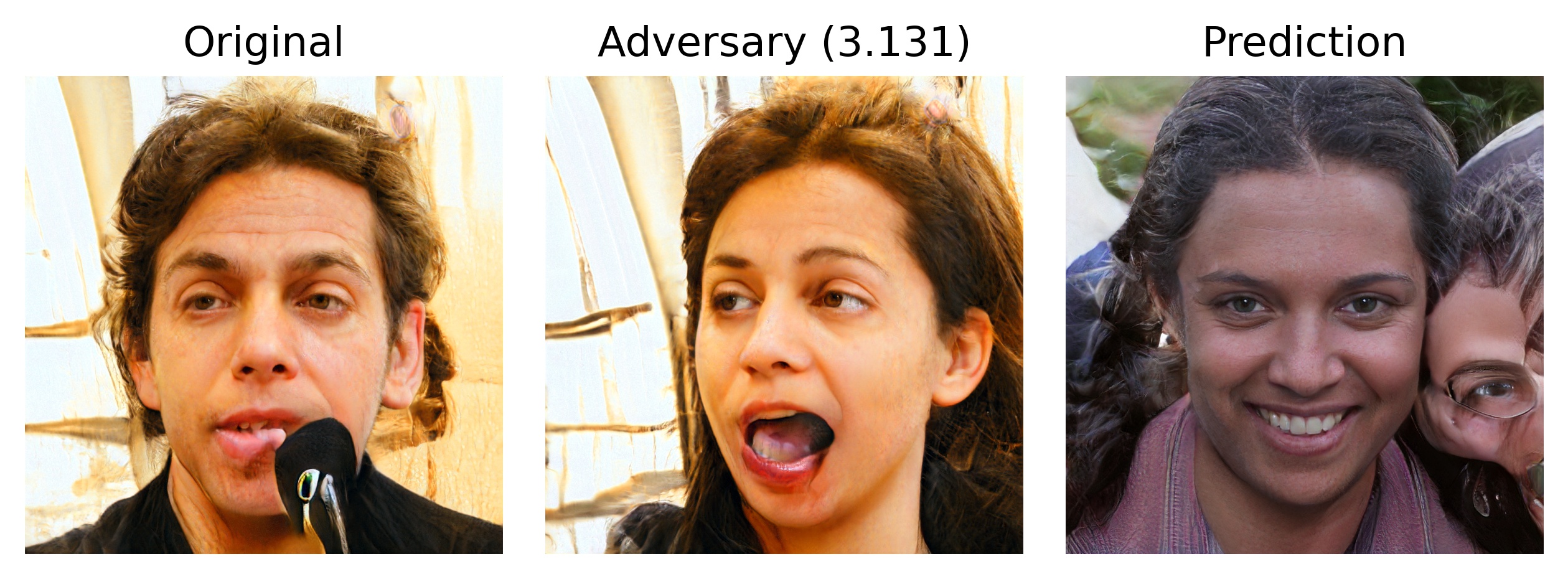}\\
    \raisebox{0.35in}{\rotatebox[origin=t]{90}{4.14}}\includegraphics[trim=0cm 0.3cm 0cm 0.7cm,clip,width=0.96\columnwidth]{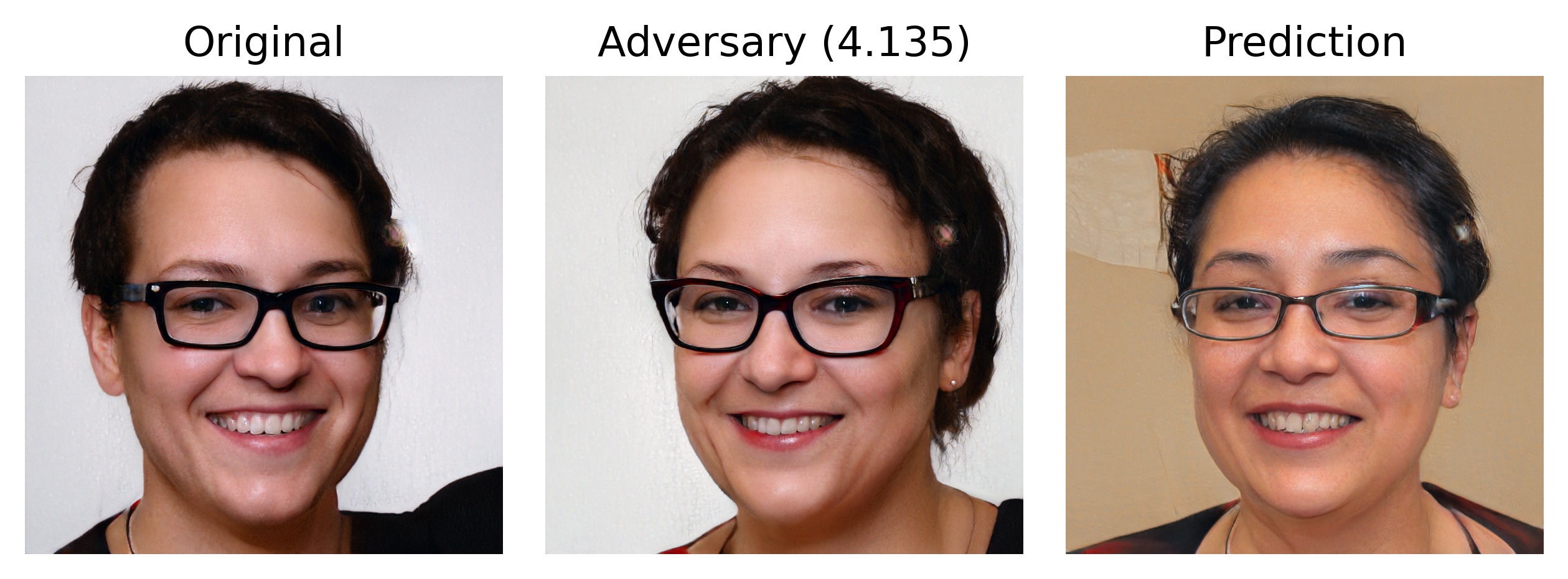}\\
    \raisebox{0.35in}{\rotatebox[origin=t]{90}{2.12}}\includegraphics[trim=0cm 0.3cm 0cm 0.7cm,clip,width=0.96\columnwidth]{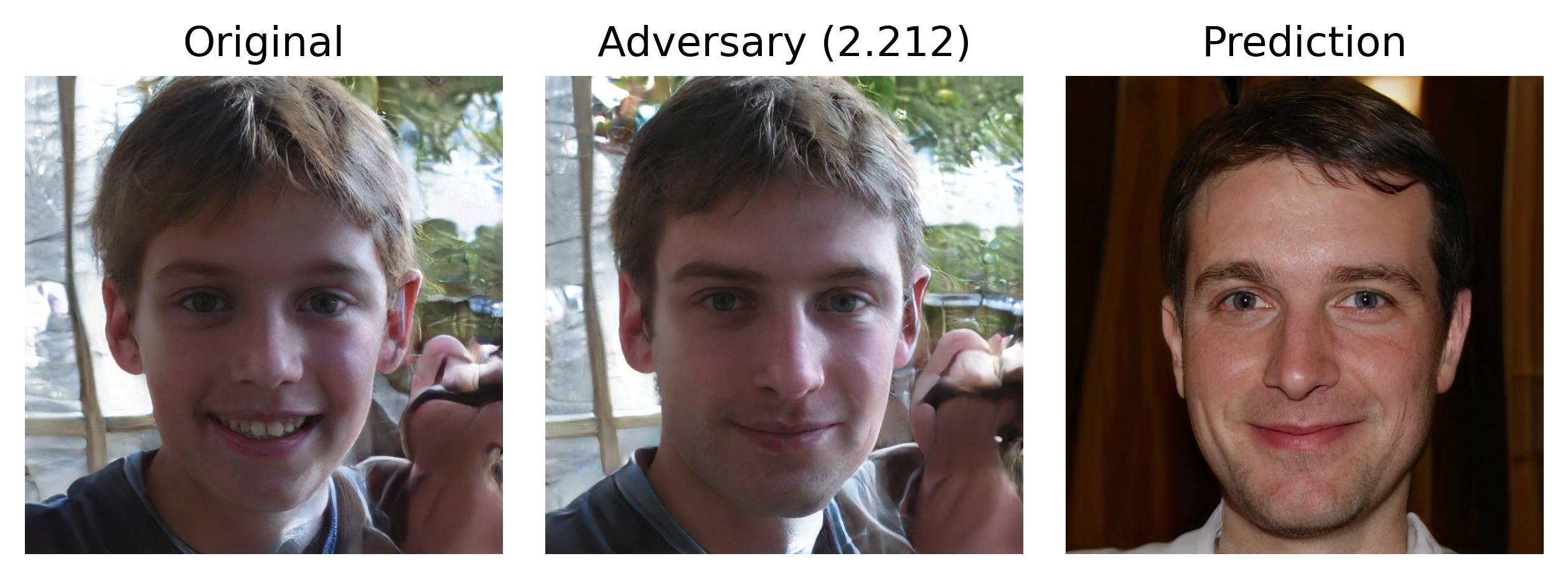}\\
    \raisebox{0.35in}{\rotatebox[origin=t]{90}{5.01}}\includegraphics[trim=0cm 0.3cm 0cm 0.7cm,clip,width=0.96\columnwidth]{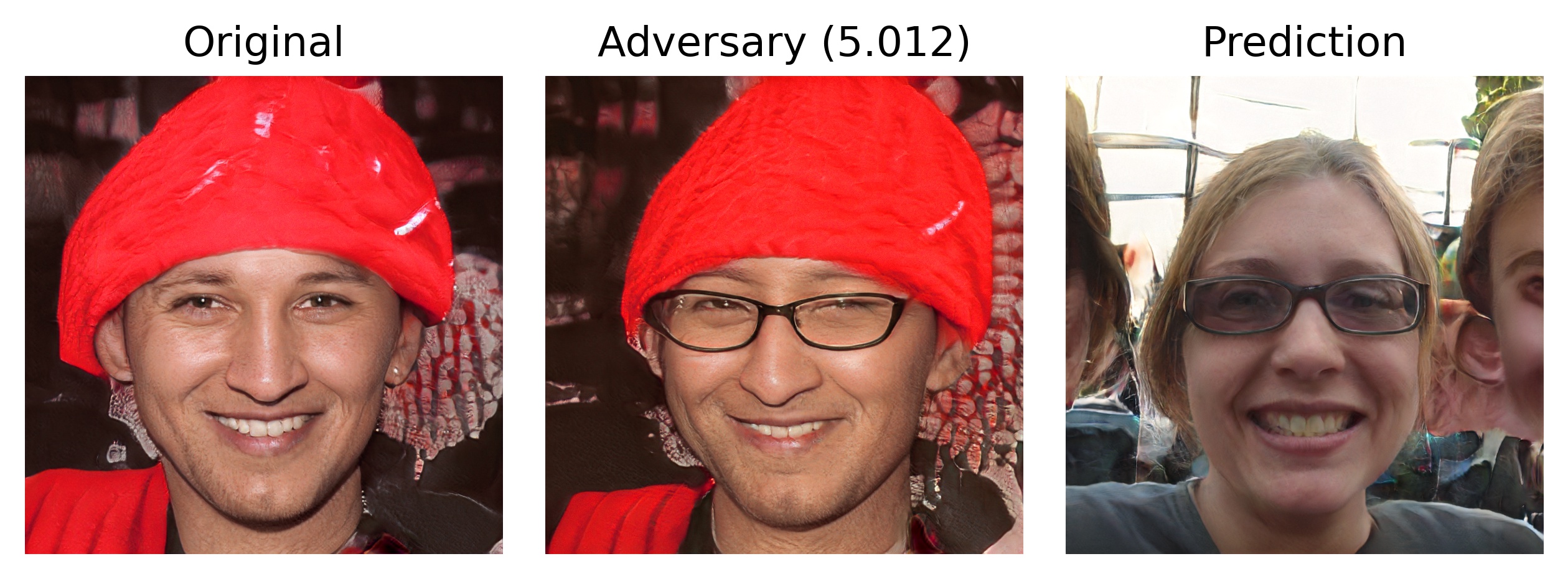}\\
    \raisebox{0.35in}{\rotatebox[origin=t]{90}{4.94}}\includegraphics[trim=0cm 0.3cm 0cm 0.7cm,clip,width=0.96\columnwidth]{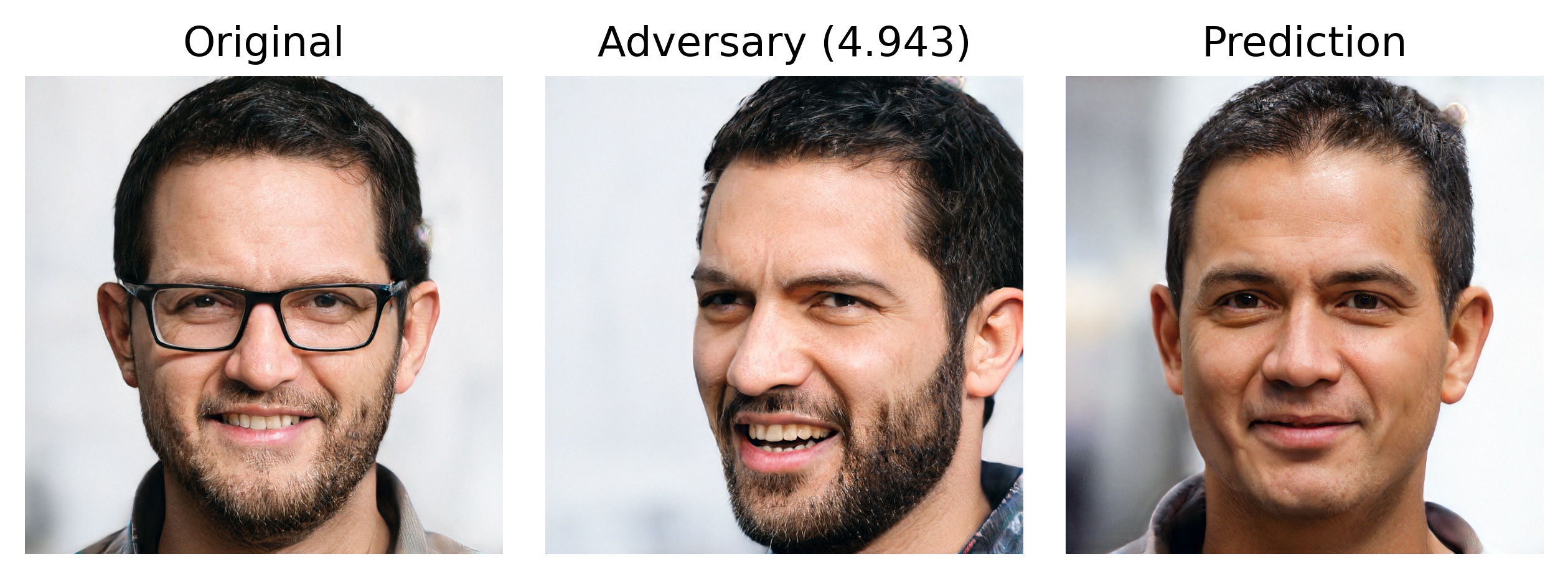}\\
    \raisebox{0.35in}{\rotatebox[origin=t]{90}{2.31}}\includegraphics[trim=0cm 0.3cm 0cm 0.7cm,clip,width=0.96\columnwidth]{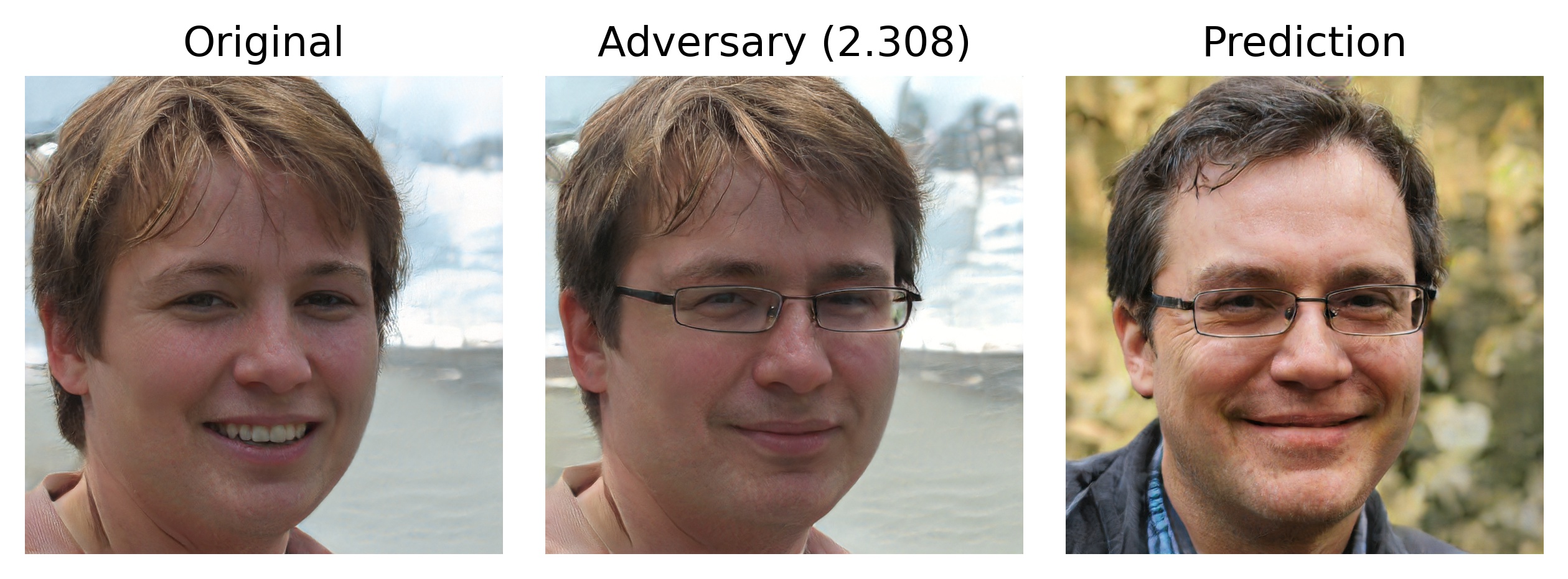}\\
    \raisebox{0.35in}{\rotatebox[origin=t]{90}{3.27}}\includegraphics[trim=0cm 0.3cm 0cm 0.7cm,clip,width=0.96\columnwidth]{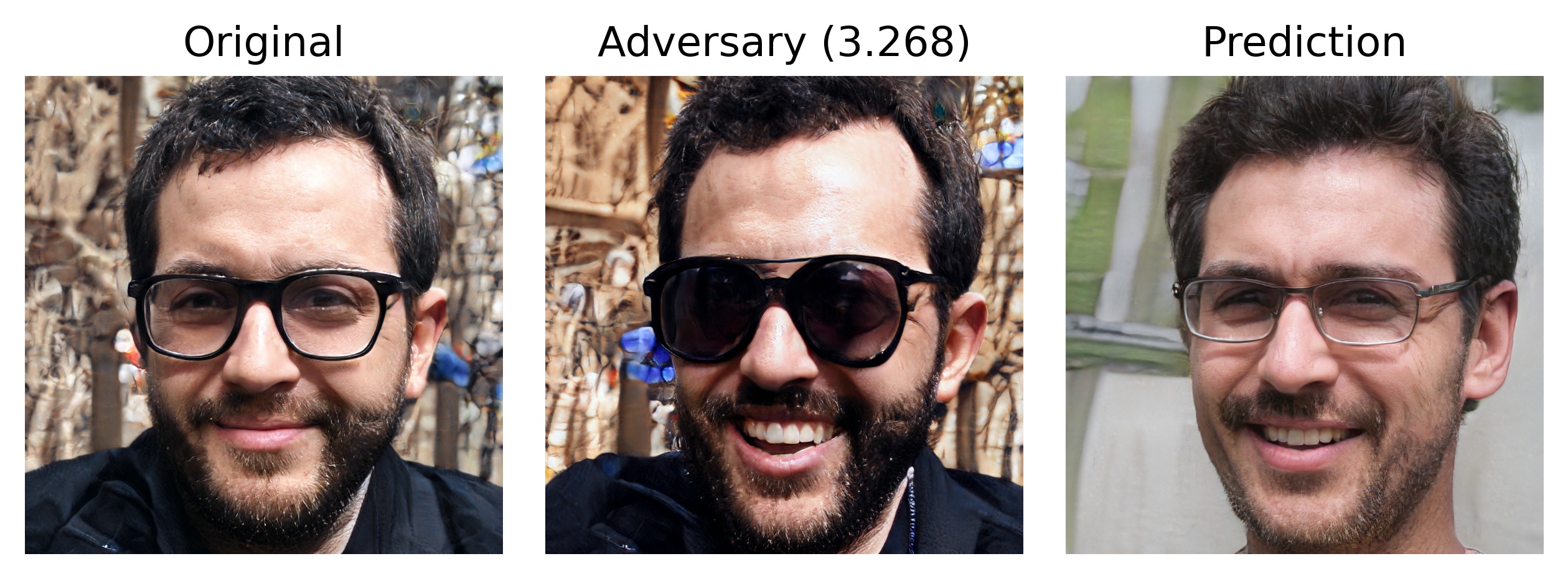}\\
    \raisebox{0.35in}{\rotatebox[origin=t]{90}{4.13}}\includegraphics[trim=0cm 0.3cm 0cm 0.7cm,clip,width=0.96\columnwidth]{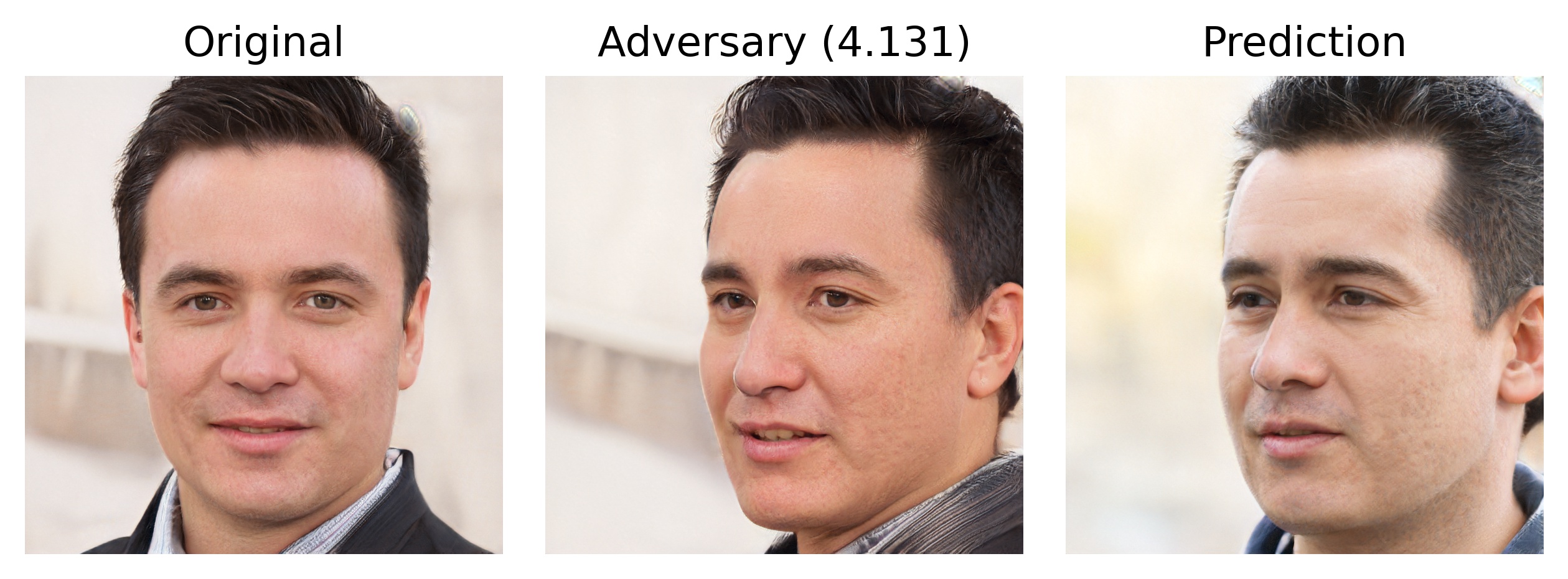}\\
    \caption{\textbf{Adversarial examples found by FAB.} 
    Each row is a different identity. 
    We report each perturbation's energy, $\|\pmb{\delta}\|_{M, 2}$, at the far left. 
    \textit{Left:} original face \textcolor{blue}{\fontfamily{cmss}\selectfont\textbf{A}}, \textit{middle:} modified face \textcolor{orange}{\fontfamily{cmss}\selectfont\textbf{A$^\star$}}, \textit{right:} match \textcolor{purple}{\fontfamily{cmss}\selectfont\textbf{B}}. The FRM prefers to match \textcolor{orange}{\fontfamily{cmss}\selectfont\textbf{A$^\star$}} with \textcolor{purple}{\fontfamily{cmss}\selectfont\textbf{B}} rather than with \textcolor{blue}{\fontfamily{cmss}\selectfont\textbf{A}}.}
    \label{fig:fab_qual1}
\end{figure}

\begin{figure}
    \centering
    \raisebox{0.35in}{\rotatebox[origin=t]{90}{2.45}}\includegraphics[trim=0cm 0.3cm 0cm 0.7cm,clip,width=0.96\columnwidth]{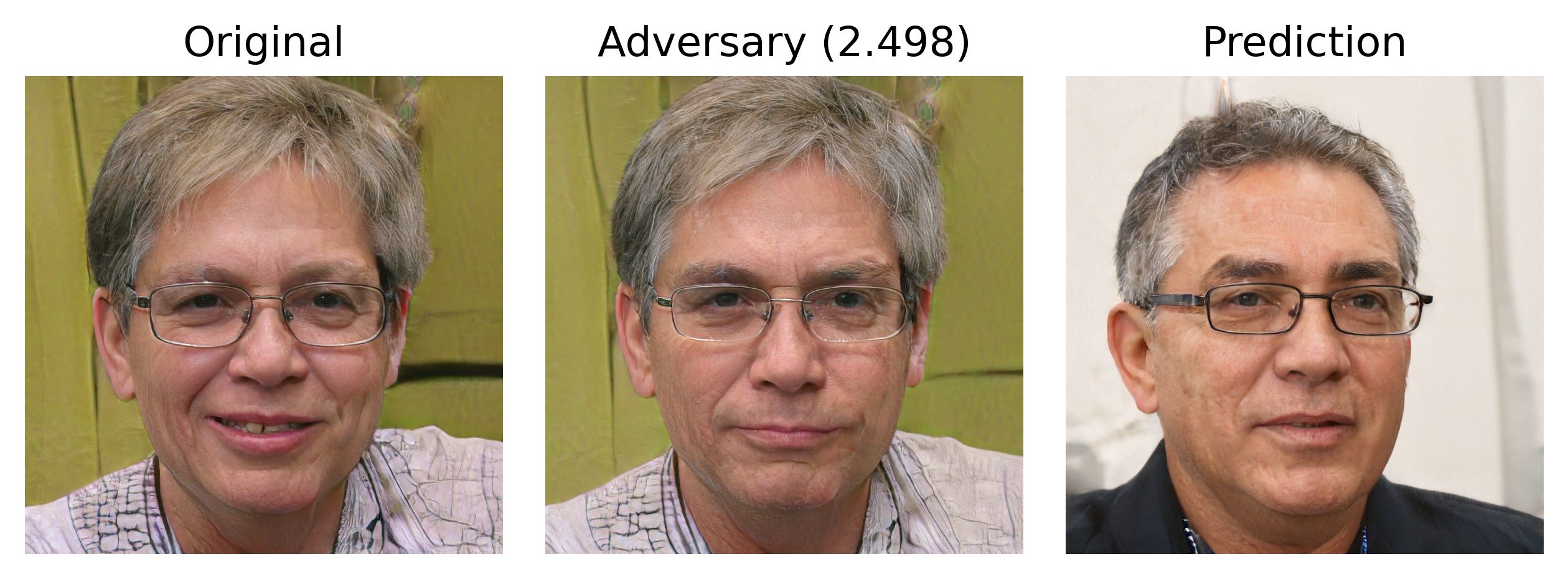}\\
    \raisebox{0.35in}{\rotatebox[origin=t]{90}{2.01}}\includegraphics[trim=0cm 0.3cm 0cm 0.7cm,clip,width=0.96\columnwidth]{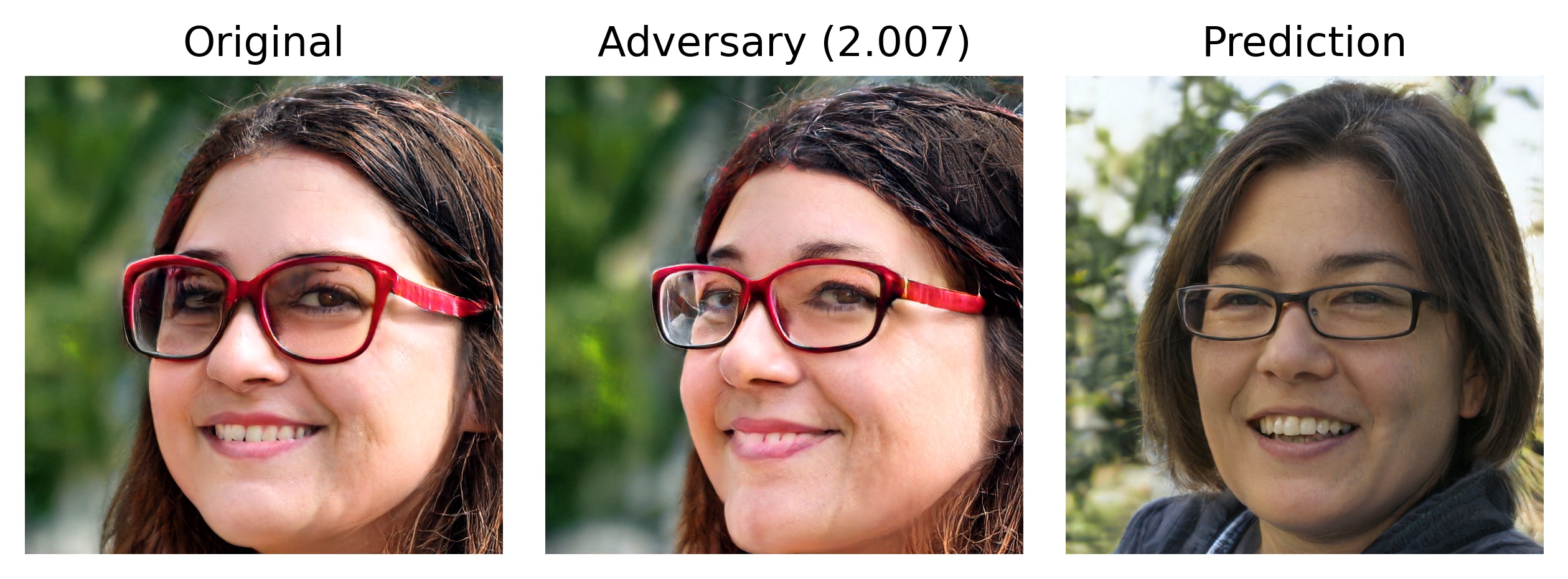}\\
    \raisebox{0.35in}{\rotatebox[origin=t]{90}{1.91}}\includegraphics[trim=0cm 0.3cm 0cm 0.7cm,clip,width=0.96\columnwidth]{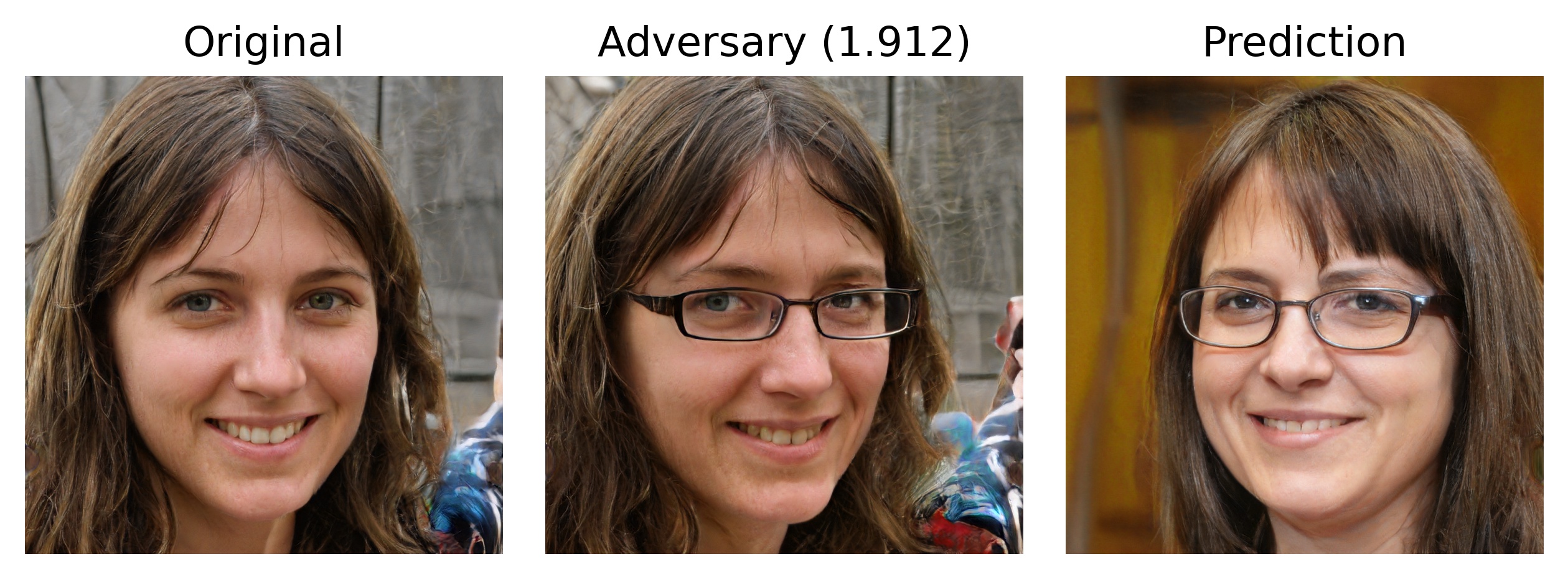}\\
    \raisebox{0.35in}{\rotatebox[origin=t]{90}{2.13}}\includegraphics[trim=0cm 0.3cm 0cm 0.7cm,clip,width=0.96\columnwidth]{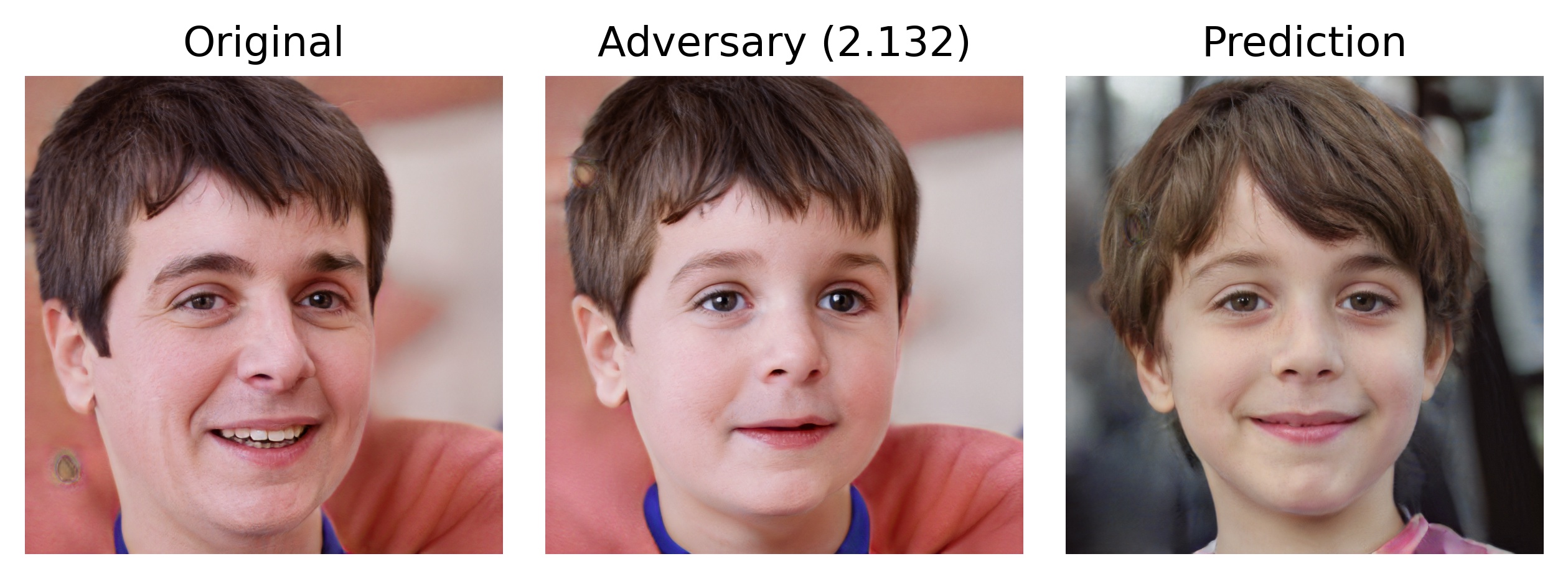}\\
    \raisebox{0.35in}{\rotatebox[origin=t]{90}{2.64}}\includegraphics[trim=0cm 0.3cm 0cm 0.7cm,clip,width=0.96\columnwidth]{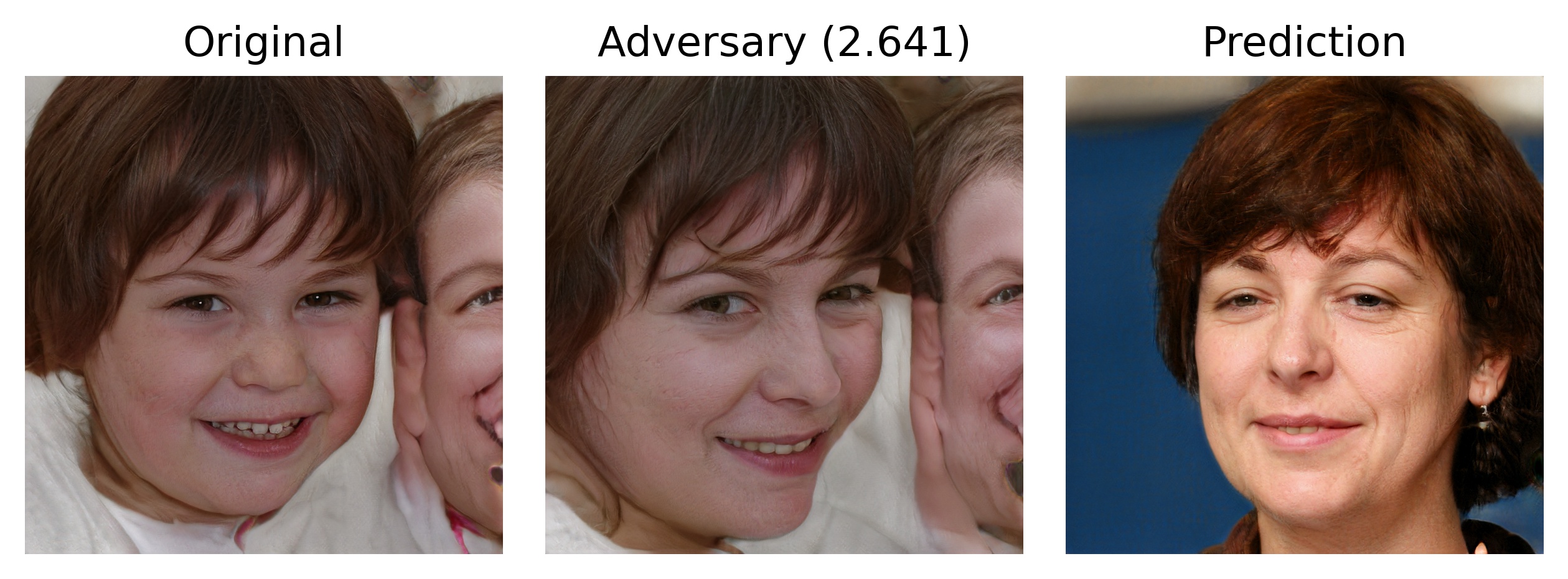}\\
    \raisebox{0.35in}{\rotatebox[origin=t]{90}{3.22}}\includegraphics[trim=0cm 0.3cm 0cm 0.7cm,clip,width=0.96\columnwidth]{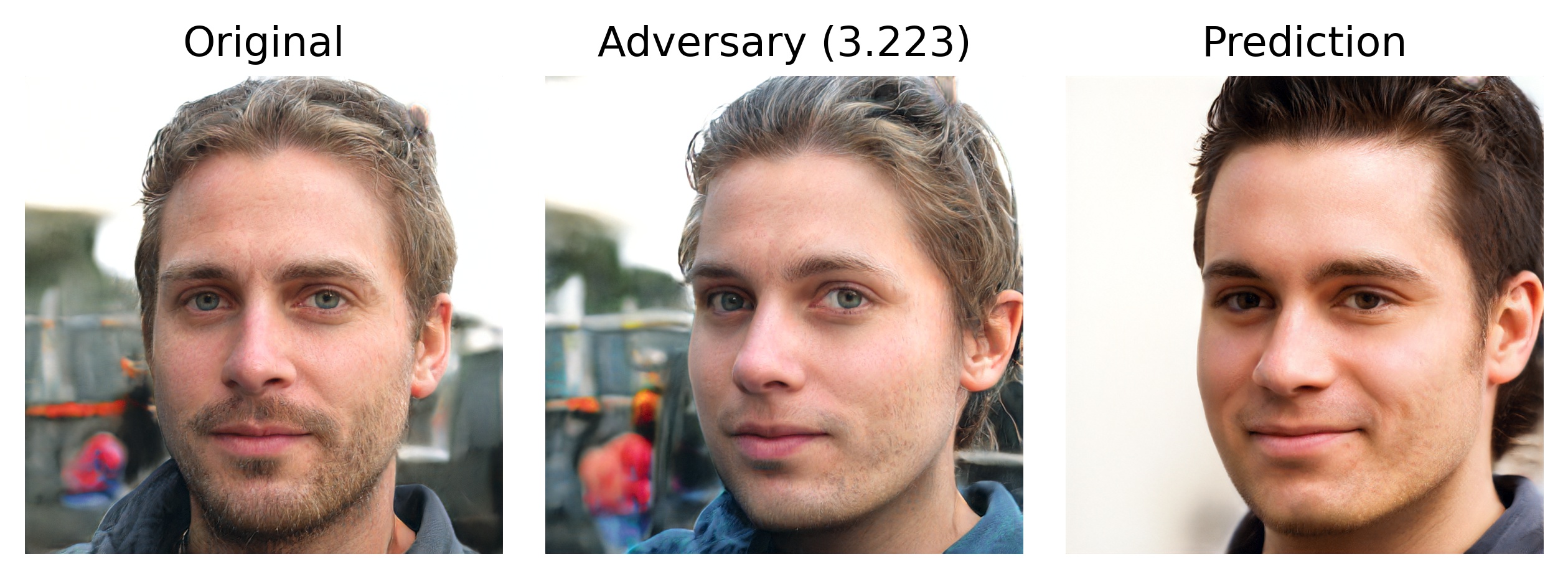}\\
    \raisebox{0.35in}{\rotatebox[origin=t]{90}{3.62}}\includegraphics[trim=0cm 0.3cm 0cm 0.7cm,clip,width=0.96\columnwidth]{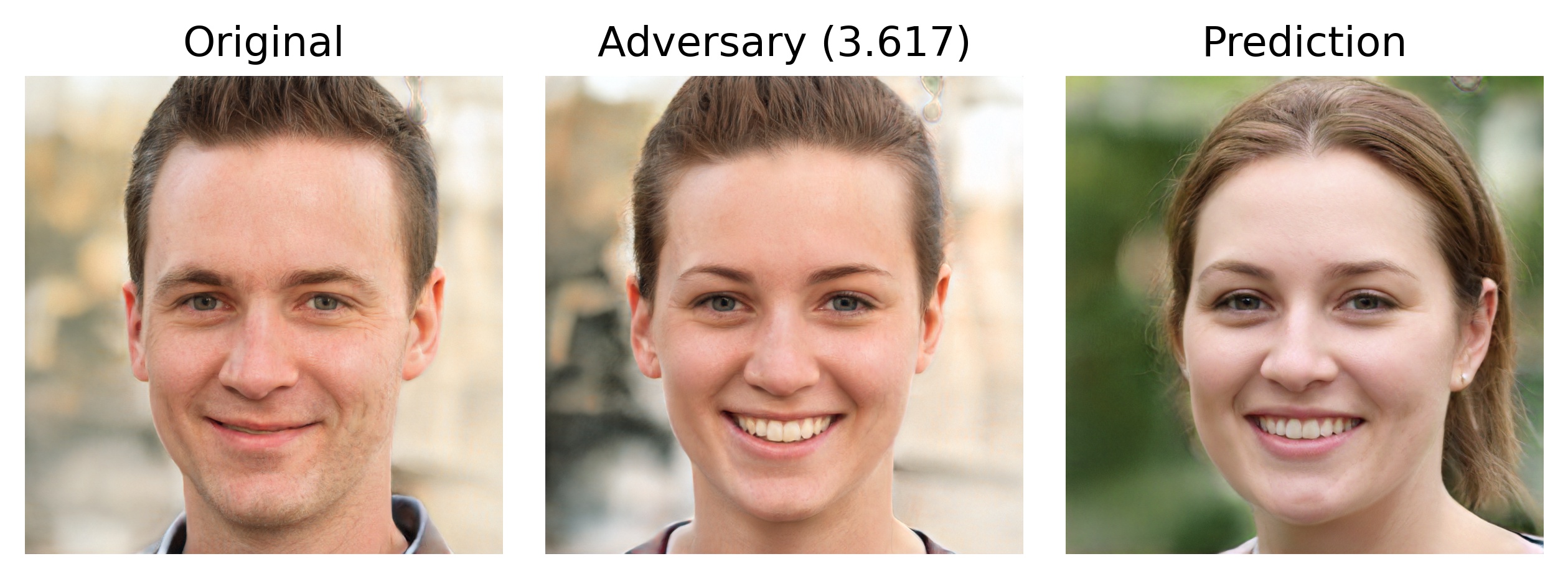}\\
    \raisebox{0.35in}{\rotatebox[origin=t]{90}{3.03}}\includegraphics[trim=0cm 0.3cm 0cm 0.7cm,clip,width=0.96\columnwidth]{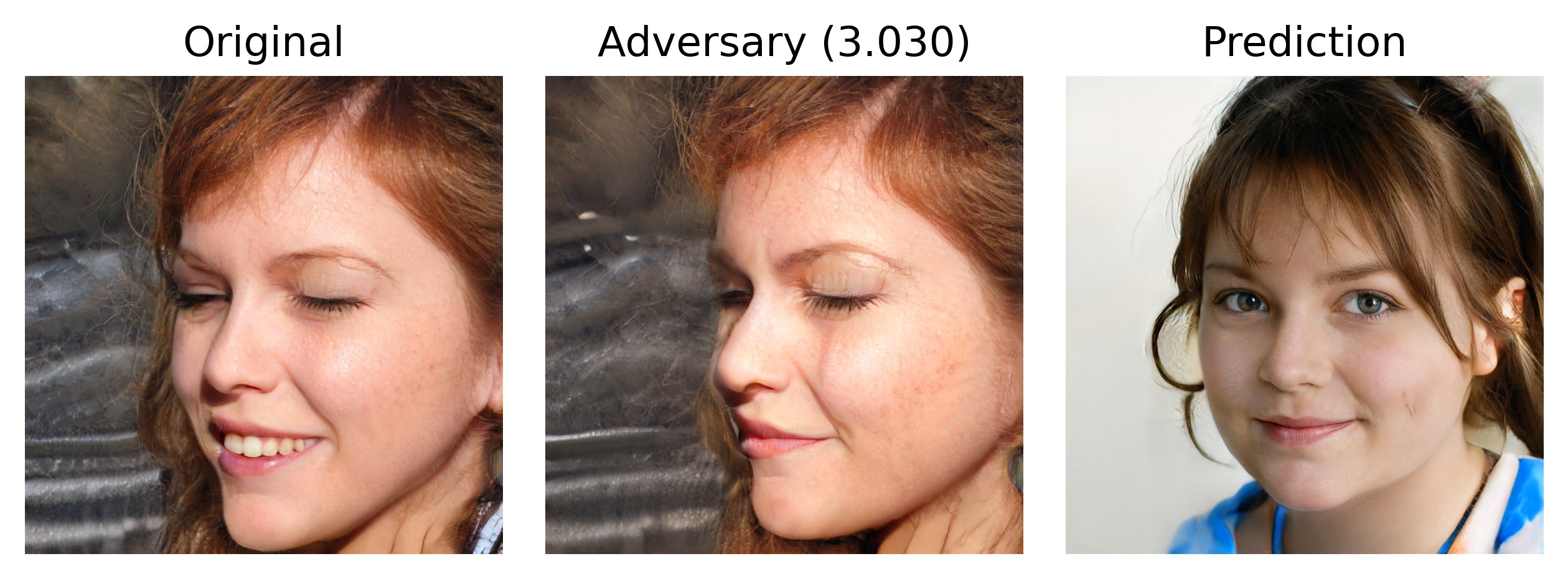}\\
    \caption{\textbf{Adversarial examples found by FAB.} 
    Each row is a different identity. 
    We report each perturbation's energy, $\|\pmb{\delta}\|_{M, 2}$, at the far left. 
    \textit{Left:} original face \textcolor{blue}{\fontfamily{cmss}\selectfont\textbf{A}}, \textit{middle:} modified face \textcolor{orange}{\fontfamily{cmss}\selectfont\textbf{A$^\star$}}, \textit{right:} match \textcolor{purple}{\fontfamily{cmss}\selectfont\textbf{B}}. The FRM prefers to match \textcolor{orange}{\fontfamily{cmss}\selectfont\textbf{A$^\star$}} with \textcolor{purple}{\fontfamily{cmss}\selectfont\textbf{B}} rather than with \textcolor{blue}{\fontfamily{cmss}\selectfont\textbf{A}}.}
    \label{fig:fab_qual2}
\end{figure}

\begin{figure}
    \centering
    \raisebox{0.35in}{\rotatebox[origin=t]{90}{2.00}}\includegraphics[trim=0cm 0.3cm 0cm 0.7cm,clip,width=0.96\columnwidth]{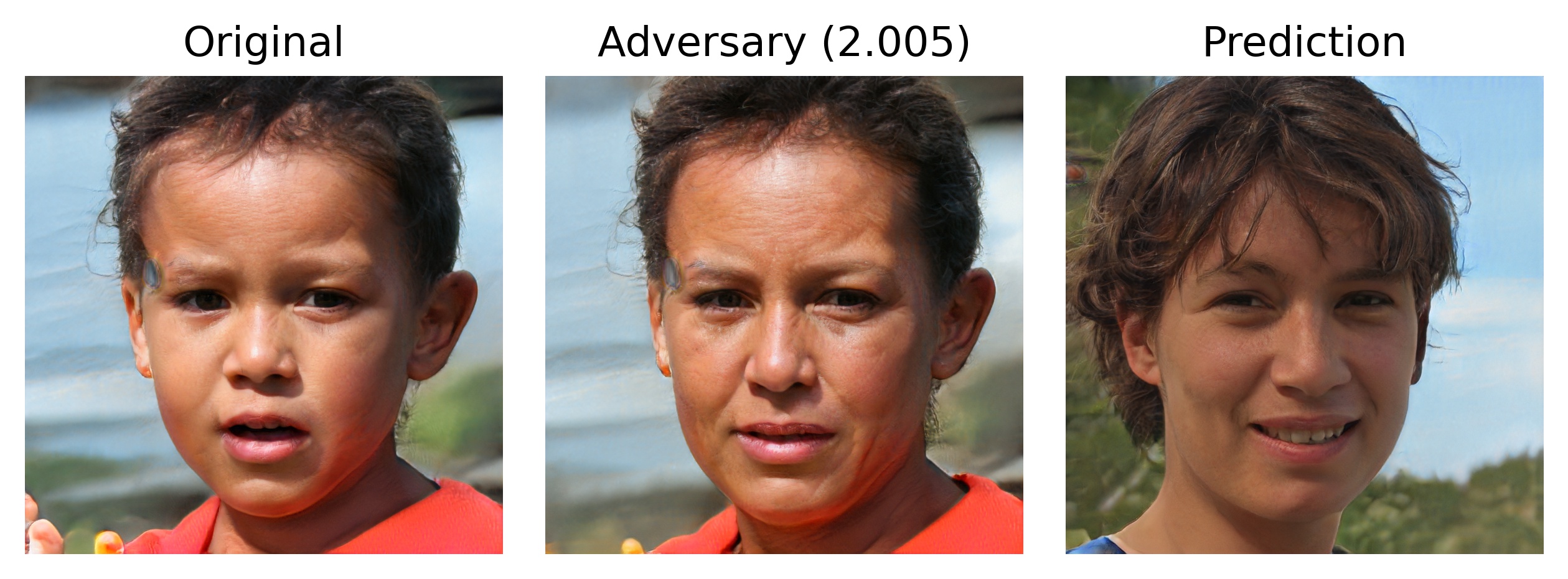}\\
    \raisebox{0.35in}{\rotatebox[origin=t]{90}{4.20}}\includegraphics[trim=0cm 0.3cm 0cm 0.7cm,clip,width=0.96\columnwidth]{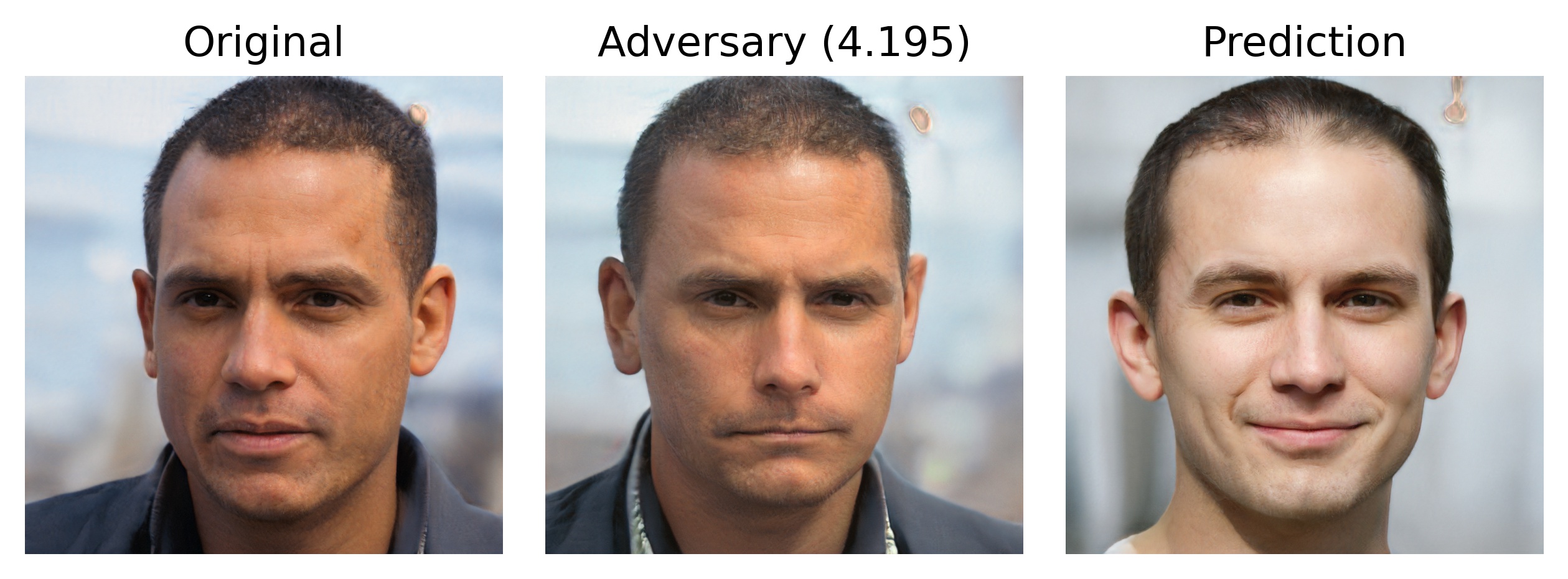}\\
    \raisebox{0.35in}{\rotatebox[origin=t]{90}{4.76}}\includegraphics[trim=0cm 0.3cm 0cm 0.7cm,clip,width=0.96\columnwidth]{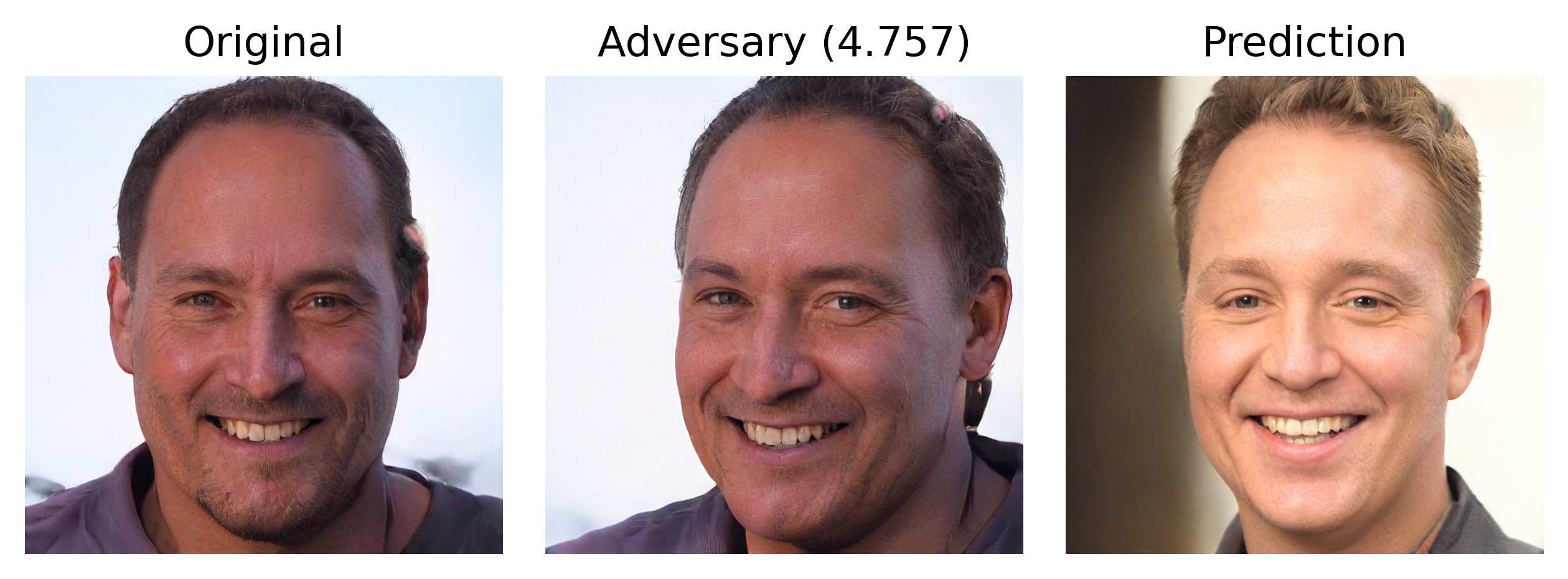}\\
    \raisebox{0.35in}{\rotatebox[origin=t]{90}{3.69}}\includegraphics[trim=0cm 0.3cm 0cm 0.7cm,clip,width=0.96\columnwidth]{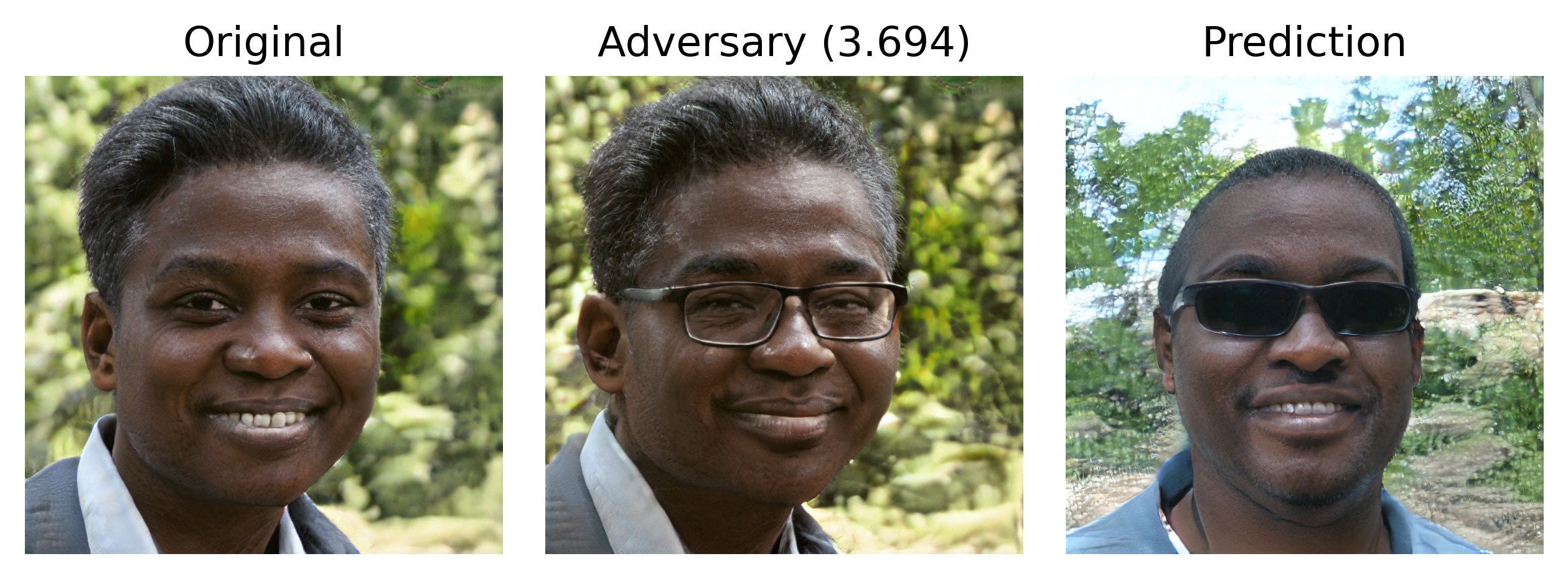}\\
    \raisebox{0.35in}{\rotatebox[origin=t]{90}{2.64}}\includegraphics[trim=0cm 0.3cm 0cm 0.7cm,clip,width=0.96\columnwidth]{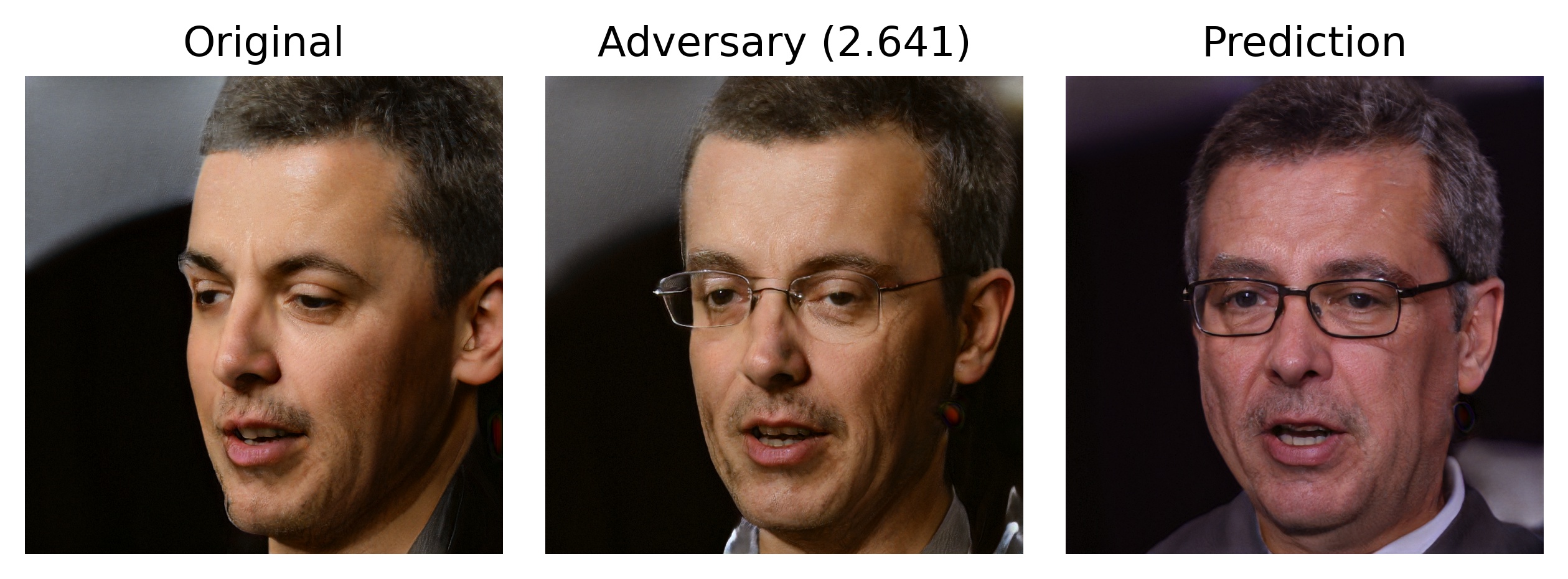}\\
    \raisebox{0.35in}{\rotatebox[origin=t]{90}{4.06}}\includegraphics[trim=0cm 0.3cm 0cm 0.7cm,clip,width=0.96\columnwidth]{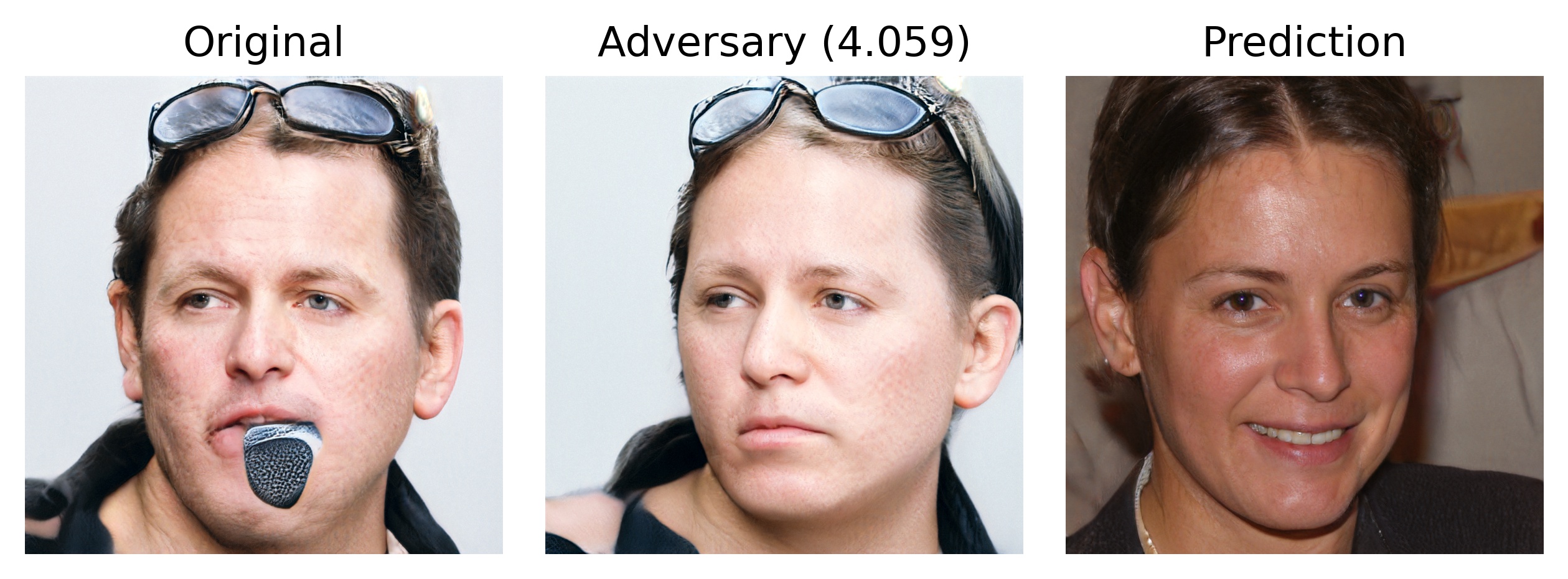}\\
    \raisebox{0.35in}{\rotatebox[origin=t]{90}{1.89}}\includegraphics[trim=0cm 0.3cm 0cm 0.7cm,clip,width=0.96\columnwidth]{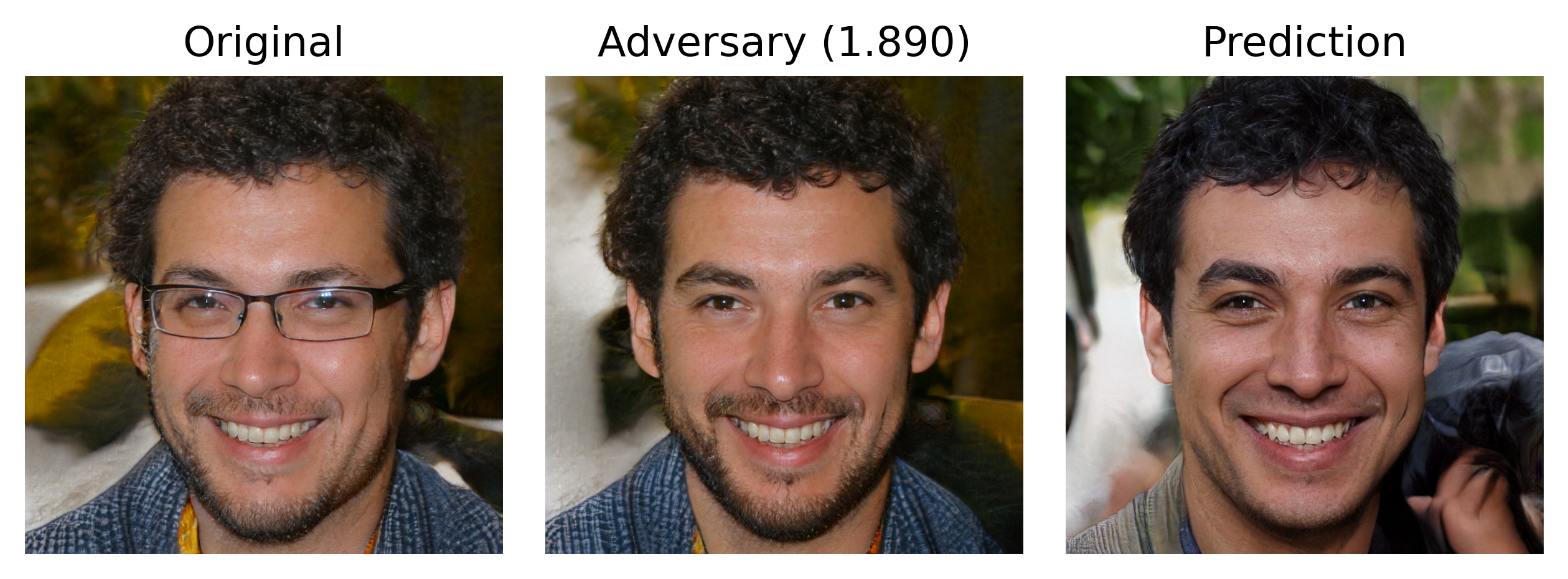}\\
    \raisebox{0.35in}{\rotatebox[origin=t]{90}{4.06}}\includegraphics[trim=0cm 0.3cm 0cm 0.7cm,clip,width=0.96\columnwidth]{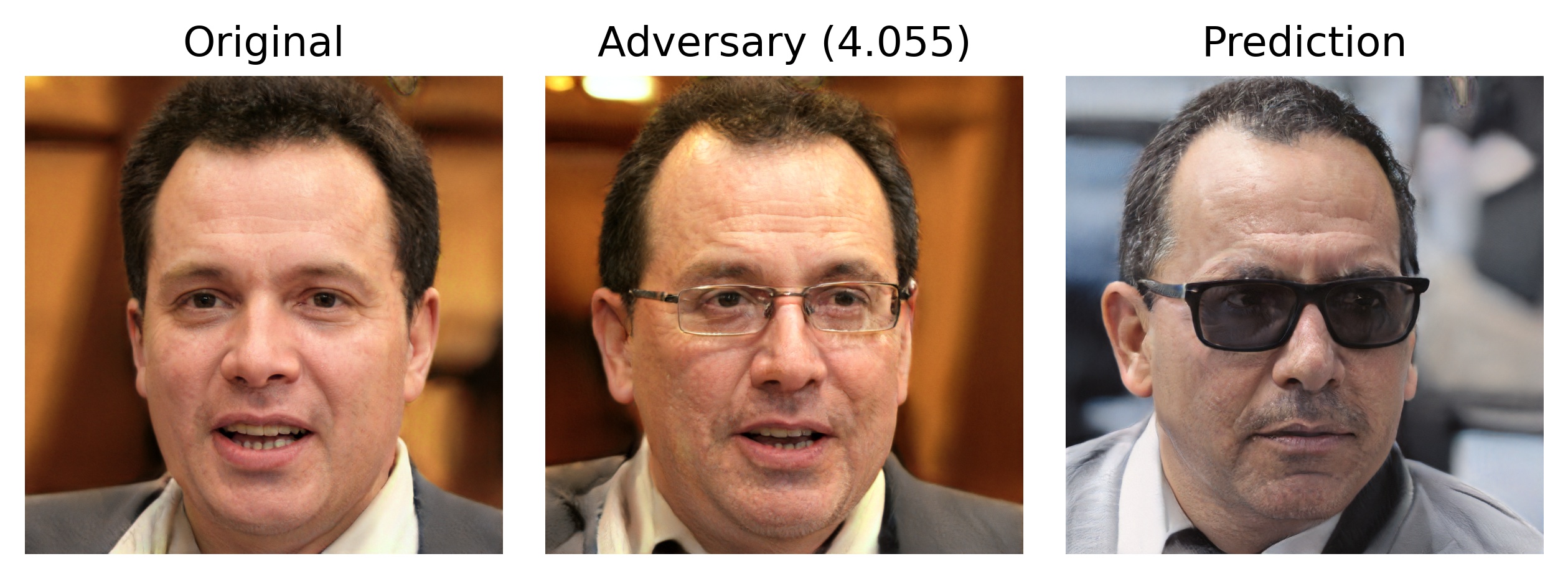}\\
    \caption{\textbf{Adversarial examples found by FAB.} 
    Each row is a different identity. 
    We report each perturbation's energy, $\|\pmb{\delta}\|_{M, 2}$, at the far left. 
    \textit{Left:} original face \textcolor{blue}{\fontfamily{cmss}\selectfont\textbf{A}}, \textit{middle:} modified face \textcolor{orange}{\fontfamily{cmss}\selectfont\textbf{A$^\star$}}, \textit{right:} match \textcolor{purple}{\fontfamily{cmss}\selectfont\textbf{B}}. The FRM prefers to match \textcolor{orange}{\fontfamily{cmss}\selectfont\textbf{A$^\star$}} with \textcolor{purple}{\fontfamily{cmss}\selectfont\textbf{B}} rather than with \textcolor{blue}{\fontfamily{cmss}\selectfont\textbf{A}}.}
    \label{fig:fab_qual3}
\end{figure}

\begin{figure}
    \centering
    \raisebox{0.35in}{\rotatebox[origin=t]{90}{2.43}}\includegraphics[trim=0cm 0.3cm 0cm 0.7cm,clip,width=0.96\columnwidth]{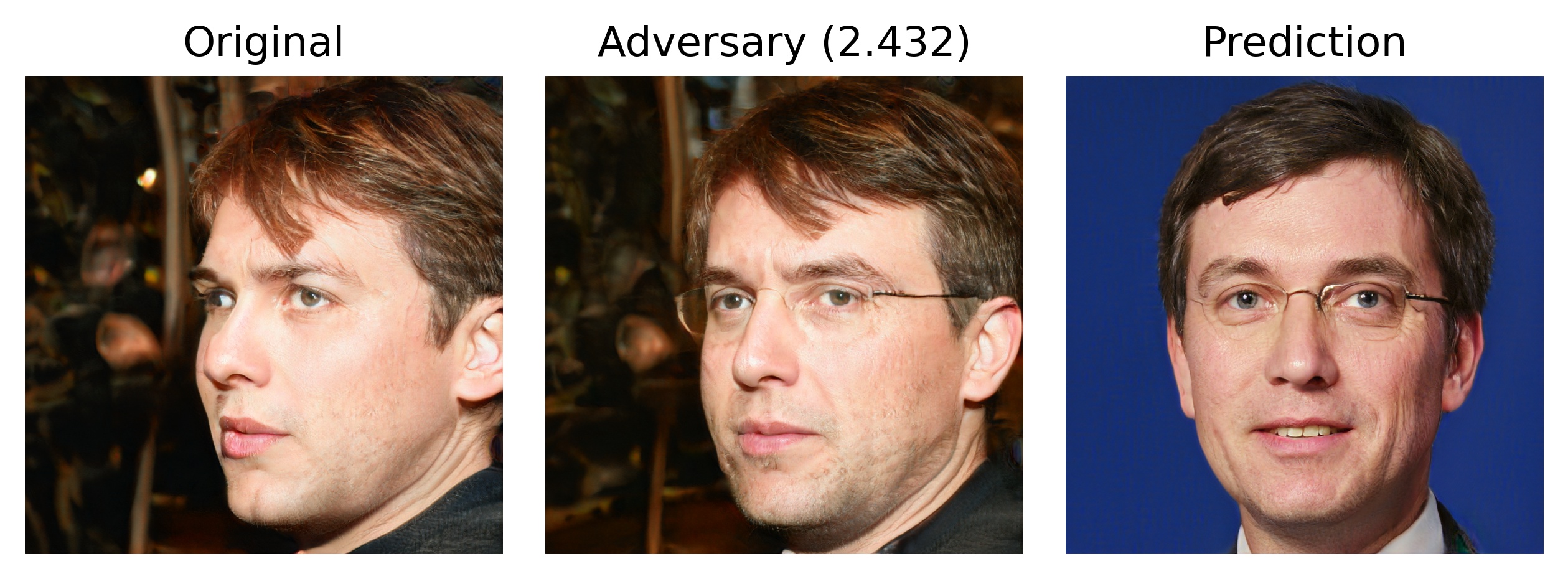}\\
    \raisebox{0.35in}{\rotatebox[origin=t]{90}{4.49}}\includegraphics[trim=0cm 0.3cm 0cm 0.7cm,clip,width=0.96\columnwidth]{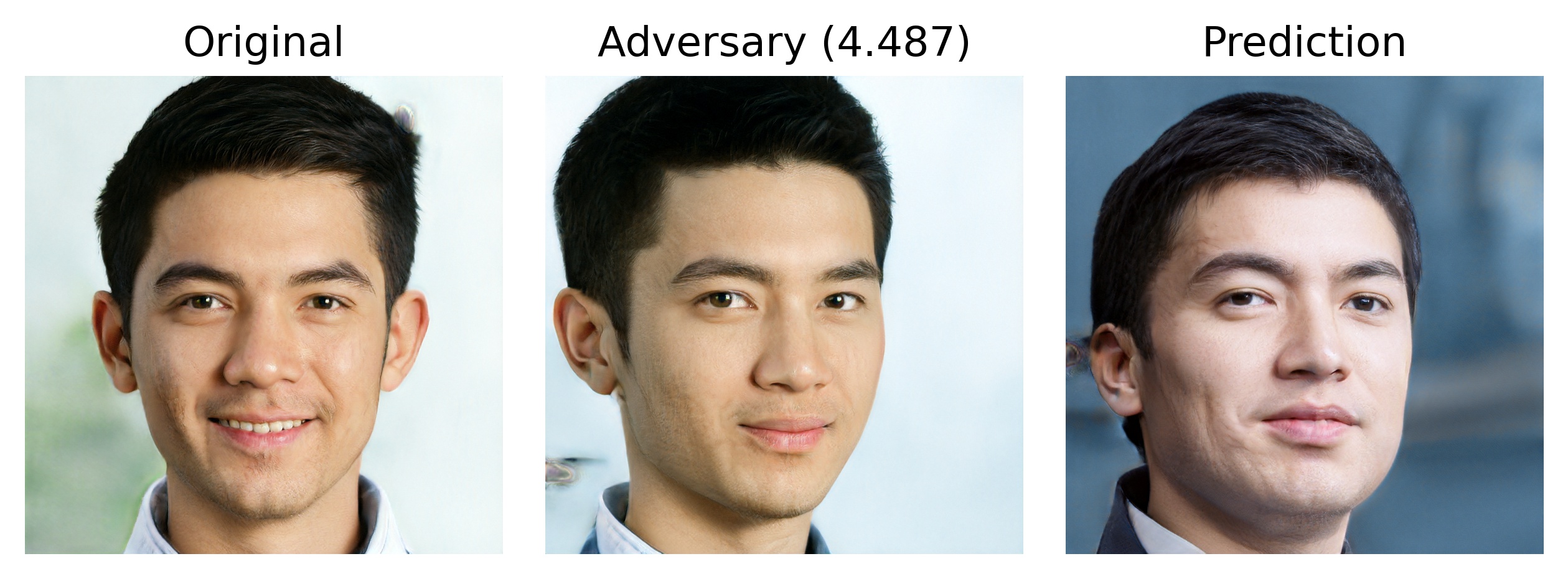}\\
    \raisebox{0.35in}{\rotatebox[origin=t]{90}{2.98}}\includegraphics[trim=0cm 0.3cm 0cm 0.7cm,clip,width=0.96\columnwidth]{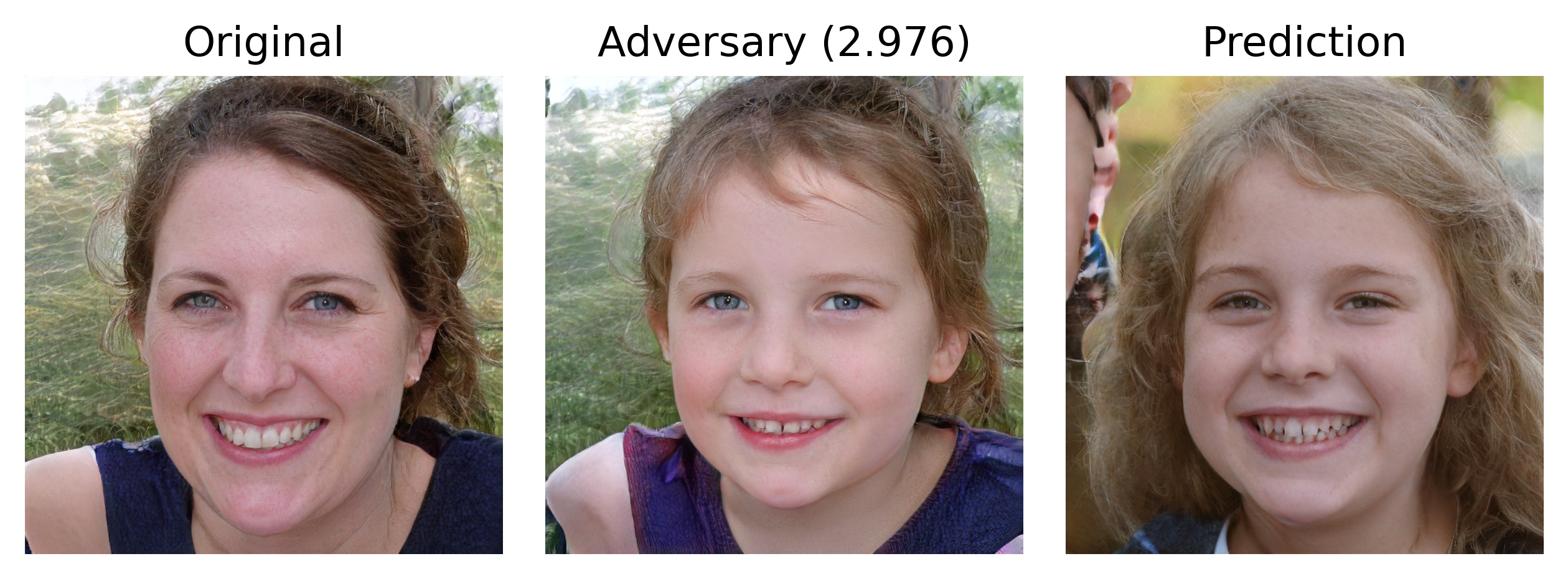}\\
    \raisebox{0.35in}{\rotatebox[origin=t]{90}{3.12}}\includegraphics[trim=0cm 0.3cm 0cm 0.7cm,clip,width=0.96\columnwidth]{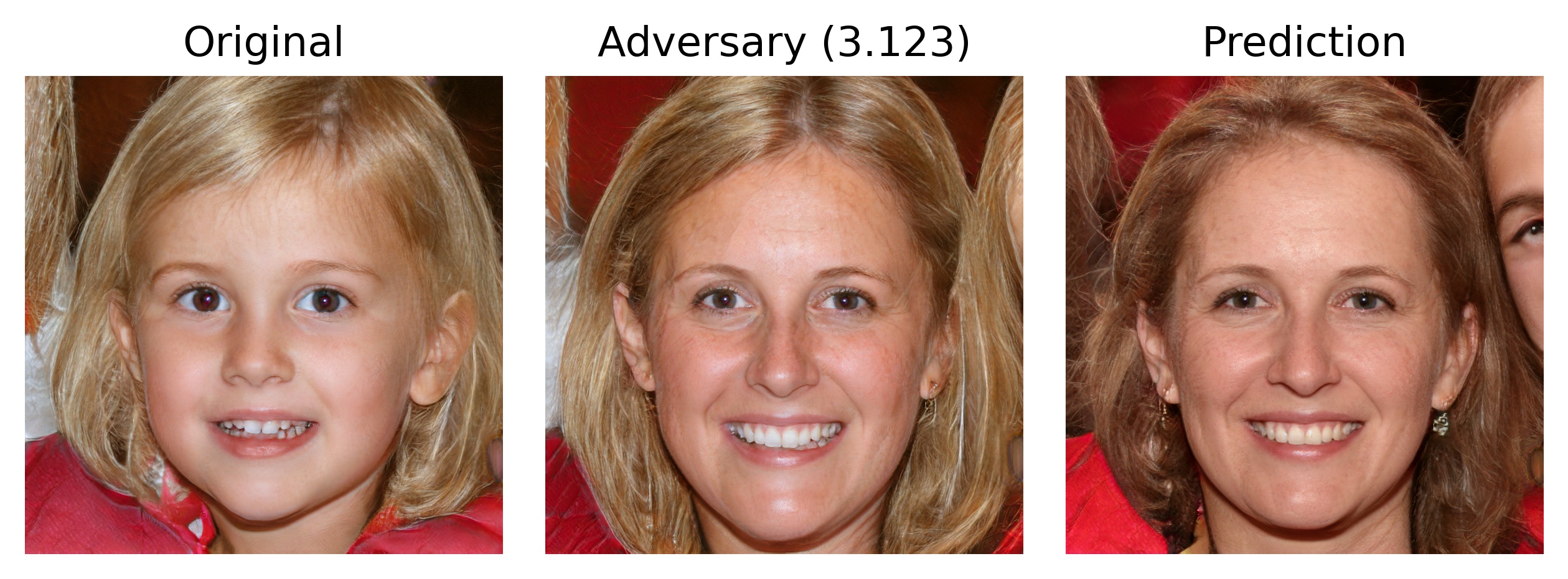}\\
    \raisebox{0.35in}{\rotatebox[origin=t]{90}{2.70}}\includegraphics[trim=0cm 0.3cm 0cm 0.7cm,clip,width=0.96\columnwidth]{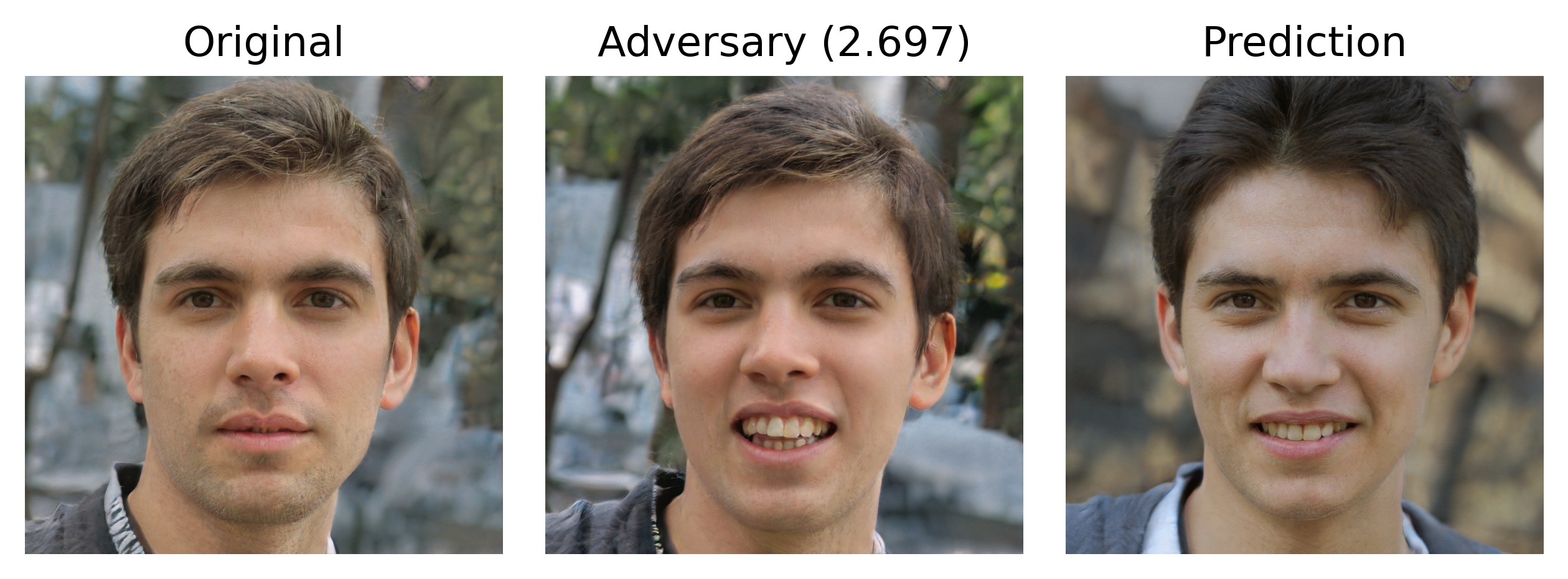}\\
    \raisebox{0.35in}{\rotatebox[origin=t]{90}{2.35}}\includegraphics[trim=0cm 0.3cm 0cm 0.7cm,clip,width=0.96\columnwidth]{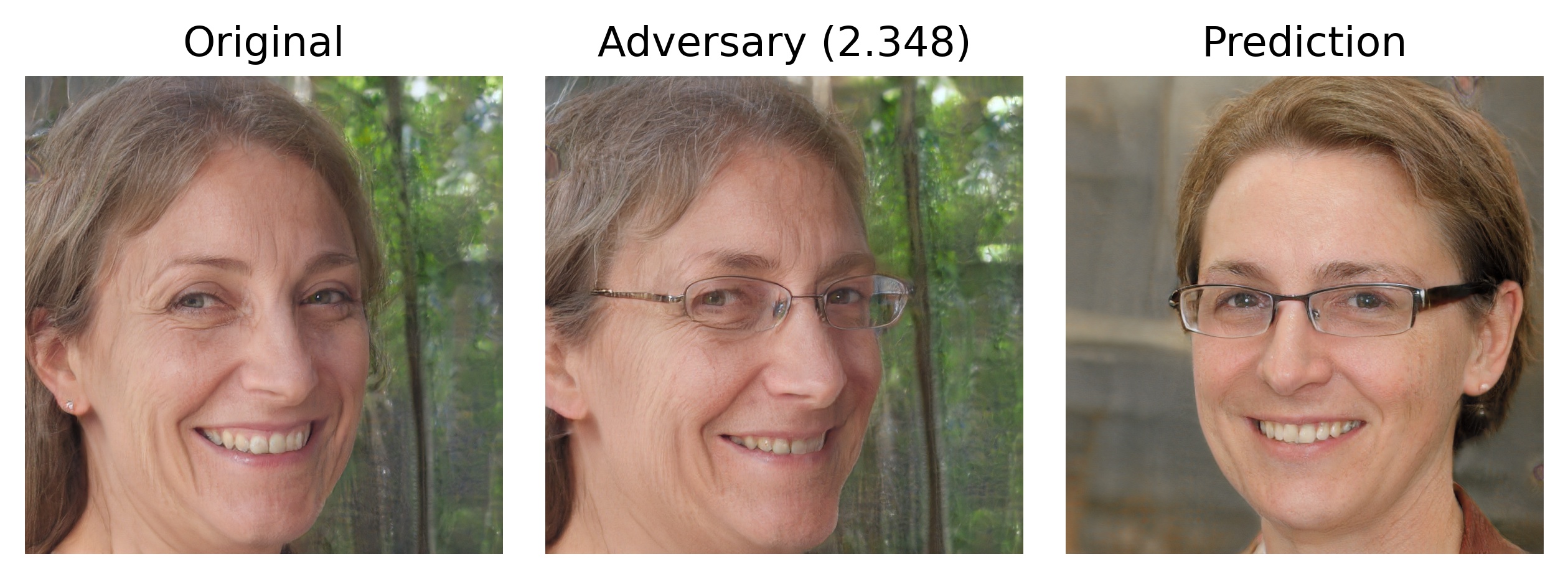}\\
    \raisebox{0.35in}{\rotatebox[origin=t]{90}{1.84}}\includegraphics[trim=0cm 0.3cm 0cm 0.7cm,clip,width=0.96\columnwidth]{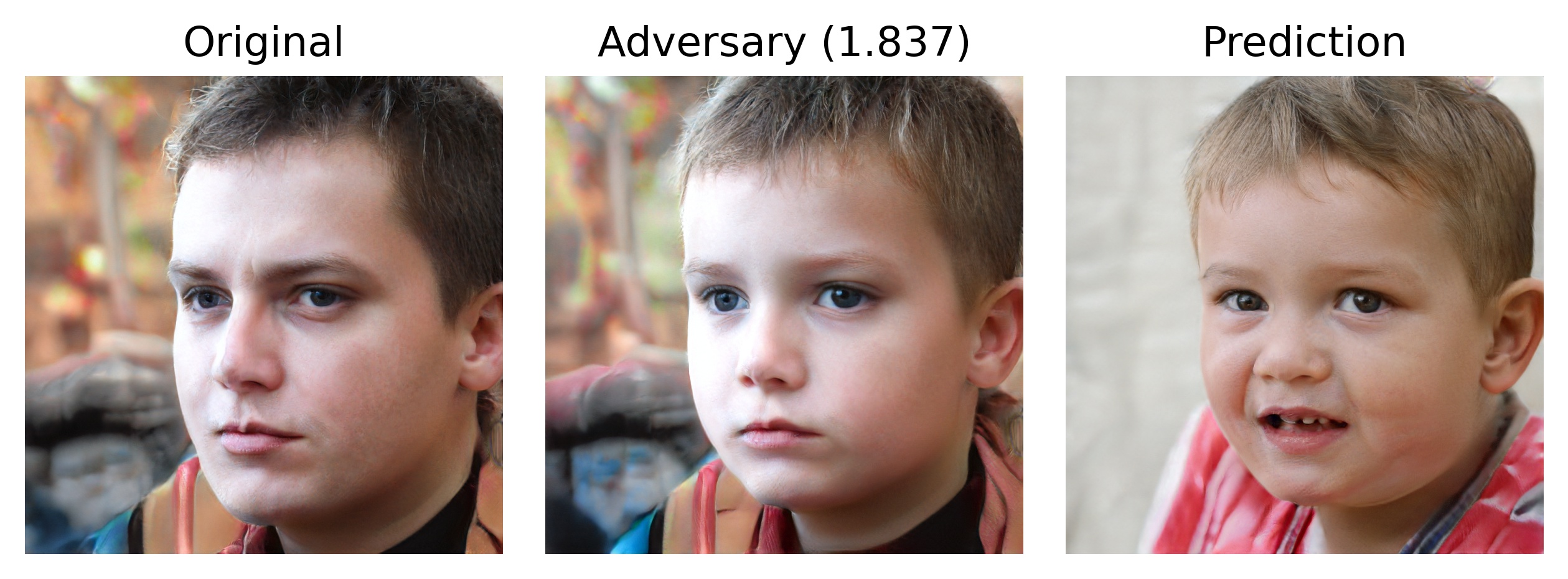}\\
    \raisebox{0.35in}{\rotatebox[origin=t]{90}{1.51}}\includegraphics[trim=0cm 0.3cm 0cm 0.7cm,clip,width=0.96\columnwidth]{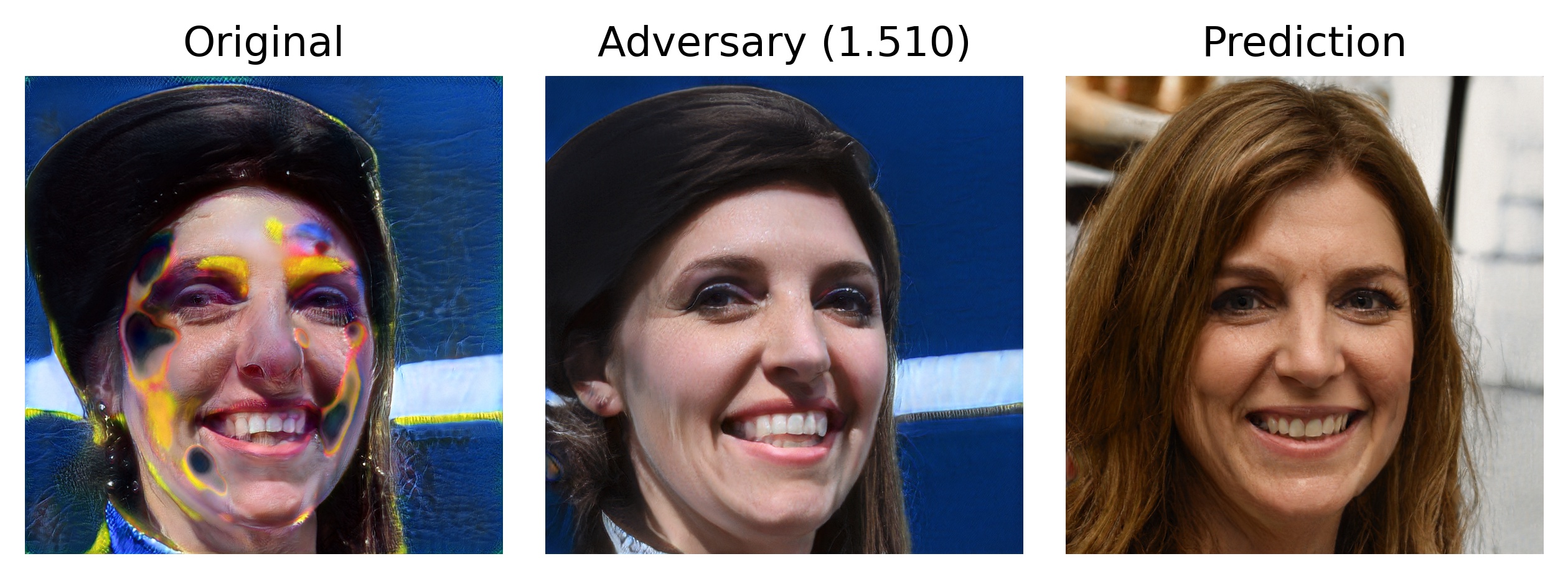}\\
    \caption{\textbf{Adversarial examples found by FAB.} 
    Each row is a different identity. 
    We report each perturbation's energy, $\|\pmb{\delta}\|_{M, 2}$, at the far left. 
    \textit{Left:} original face \textcolor{blue}{\fontfamily{cmss}\selectfont\textbf{A}}, \textit{middle:} modified face \textcolor{orange}{\fontfamily{cmss}\selectfont\textbf{A$^\star$}}, \textit{right:} match \textcolor{purple}{\fontfamily{cmss}\selectfont\textbf{B}}. The FRM prefers to match \textcolor{orange}{\fontfamily{cmss}\selectfont\textbf{A$^\star$}} with \textcolor{purple}{\fontfamily{cmss}\selectfont\textbf{B}} rather than with \textcolor{blue}{\fontfamily{cmss}\selectfont\textbf{A}}.}
    \label{fig:fab_qual4}
\end{figure}
\onecolumn

\section{FAB Attribute Interpretation}
We characterize semantic robustness by analyzing the adversarial examples found by FAB and obtaining a ranking of attributes.

\paragraph{Quantitative results.} 
Table~\ref{tab:deltas_fab_pvals} reports the statistical details of our analysis.
In particular, we show the \textit{p}-values corresponding to each pair-wise comparison in each ranking, and the number of samples on which the comparisons were run (\ie $n$, equivalent to the number of adversarial examples found by FAB).
Note that $n$, for FAB, is equivalent to the entire dataset size, as FAB is able to find adversarial examples for every instance it attacks.
With a significance of $\alpha = 0.05$, we found most comparisons to be statistically significant.
We mark these statistically-significant comparisons with ``\cmark'', and mark the rest with ``\xmark''.

\paragraph{Interpretation.}
Note that, while the procedure we conduct here shares spirit with the one for PGD-based adversarial examples (from Section~\ref{sec:pgd_attacks} and Appendix~\ref{sec:app_pgd_attr}), the resulting ranking \textit{cannot} be interpreted in the same way as the one we obtained when using PGD.
FAB is tasked with finding adversarial examples for \textit{every} identity while minimizing the perturbation induced in the identity.
Consequently, the adversarial examples found by FAB need not belong to the identity of the person the attack was targeting.
That is, an adversarial example for person \underline{A} can be interpreted as \textit{``a face of a person that shares facial attributes with \underline{A} (to some extent) yet is classified as different from \underline{A} by the FRM''}.
Thus, we argue that the relative energy spent on modifying an attribute can be interpreted as related to \textit{``the FRM's reliance on such attribute''}.

\begin{table*}[]
\centering
\caption{\textbf{Ranking of FAB's per-attribute energy spent.}
For each method, we report the attribute ranking we obtain.
We also report the \textit{p}-value associated with each pair-wise comparison of neighboring attributes (\ie the 1$^\text{st}$ with the 2$^\text{nd}$, the 2$^\text{nd}$ with the 3$^\text{rd}$, \textit{etc.}), and the number of samples, $n$, on which the tests were run.
\label{tab:deltas_fab_pvals}}
\begin{tabular}{l|clllllllll|c}
\toprule
\multirow{2}{*}{Method} & \multicolumn{10}{c|}{Ranking} &  \\
                                              & \multicolumn{2}{C{2cm}}{1$^{\text{st}}$} & \multicolumn{2}{C{2cm}}{2$^{\text{nd}}$} & \multicolumn{2}{C{2cm}}{3$^{\text{rd}}$} & \multicolumn{2}{C{2cm}}{4$^{\text{th}}$} & \multicolumn{2}{C{2cm}}{5$^{\text{th}}$} & \multicolumn{1}{|c}{$n$}      \\ \hline
\multicolumn{1}{l|}{\multirow{2}{*}{ArcFace}} &  \multicolumn{1}{C{1cm}}{} & \multicolumn{2}{C{2cm}}{\cmark\tiny{$6.35\times10^{-13}$}}  & \multicolumn{2}{C{2cm}}{\cmark\tiny{$3.00\times10^{-15}$}}  & \multicolumn{2}{C{2cm}}{\cmark\tiny{$3.52\times10^{-9}$}}  & \multicolumn{2}{C{2cm}}{\cmark\tiny{$9.48\times10^{-59}$}} & \multicolumn{1}{C{1cm}}{}        & \multicolumn{1}{|c}{\multirow{2}{*}{5000}} \\
\multicolumn{1}{l|}{}                         & \multicolumn{2}{c}{Gender}           & \multicolumn{2}{c}{Pose} & \multicolumn{2}{c}{Eyeglasses} & \multicolumn{2}{c}{Age} & \multicolumn{2}{c}{Smile}                                                                                                                                      & \multicolumn{1}{|c}{} \\\hline
\multicolumn{1}{l|}{\multirow{2}{*}{FaceNet$^C$}} &  \multicolumn{1}{C{1cm}}{} & \multicolumn{2}{C{2cm}}{\cmark\tiny{$3.32\times10^{-2}$}}  & \multicolumn{2}{C{2cm}}{\xmark\tiny{$2.27\times10^{-1}$}}  & \multicolumn{2}{C{2cm}}{\cmark\tiny{$2.98\times10^{-7}$}}  & \multicolumn{2}{C{2cm}}{\cmark\tiny{$3.34\times10^{-34}$}} & \multicolumn{1}{C{1cm}}{}& \multicolumn{1}{|c}{\multirow{2}{*}{5000}} \\
\multicolumn{1}{l|}{}                         & \multicolumn{2}{c}{Gender}           & \multicolumn{2}{c}{Eyeglasses} & \multicolumn{2}{c}{Pose} & \multicolumn{2}{c}{Age} & \multicolumn{2}{c}{Smile}                                                                                                                                      & \multicolumn{1}{|c}{} \\\hline
\multicolumn{1}{l|}{\multirow{2}{*}{FaceNet$^V$}} &  \multicolumn{1}{C{1cm}}{} & \multicolumn{2}{C{2cm}}{\xmark\tiny{$4.19\times10^{-1}$}}  & \multicolumn{2}{C{2cm}}{\cmark\tiny{$3.78\times10^{-2}$}}  & \multicolumn{2}{C{2cm}}{\cmark\tiny{$1.58\times10^{-2}$}}  & \multicolumn{2}{C{2cm}}{\cmark\tiny{$7.38\times10^{-9}$}} & \multicolumn{1}{C{1cm}}{}& \multicolumn{1}{|c}{\multirow{2}{*}{5000}} \\
\multicolumn{1}{l|}{}                         & \multicolumn{2}{c}{Eyeglasses}           & \multicolumn{2}{c}{Pose} & \multicolumn{2}{c}{Gender} & \multicolumn{2}{c}{Age} & \multicolumn{2}{c}{Smile}                                                                                                                                      & \multicolumn{1}{|c}{} \\\bottomrule
\end{tabular}
\end{table*}

From the ranking we find (Table~\ref{tab:deltas_fab_pvals}), we mark two main observations that hold for all FRMs: \textit{(i)} the attributes ``Gender'', ``Pose'' and ``Eyeglasses'' always come up in the 1$^{\text{st}}$ to 3$^{\text{rd}}$ positions, and \textit{(ii)} the 4$^{\text{th}}$ and 5$^{\text{th}}$ positions are occupied, respectively, by ``Age'' and ``Smile''.

We interpret the first observation as follows.
First, the fact that ``Gender'' occupies early positions in the ranking is reasonable, as human's most likely also rely heavily on the attributes they customarily associate with gender.
That is, it follows intuitive sense that an FRM would also rely on this informative cue to recognize human faces.
Similarly, seeing ``Eyeglasses'' prompt in the first positions also agrees with our previous observations from Section~\ref{sec:pgd_attacks}, in which we notice the FRMs' striking difficulty for recognizing faces when glasses are added/removed.
While the reliance on either ``Gender'' or ``Eyeglasses'' points to a direction for improving FRMs, we argue it is difficult to question how FRMs learn dependence on these cues, as humans also do.
In contrast, the early position that the ``Pose'' occupies is surprising.
We are unable to give a reasonable explanation to why FRMs would rely on the signal conveyed by ``Pose''.
Thus, we attribute this case of failure of our methodology to the artifacts that StyleGAN introduces in faces when handling extreme poses.

We interpret the second observation as follows.
The fact that ``Smile'' ranks last is satisfactory.
Indeed, FRMs should not be expected to rely much on the ``degree'' of smile to recognize a human face.
However, the fact that ``Age'' ranks relatively low (4$^{\text{th}}$ on a ranking of five attributes), requires careful interpretation.
In particular, we make two remarks.
First, a ranking is, inherently, a \textit{relative} comparison of importance.
That is to say, ``Age'' ranking 4$^{\text{th}}$ does not mean that FRMs \textit{do not} rely on a person's age to recognize them, but rather that they rely less on the person's age than on the other attributes we considered.
Additionally, qualitative inspection of the age-manipulated faces suggests that modifying age is rather unstable.
We argue that modeling age with a single direction vector (independent of the original face) can result in unpredictable outcomes; for instance, varying a child's age \textit{rapidly} changes several attributes, while varying an adult's face by the same amount changes few attributes.
Ultimately, we think that some of the cases of failure we observe can be, for the most part, attributed to the instability of the GAN's output \textit{and} the problems associated with adequately modeling disentangled manipulation of attributes.

\paragraph{Brute-force approach with FAB.}
Analogous to our experiments with PGD, one can also consider a brute-force approach for characterizing semantic robustness against modifications of individual attributes.
This approach consists of giving FAB access to modifying a single attribute, \ie a single direction, conducting the attack, and recording the average energy required to fool each FRM.

\begin{table}[]
\centering
\caption{\textbf{Attribute-restricted search for adversaries through FAB.} 
We restrict FAB's attribute-perturbing capacity for each single attribute, and report the average perturbation required to fool each method.
\label{tab:fab_energy}}
\centering
\begin{tabular}{c|ccccc}
\toprule
\multirow{2}{*}{Method} & \multicolumn{5}{c}{Required energy} \\
                        & Pose      & Age       & Gender    & Smile     & Eyeglasses    \\ \hline
ArcFace                 & \textbf{3.4}	    & 5.5	    & 28.9	    & 4.3	   & 5.1       \\
FaceNet$^C$      & \textbf{3.2}	    & 5.2	    & 13.0	    & 3.4	   & 4.5       \\
FaceNet$^V$      & 3.1	    & 4.5	    & 10.8	    & \textbf{3.0}	   & 4.2       \\\bottomrule

\end{tabular}
\end{table}

Table~\ref{tab:fab_energy} reports the results of this experiment for each of the FRMs we considered.
From these results we make the following remarks: \textit{(i)} the overall average energy required to fool each method again finds a ranking of FRMs by which ArcFace $>$ FaceNet$^C$ $>$ FaceNet$^V$, \textit{(ii)} the required energy per attribute can go from as high as 28.9 (ArcFace when attacking only gender), to as low as 3.0 (FaceNet$^V$ when attacking smile), \textit{(iii)} if we rank attributes according to the energy for modification (from high to low), we get an overall ranking of: Gender $>$ Age $>$ Eyeglasses $>$ Smile $>$ Pose (except for FaceNet$^V$, for which the ranking of Smile and Pose is reversed).

Figures~\ref{fig:fab_onlyPose_qual} to~\ref{fig:fab_onlyEye_qual} report, for each of these ablations (five in total, one for each attribute), qualitative examples that show the results of following this brute-force approach with FAB.
From these results, we make various observations.
\underline{First}, indeed we find a strong correlation between the attribute being manipulated and the modification made to the face, \ie manipulating the pose attribute in fact modifies pose (and analogous observations can be made for the other attributes).
\underline{Second}, the energies required to modify the FRMs' prediction are, overall, large (as reported in Table~\ref{tab:fab_energy}); consequently, the visual changes induced in the images are also large.
This fact emphasizes how the adversarial examples find by FAB are not required to have an identity-preserving property.
\underline{Third}, the entanglement between attributes is visible in these qualitative samples.
While the entanglement was also present in the adversarial examples found by PGD, we argue that it occurred to a much lesser extent.
This fact can be immediately attributed to the larger energies associated to these samples.
Furthermore, we also argue this phenomenon can be related to a manifold-like nature, by which following directions related to attributes may be reasonable \textit{locally} and not globally.
\underline{Fourth}, while many identity-modifying changes are being introduced in the images, there certainly are several attributes that are maintained, such as the image's background, the hair color, skin color, lighting, and, arguably, overall appearance.
\underline{Fifth}, modifications involving children-like faces still display unreliable results.
\underline{Sixth}, fooling the FRM just by manipulating the smile attribute apparently requires extreme changes: the person's mouth is modified to be extremely open or remarkably closed.
\underline{Seventh}, and finally, manipulating the gender attribute seems to be unreliable in these cases.
As Figure~\ref{fig:fab_onlyGender_qual} displays, the gender-modified versions of faces are remarkably different from the corresponding original faces.
These changes are notoriously related to the large perturbations that the figure also reports (the perturbations' energy is around $\sim$ 28 for all examples).
The reason why FAB is unable to find examples with smaller perturbations (\ie closer to the original input image) is unclear.
However, regardless of the reason behind this phenomenon, manipulating the gender attribute seems to be unreliable, specially when the perturbation is of such large magnitudes.

We argue that these observations, overall, suggest that the brute-force approach to using FAB is \textit{not} a reliable source for analyzing the semantic robustness of FRMs.
Indeed, we find that the ranking from this brute-force procedure shares few similarities with the rankings from Table~\ref{tab:deltas_fab_pvals}.

\twocolumn
\begin{figure}
    \centering
    \raisebox{0.35in}{\rotatebox[origin=t]{90}{3.69}}\includegraphics[trim=0cm 0.3cm 0cm 0.7cm,clip,width=0.96\columnwidth]{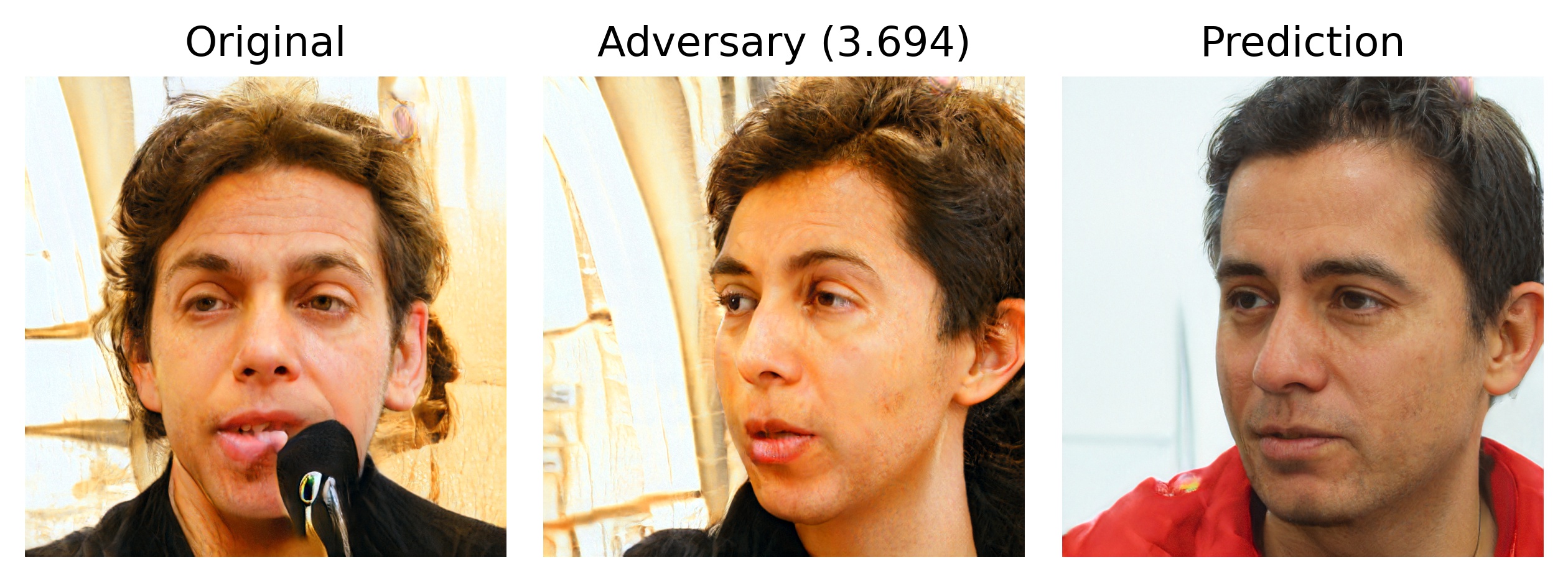}\\
    \raisebox{0.35in}{\rotatebox[origin=t]{90}{2.64}}\includegraphics[trim=0cm 0.3cm 0cm 0.7cm,clip,width=0.96\columnwidth]{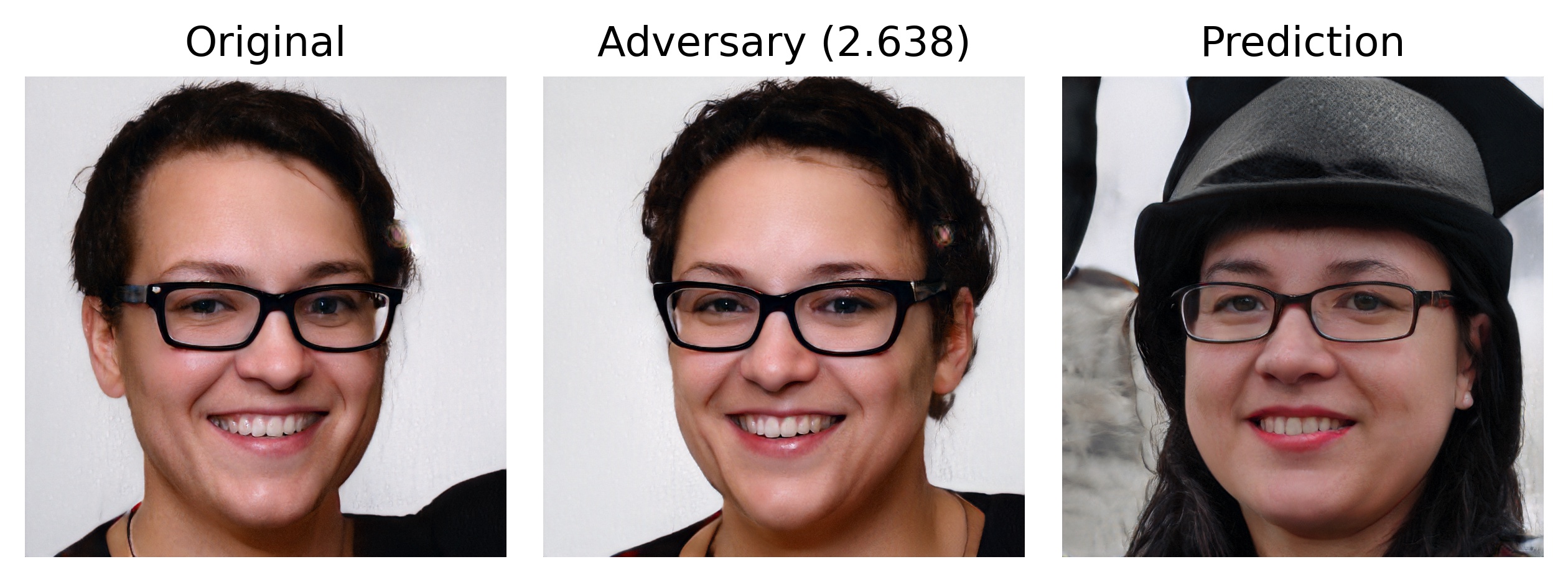}\\
    \raisebox{0.35in}{\rotatebox[origin=t]{90}{3.97}}\includegraphics[trim=0cm 0.3cm 0cm 0.7cm,clip,width=0.96\columnwidth]{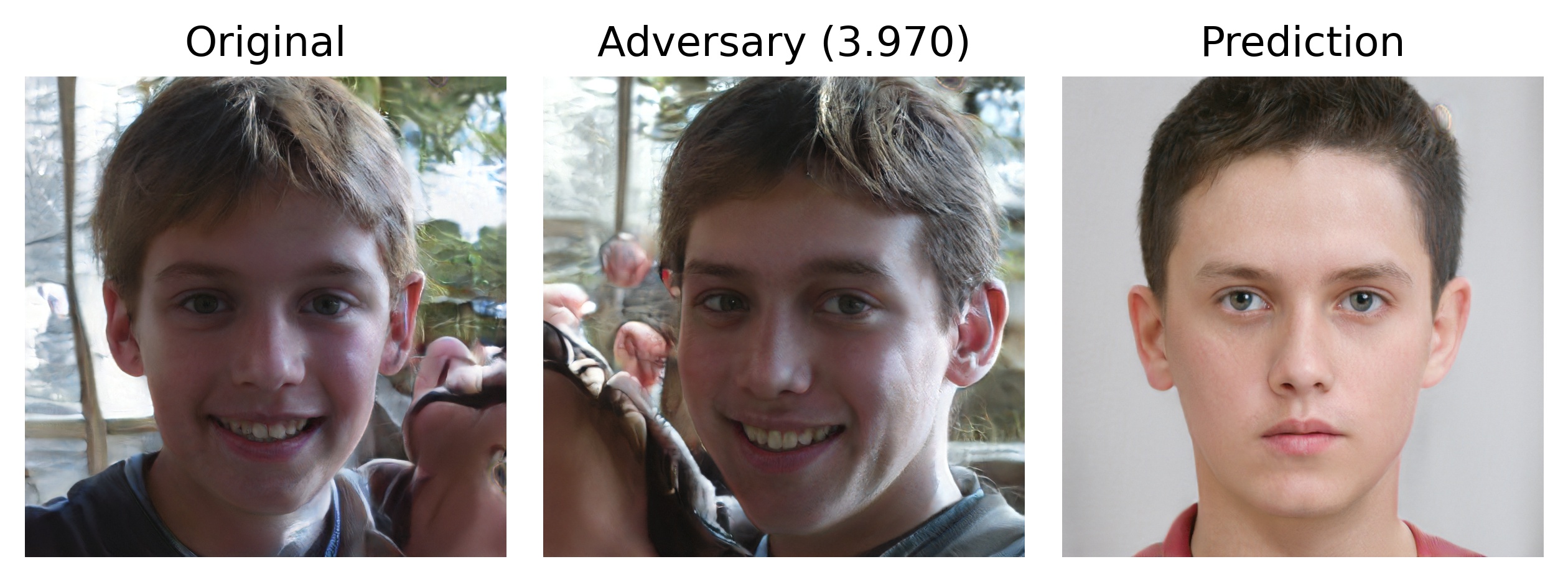}\\
    \raisebox{0.35in}{\rotatebox[origin=t]{90}{5.03}}\includegraphics[trim=0cm 0.3cm 0cm 0.7cm,clip,width=0.96\columnwidth]{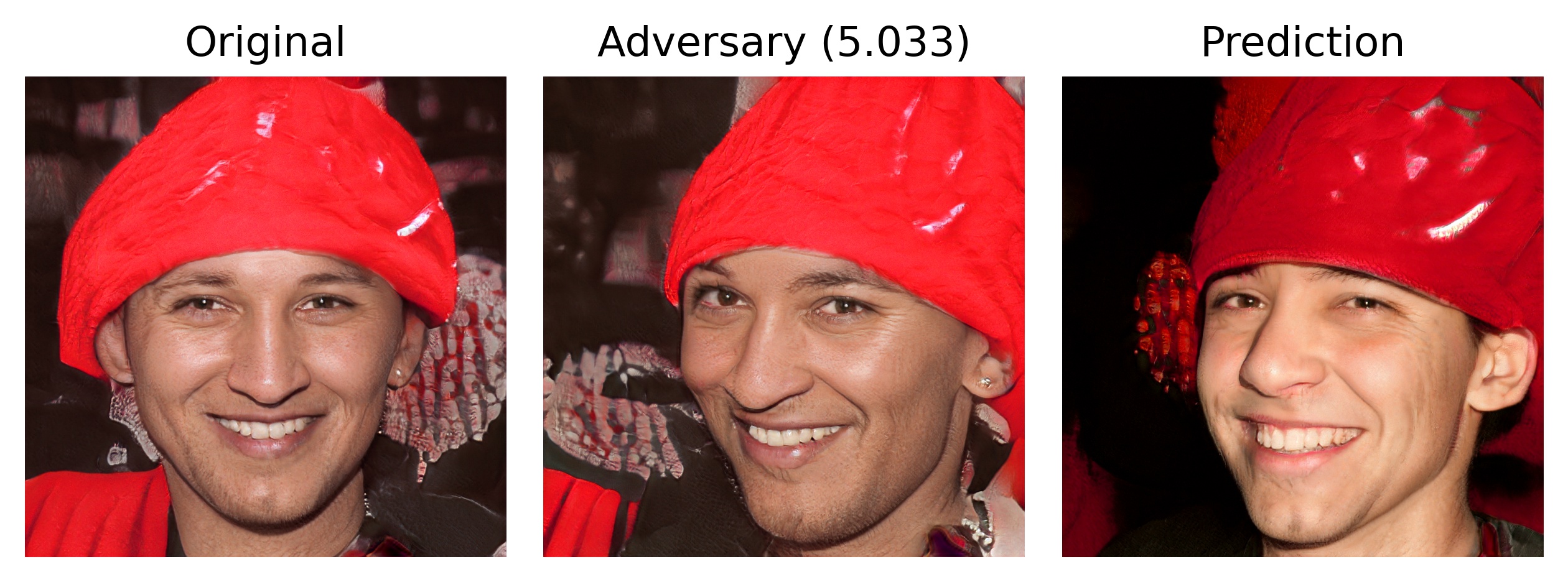}\\
    \raisebox{0.35in}{\rotatebox[origin=t]{90}{4.08}}\includegraphics[trim=0cm 0.3cm 0cm 0.7cm,clip,width=0.96\columnwidth]{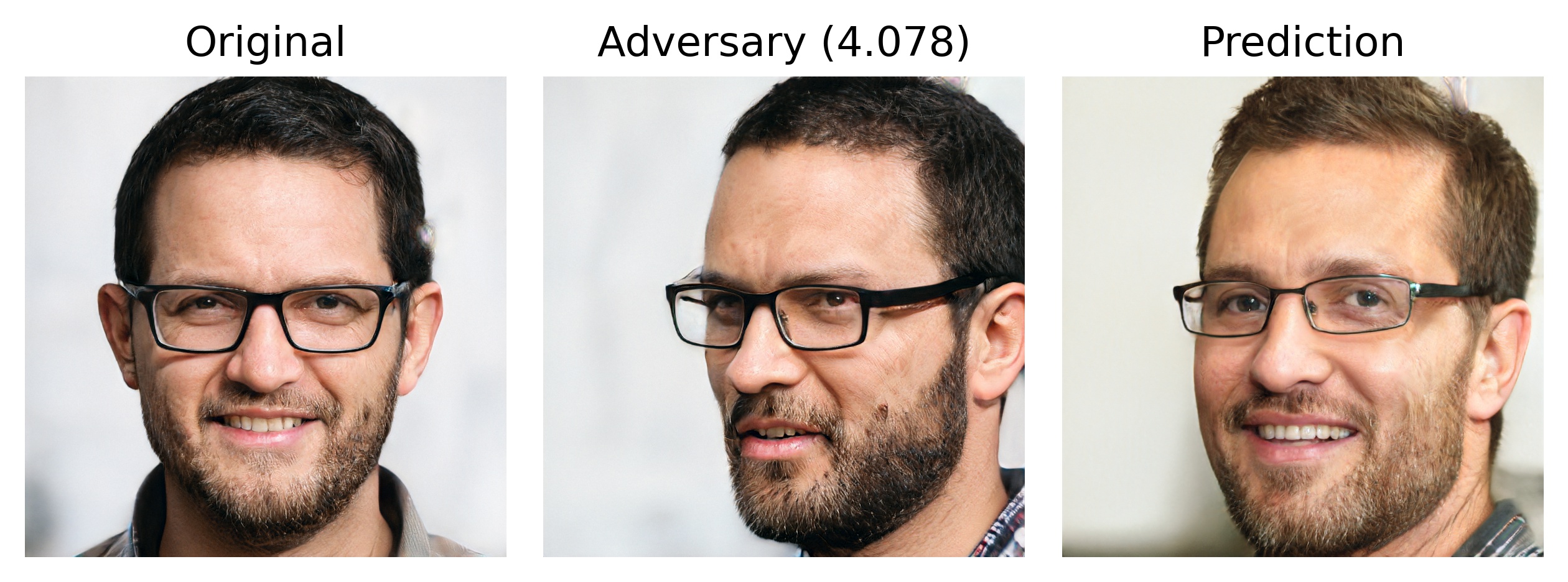}\\
    \raisebox{0.35in}{\rotatebox[origin=t]{90}{4.09}}\includegraphics[trim=0cm 0.3cm 0cm 0.7cm,clip,width=0.96\columnwidth]{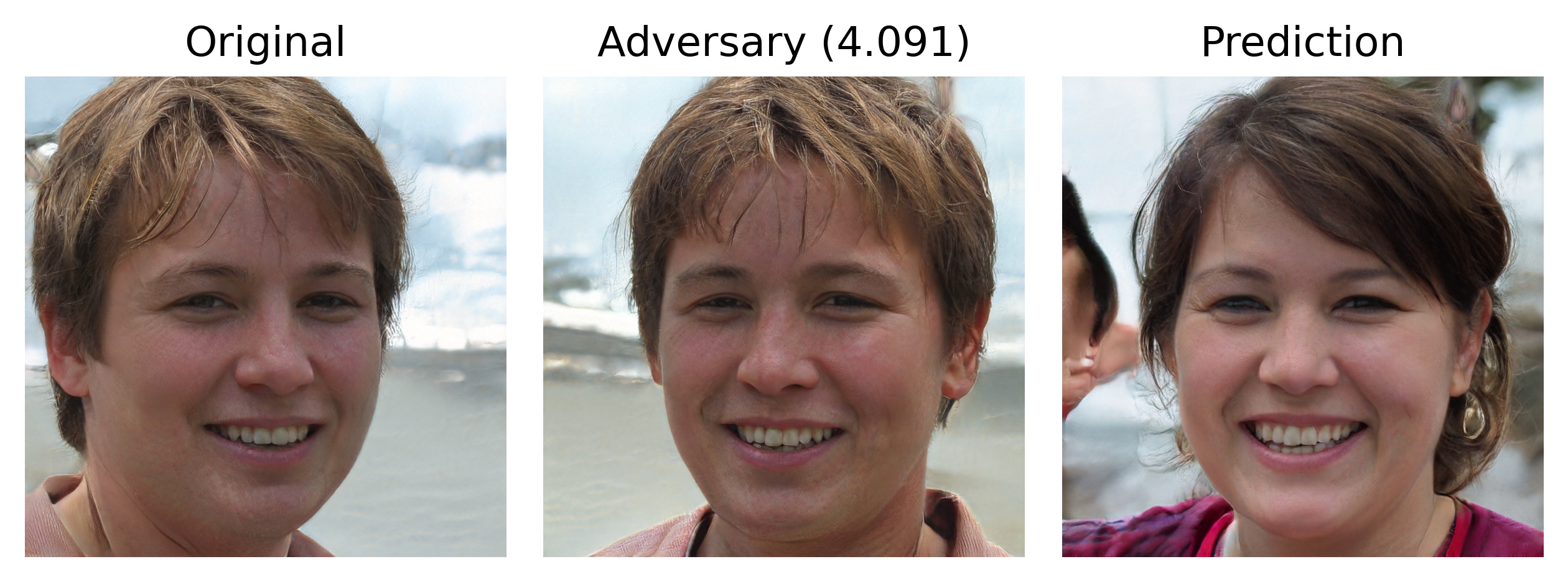}\\
    \raisebox{0.35in}{\rotatebox[origin=t]{90}{3.78}}\includegraphics[trim=0cm 0.3cm 0cm 0.7cm,clip,width=0.96\columnwidth]{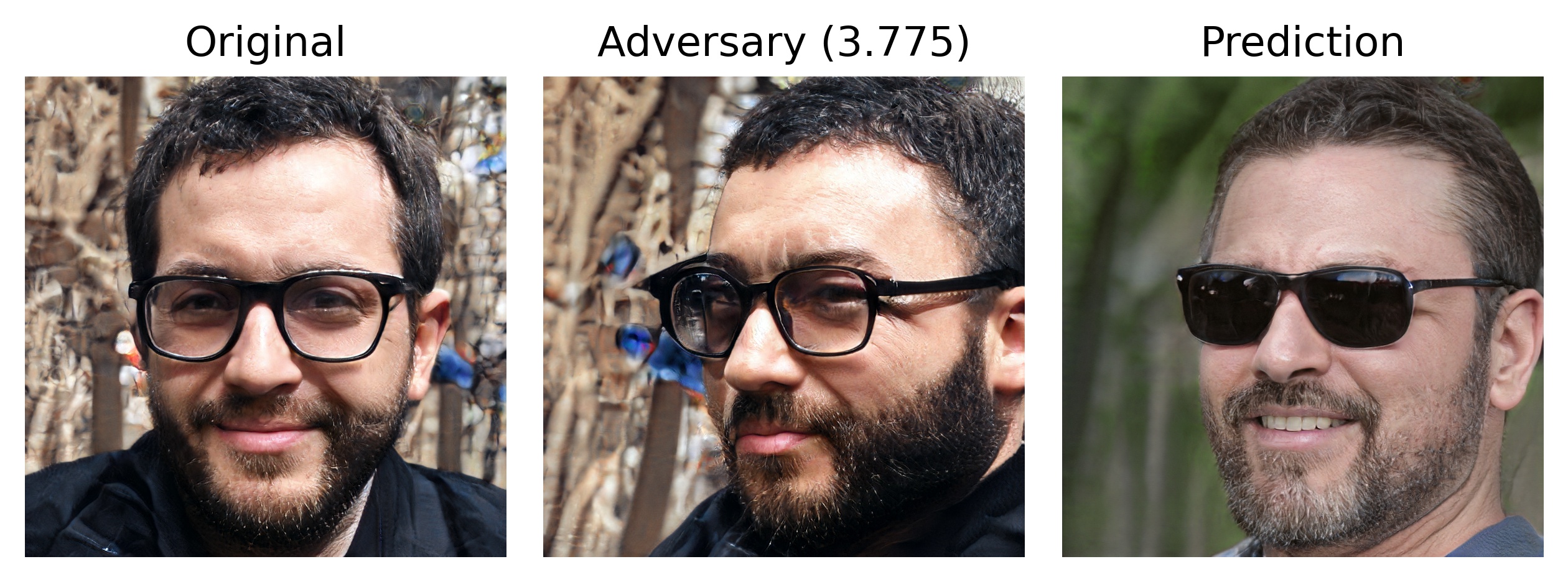}\\
    \raisebox{0.35in}{\rotatebox[origin=t]{90}{3.50}}\includegraphics[trim=0cm 0.3cm 0cm 0.7cm,clip,width=0.96\columnwidth]{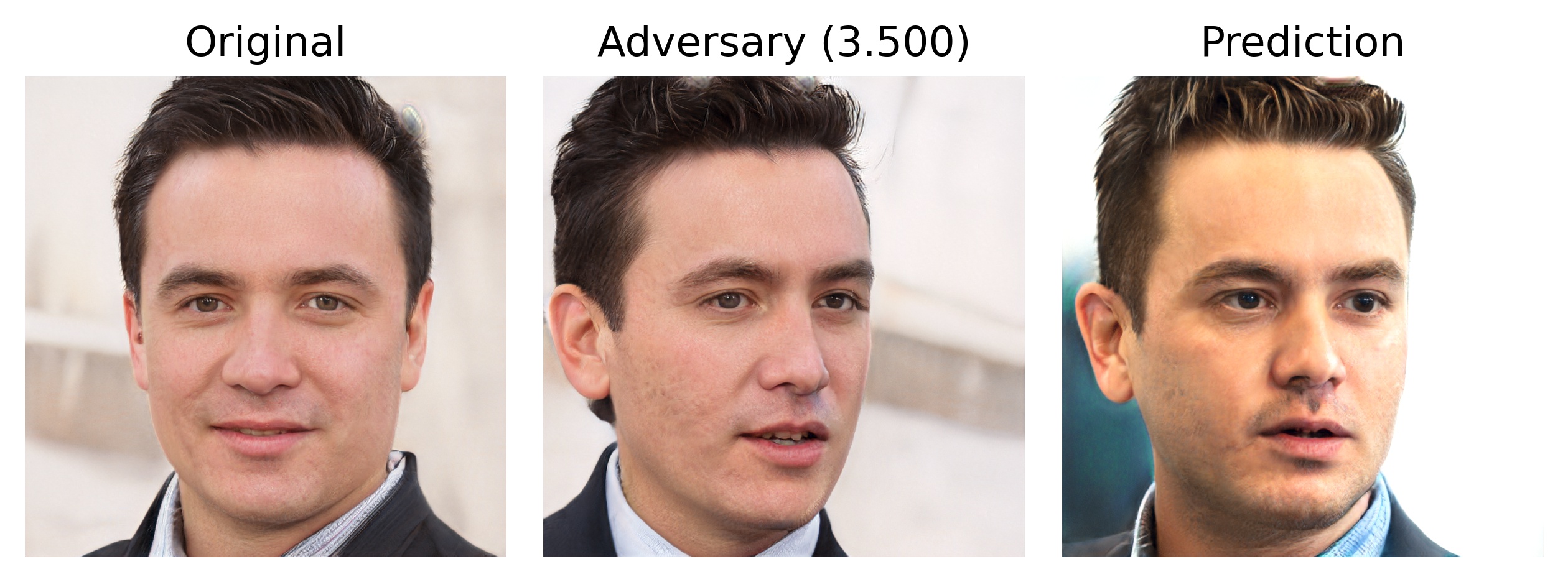}\\
    \caption{\textbf{Adversarial examples found by FAB when only attacking the \underline{Pose} attribute.} 
    Each row is a different identity. 
    We report each perturbation's energy, $\|\pmb{\delta}\|_{M, 2}$, at the far left. 
    \textit{Left:} original face \textcolor{blue}{\fontfamily{cmss}\selectfont\textbf{A}}, \textit{middle:} modified face \textcolor{orange}{\fontfamily{cmss}\selectfont\textbf{A$^\star$}}, \textit{right:} match \textcolor{purple}{\fontfamily{cmss}\selectfont\textbf{B}}. The FRM prefers to match \textcolor{orange}{\fontfamily{cmss}\selectfont\textbf{A$^\star$}} with \textcolor{purple}{\fontfamily{cmss}\selectfont\textbf{B}} rather than with \textcolor{blue}{\fontfamily{cmss}\selectfont\textbf{A}}.}
    \label{fig:fab_onlyPose_qual}
\end{figure}

\begin{figure}
    \centering
    \raisebox{0.35in}{\rotatebox[origin=t]{90}{6.99}}\includegraphics[trim=0cm 0.3cm 0cm 0.7cm,clip,width=0.96\columnwidth]{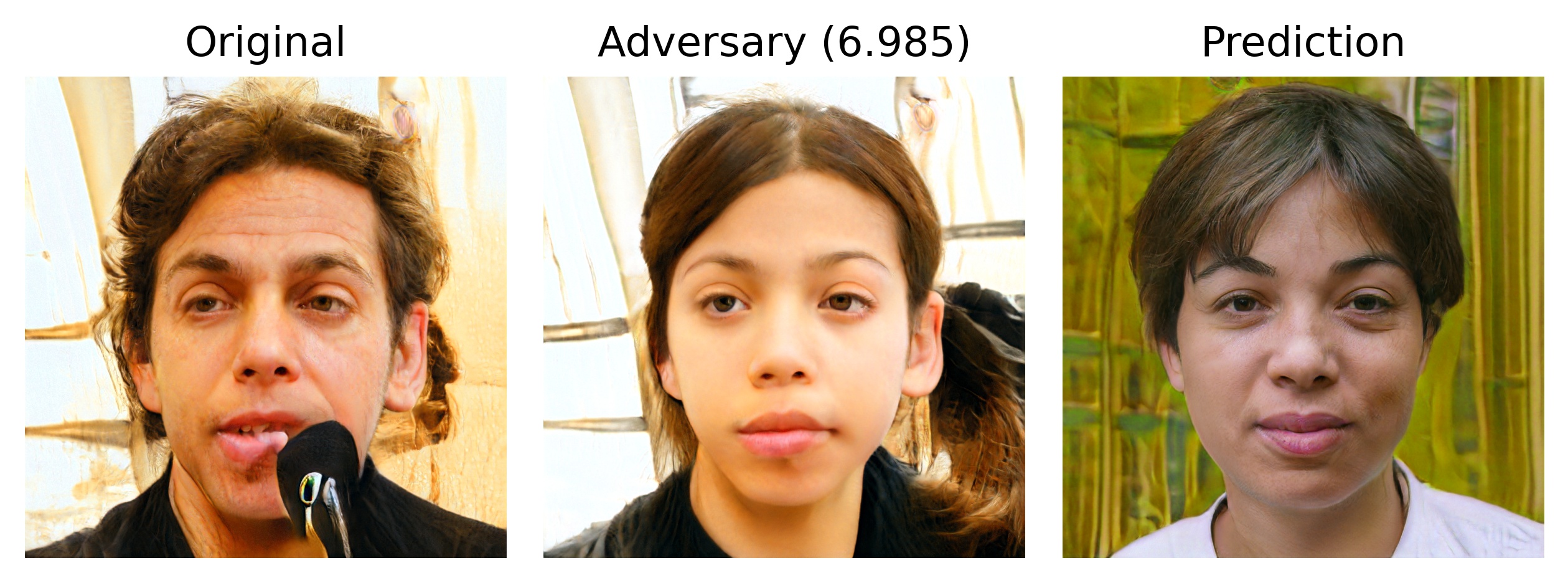}\\
    \raisebox{0.35in}{\rotatebox[origin=t]{90}{7.35}}\includegraphics[trim=0cm 0.3cm 0cm 0.7cm,clip,width=0.96\columnwidth]{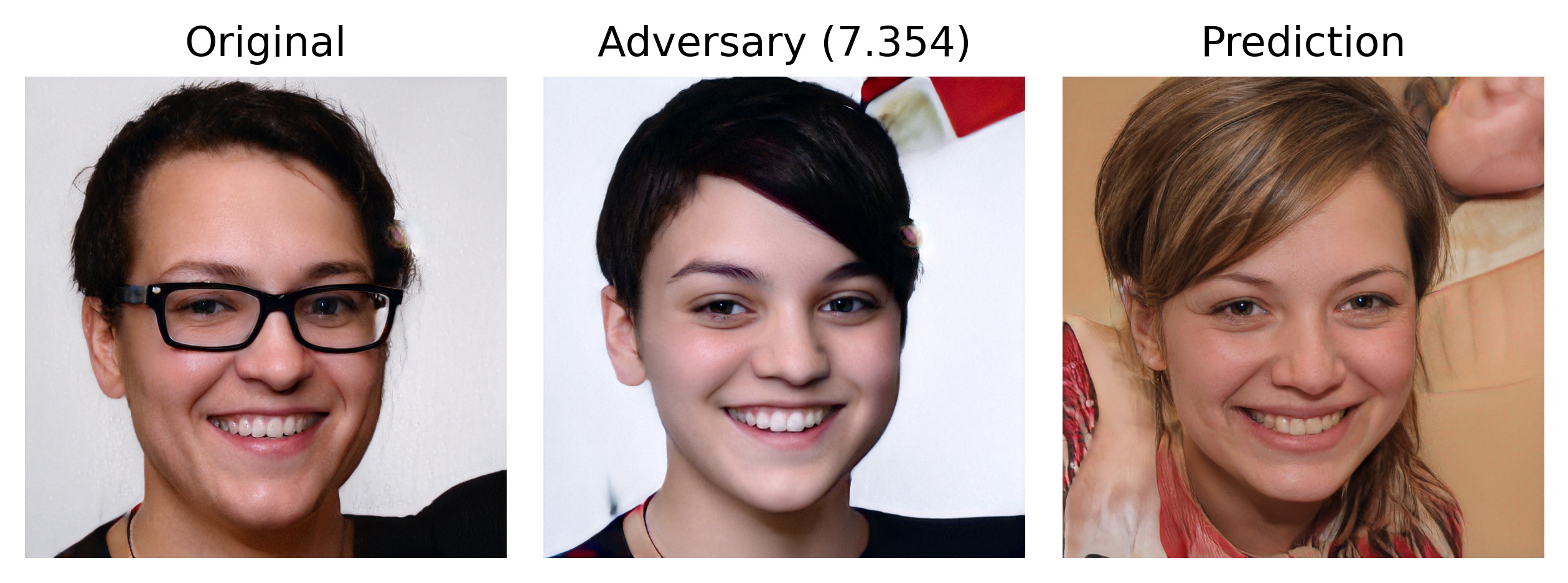}\\
    \raisebox{0.35in}{\rotatebox[origin=t]{90}{2.13}}\includegraphics[trim=0cm 0.3cm 0cm 0.7cm,clip,width=0.96\columnwidth]{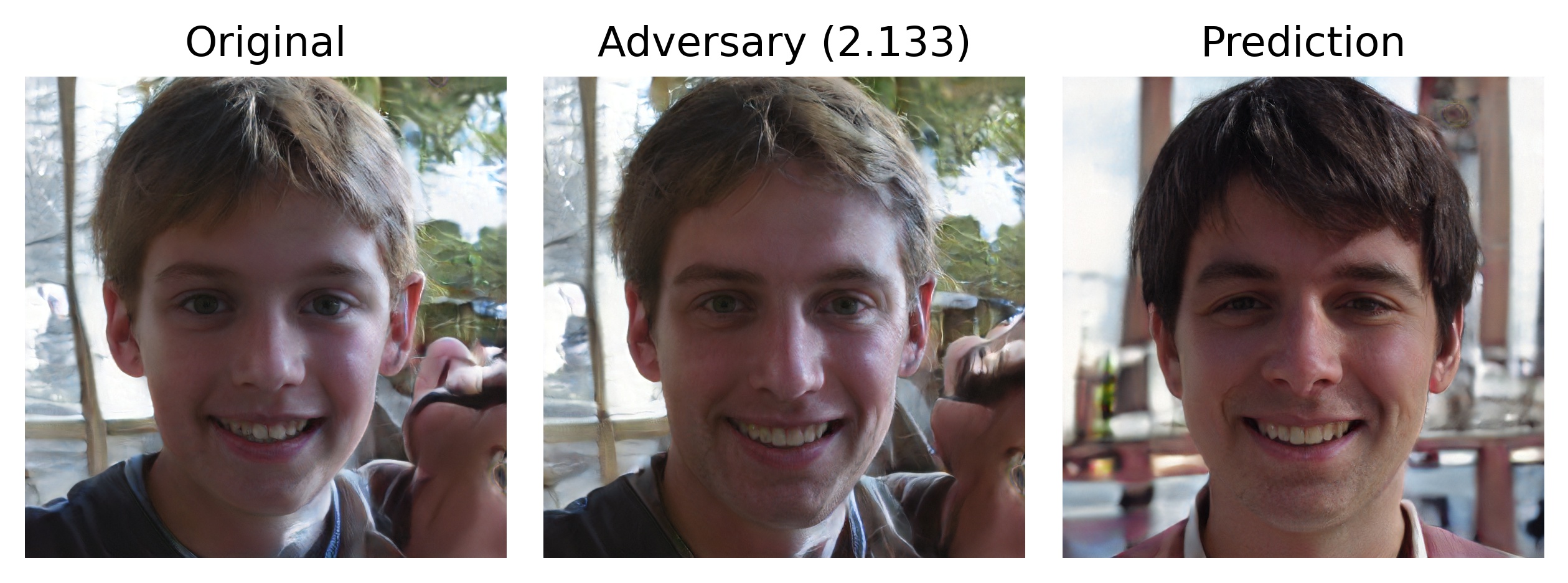}\\
    \raisebox{0.35in}{\rotatebox[origin=t]{90}{6.73}}\includegraphics[trim=0cm 0.3cm 0cm 0.7cm,clip,width=0.96\columnwidth]{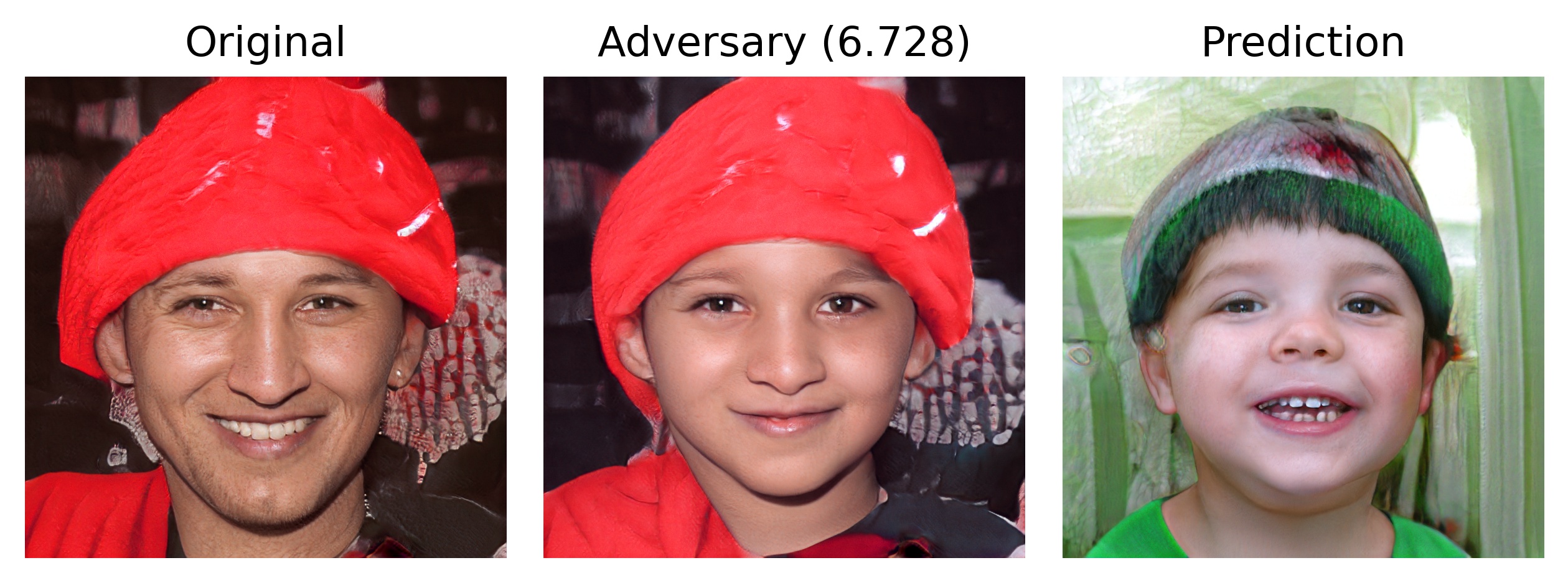}\\
    \raisebox{0.35in}{\rotatebox[origin=t]{90}{6.27}}\includegraphics[trim=0cm 0.3cm 0cm 0.7cm,clip,width=0.96\columnwidth]{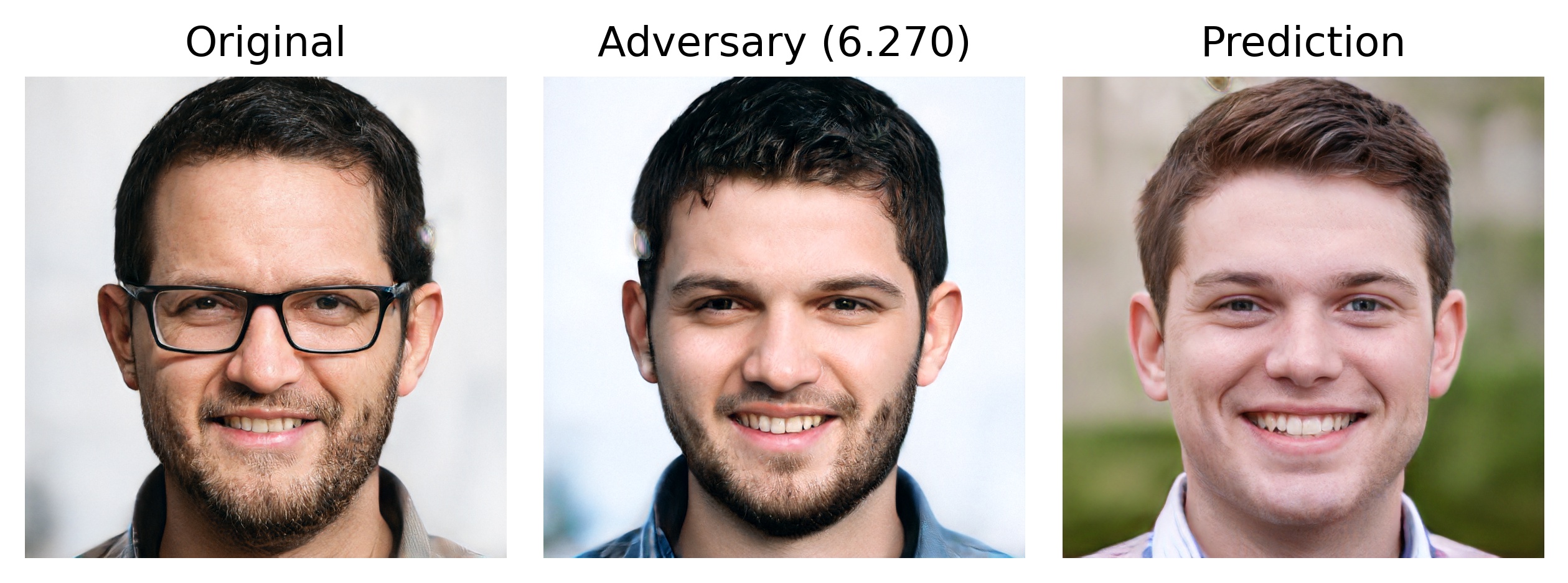}\\
    \raisebox{0.35in}{\rotatebox[origin=t]{90}{5.63}}\includegraphics[trim=0cm 0.3cm 0cm 0.7cm,clip,width=0.96\columnwidth]{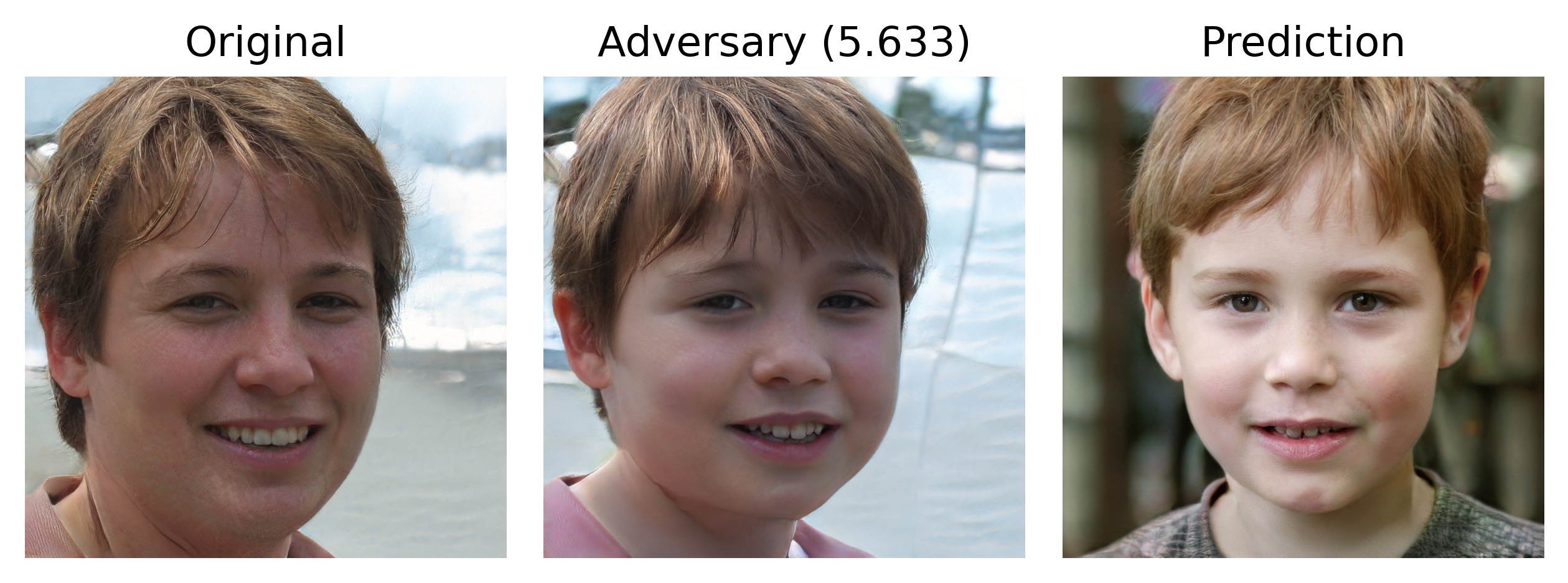}\\
    \raisebox{0.35in}{\rotatebox[origin=t]{90}{4.73}}\includegraphics[trim=0cm 0.3cm 0cm 0.7cm,clip,width=0.96\columnwidth]{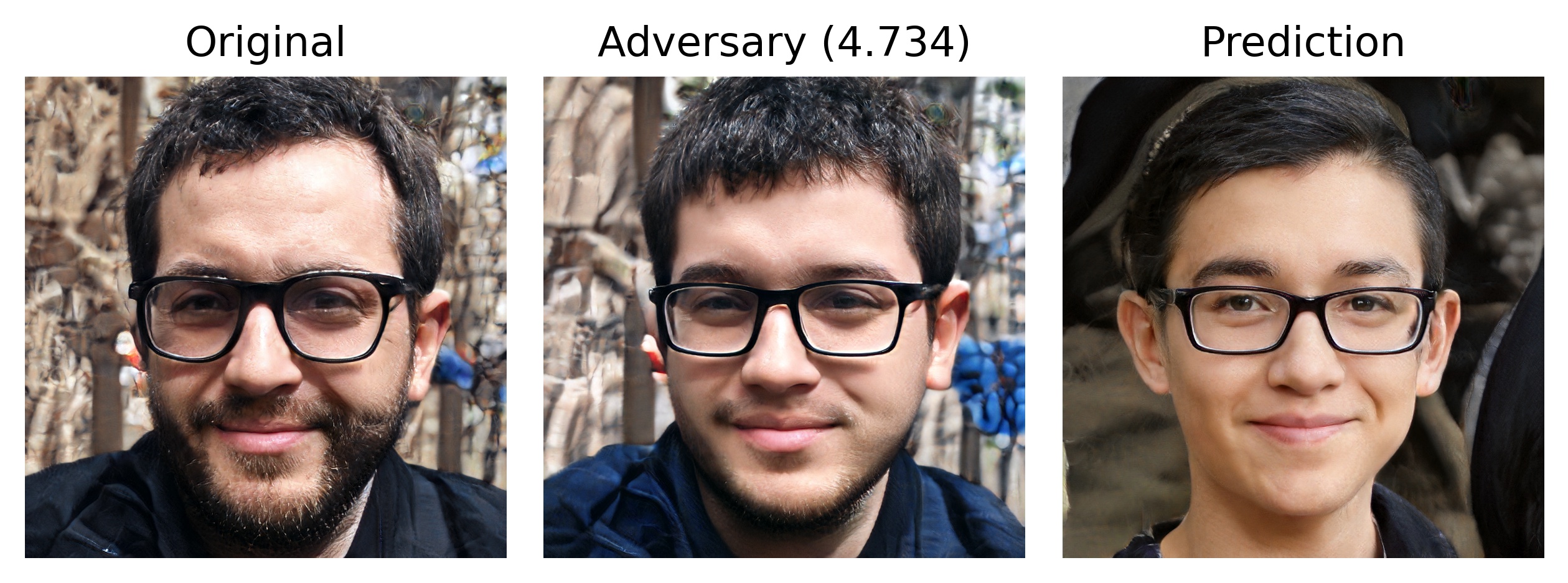}\\
    \raisebox{0.35in}{\rotatebox[origin=t]{90}{5.83}}\includegraphics[trim=0cm 0.3cm 0cm 0.7cm,clip,width=0.96\columnwidth]{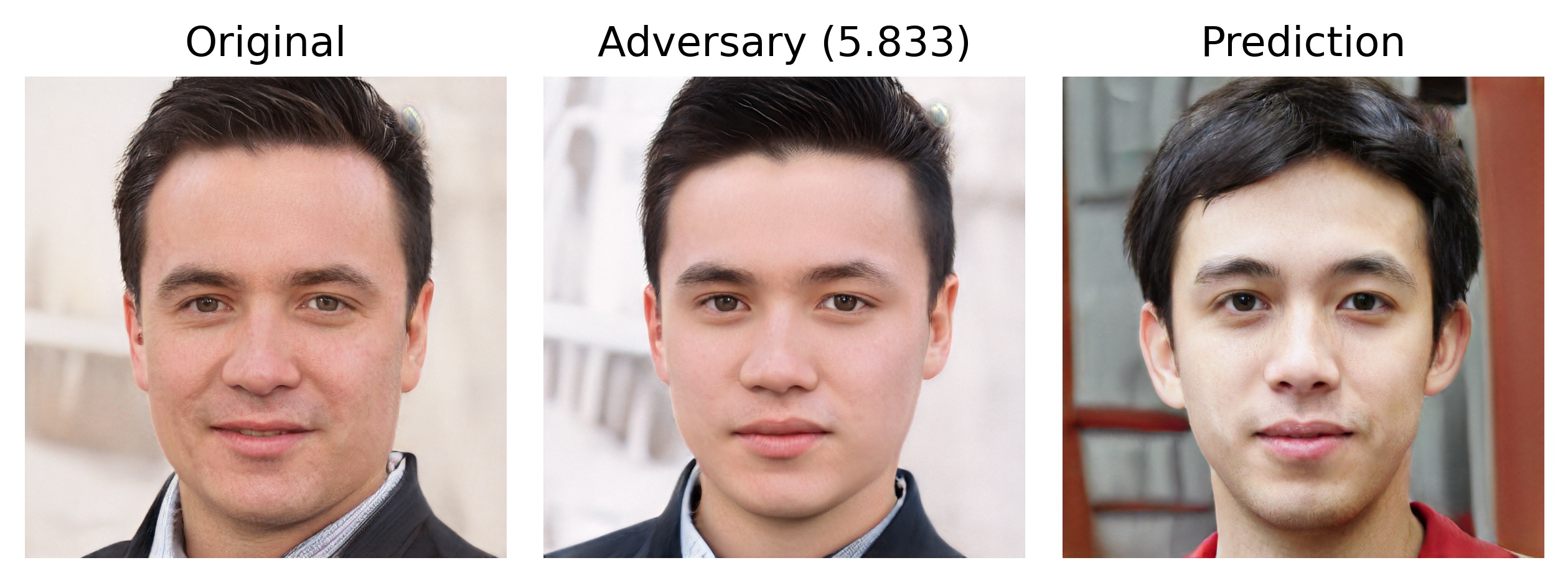}\\
    \caption{\textbf{Adversarial examples found by FAB when only attacking the \underline{Age} attribute.} 
    Each row is a different identity. 
    We report each perturbation's energy, $\|\pmb{\delta}\|_{M, 2}$, at the far left. 
    \textit{Left:} original face \textcolor{blue}{\fontfamily{cmss}\selectfont\textbf{A}}, \textit{middle:} modified face \textcolor{orange}{\fontfamily{cmss}\selectfont\textbf{A$^\star$}}, \textit{right:} match \textcolor{purple}{\fontfamily{cmss}\selectfont\textbf{B}}. The FRM prefers to match \textcolor{orange}{\fontfamily{cmss}\selectfont\textbf{A$^\star$}} with \textcolor{purple}{\fontfamily{cmss}\selectfont\textbf{B}} rather than with \textcolor{blue}{\fontfamily{cmss}\selectfont\textbf{A}}.}
    \label{fig:fab_onlyAge_qual}
\end{figure}

\begin{figure}
    \centering
    \raisebox{0.35in}{\rotatebox[origin=t]{90}{26.85}}\includegraphics[trim=0cm 0.3cm 0cm 0.7cm,clip,width=0.96\columnwidth]{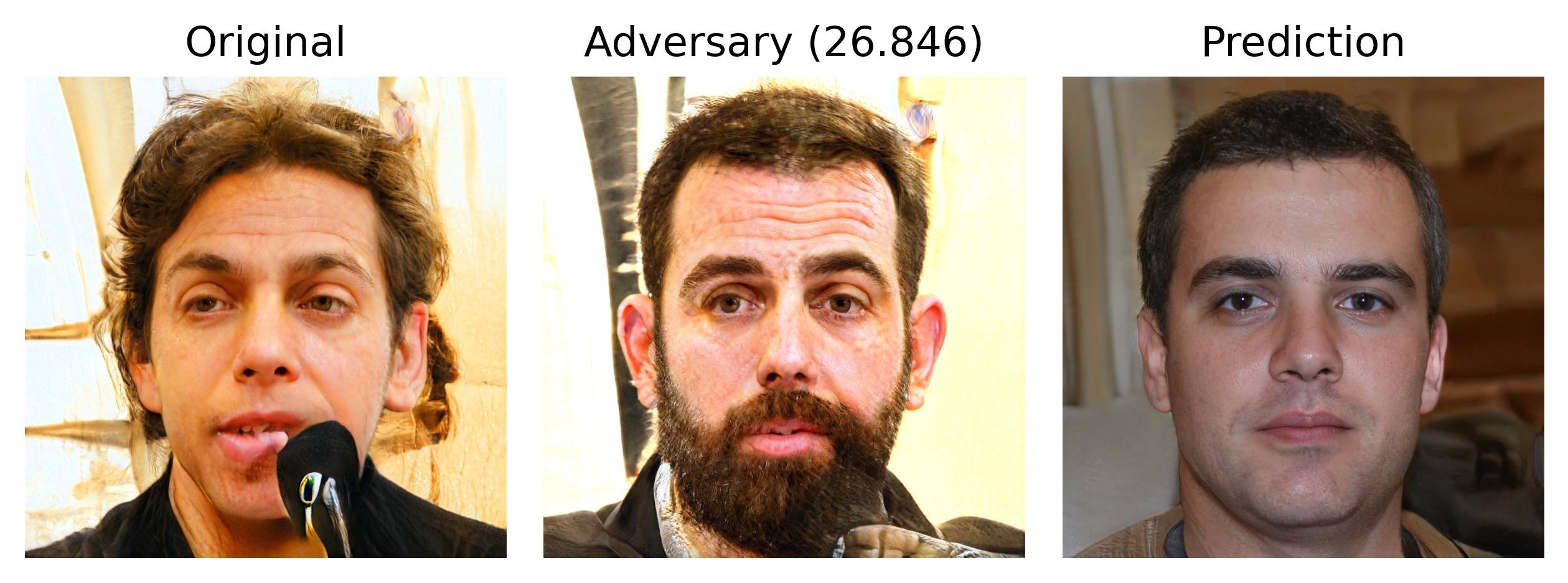}\\
    \raisebox{0.35in}{\rotatebox[origin=t]{90}{28.45}}\includegraphics[trim=0cm 0.3cm 0cm 0.7cm,clip,width=0.96\columnwidth]{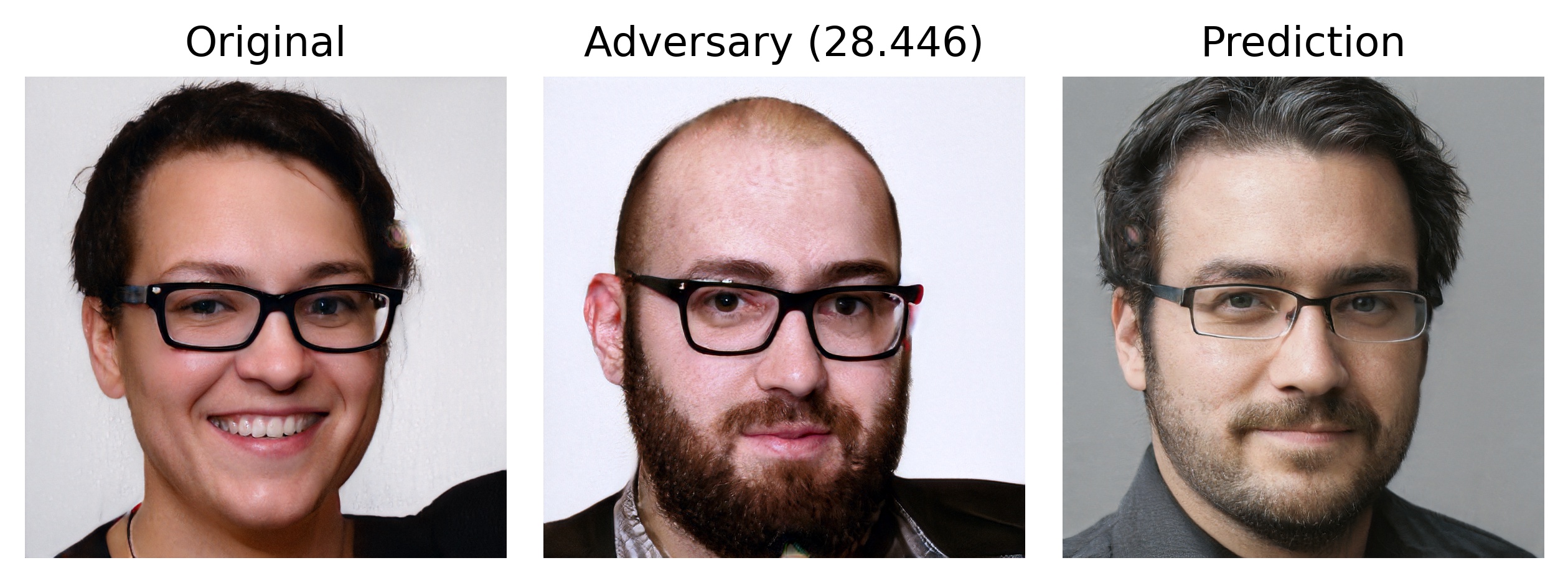}\\
    \raisebox{0.35in}{\rotatebox[origin=t]{90}{28.18}}\includegraphics[trim=0cm 0.3cm 0cm 0.7cm,clip,width=0.96\columnwidth]{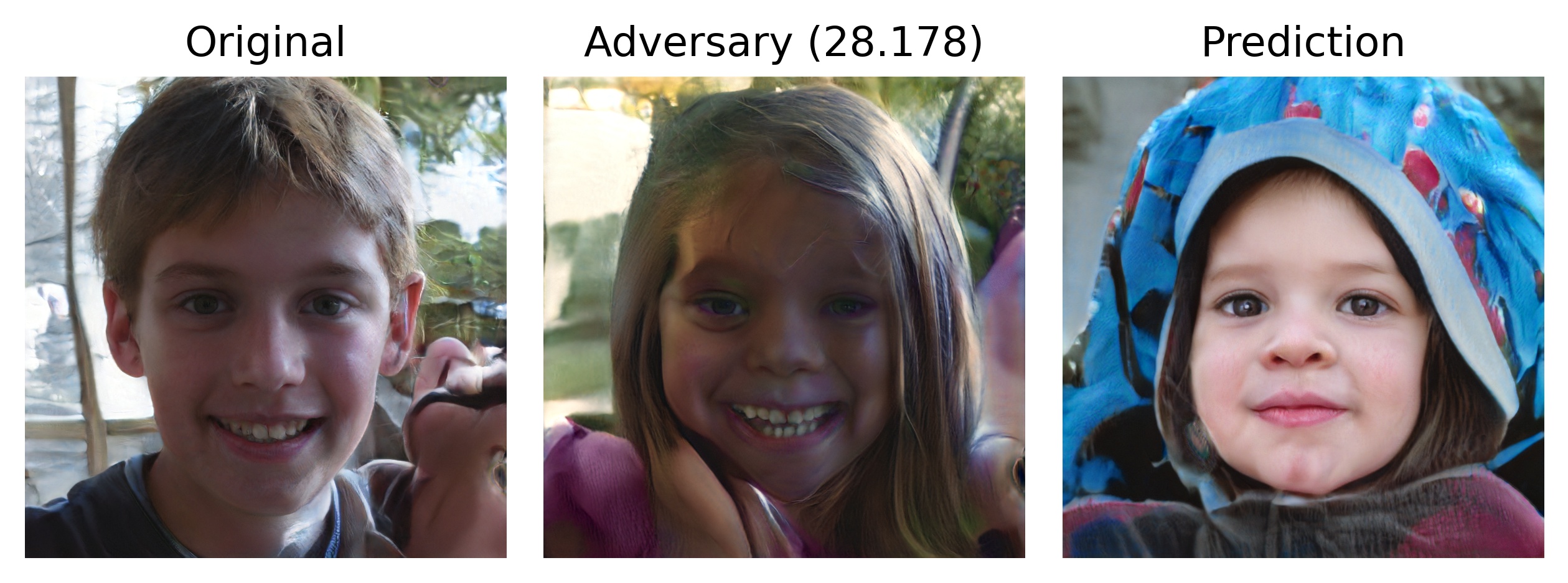}\\
    \raisebox{0.35in}{\rotatebox[origin=t]{90}{28.69}}\includegraphics[trim=0cm 0.3cm 0cm 0.7cm,clip,width=0.96\columnwidth]{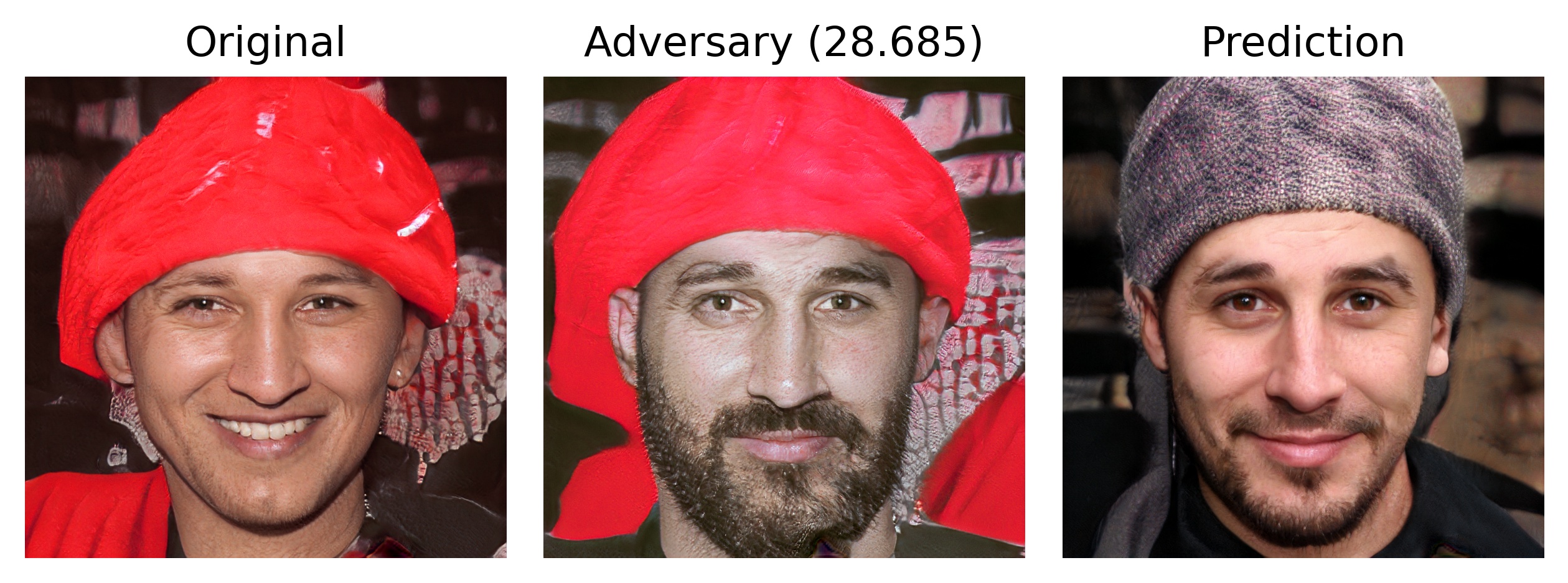}\\
    \raisebox{0.35in}{\rotatebox[origin=t]{90}{29.63}}\includegraphics[trim=0cm 0.3cm 0cm 0.7cm,clip,width=0.96\columnwidth]{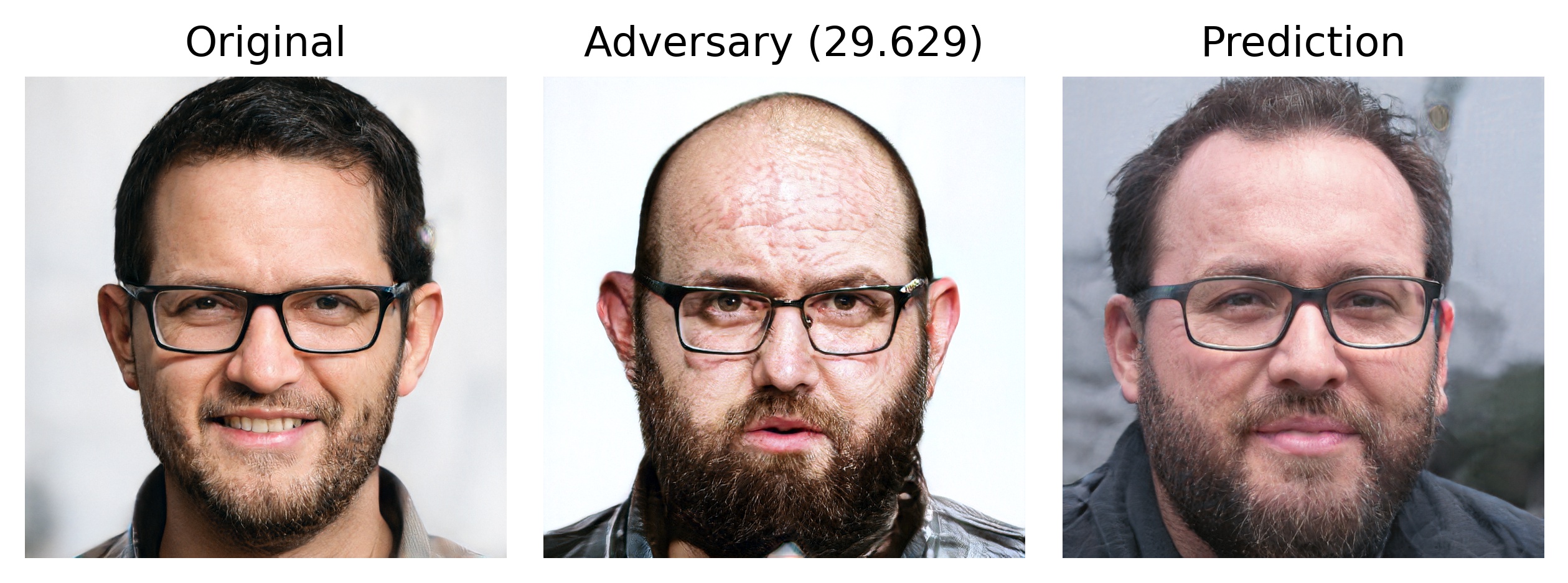}\\
    \raisebox{0.35in}{\rotatebox[origin=t]{90}{26.90}}\includegraphics[trim=0cm 0.3cm 0cm 0.7cm,clip,width=0.96\columnwidth]{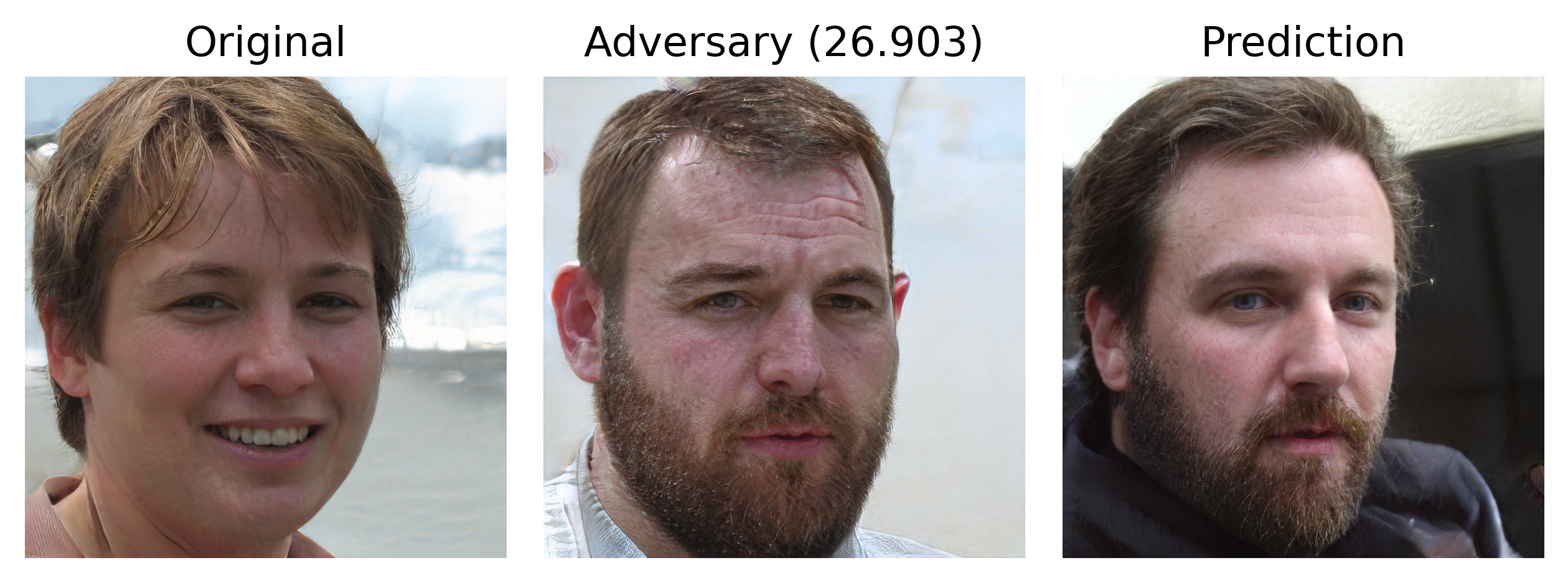}\\
    \raisebox{0.35in}{\rotatebox[origin=t]{90}{29.20}}\includegraphics[trim=0cm 0.3cm 0cm 0.7cm,clip,width=0.96\columnwidth]{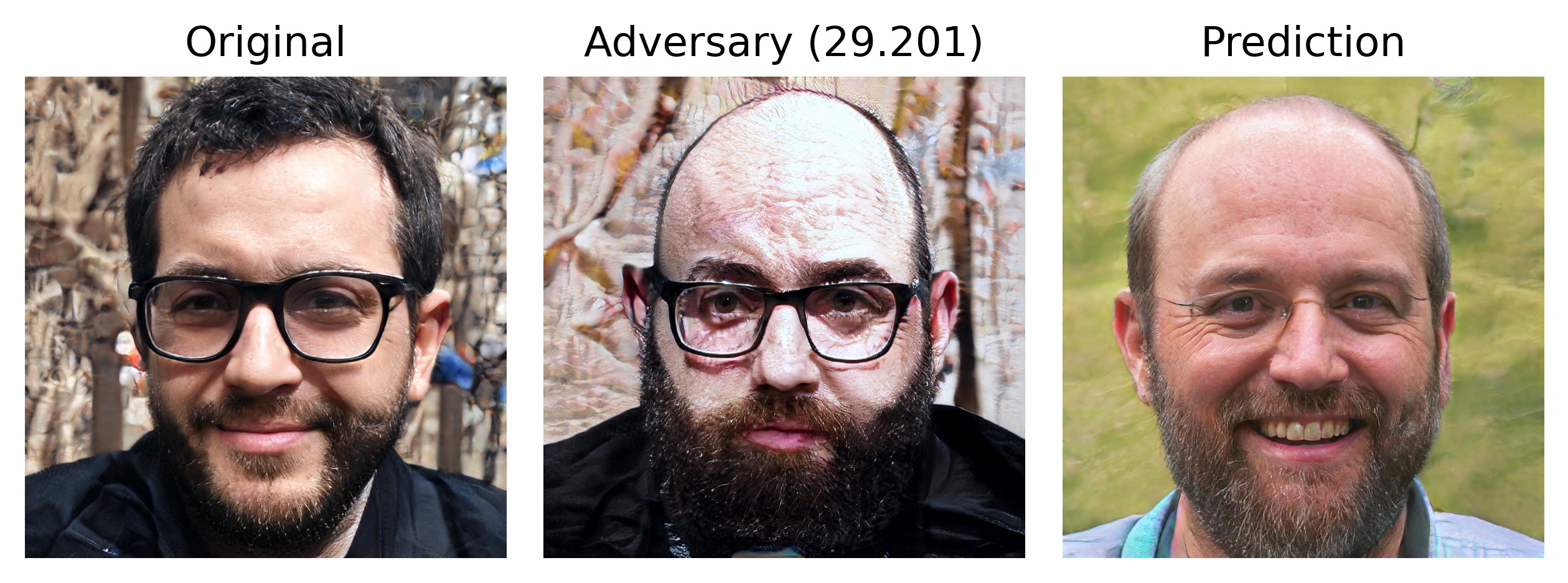}\\
    \raisebox{0.35in}{\rotatebox[origin=t]{90}{28.47}}\includegraphics[trim=0cm 0.3cm 0cm 0.7cm,clip,width=0.96\columnwidth]{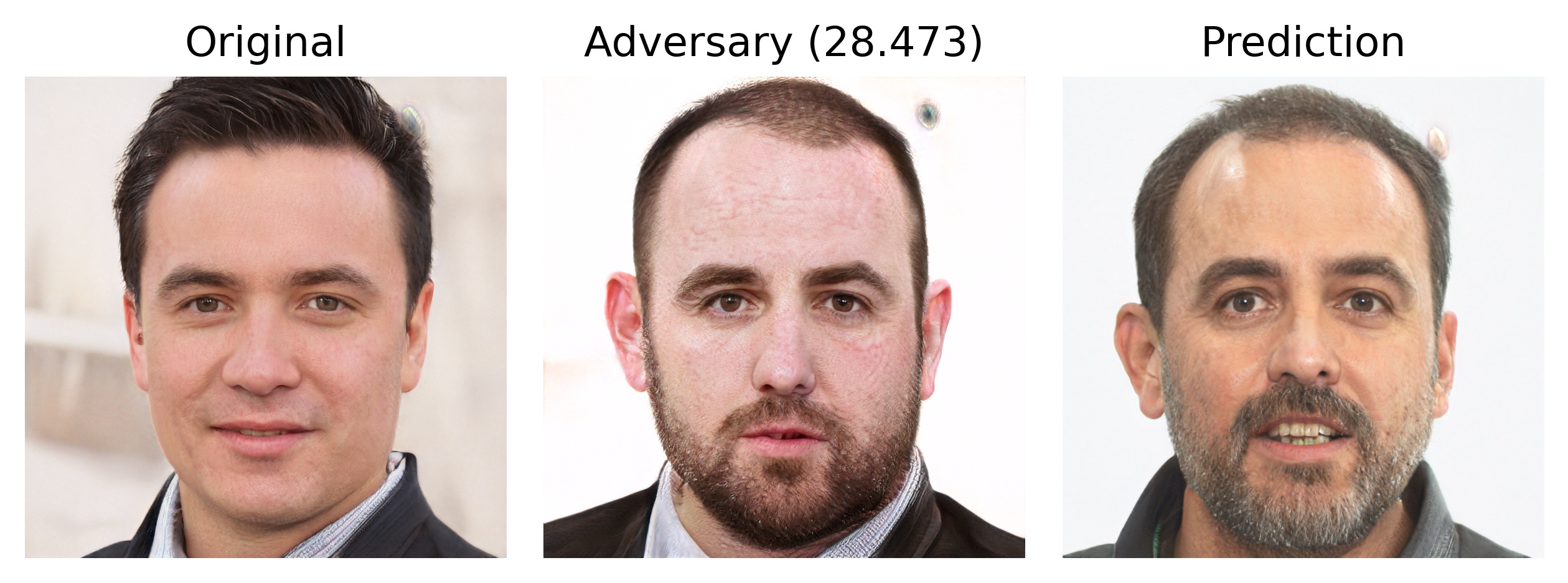}\\
    \caption{\textbf{Adversarial examples found by FAB when only attacking the \underline{Gender} attribute.} 
    Each row is a different identity. 
    We report each perturbation's energy, $\|\pmb{\delta}\|_{M, 2}$, at the far left. 
    \textit{Left:} original face \textcolor{blue}{\fontfamily{cmss}\selectfont\textbf{A}}, \textit{middle:} modified face \textcolor{orange}{\fontfamily{cmss}\selectfont\textbf{A$^\star$}}, \textit{right:} match \textcolor{purple}{\fontfamily{cmss}\selectfont\textbf{B}}. The FRM prefers to match \textcolor{orange}{\fontfamily{cmss}\selectfont\textbf{A$^\star$}} with \textcolor{purple}{\fontfamily{cmss}\selectfont\textbf{B}} rather than with \textcolor{blue}{\fontfamily{cmss}\selectfont\textbf{A}}.}
    \label{fig:fab_onlyGender_qual}
\end{figure}

\begin{figure}
    \centering
    \raisebox{0.35in}{\rotatebox[origin=t]{90}{2.61}}\includegraphics[trim=0cm 0.3cm 0cm 0.7cm,clip,width=0.96\columnwidth]{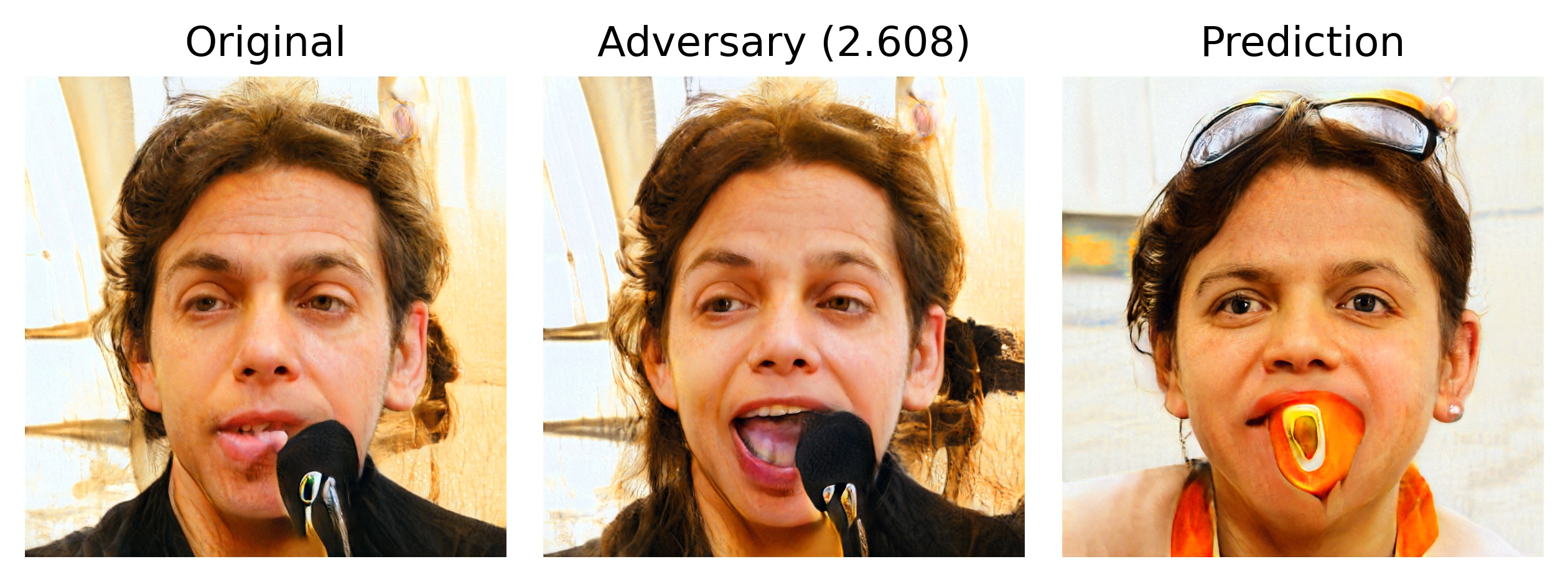}\\
    \raisebox{0.35in}{\rotatebox[origin=t]{90}{4.08}}\includegraphics[trim=0cm 0.3cm 0cm 0.7cm,clip,width=0.96\columnwidth]{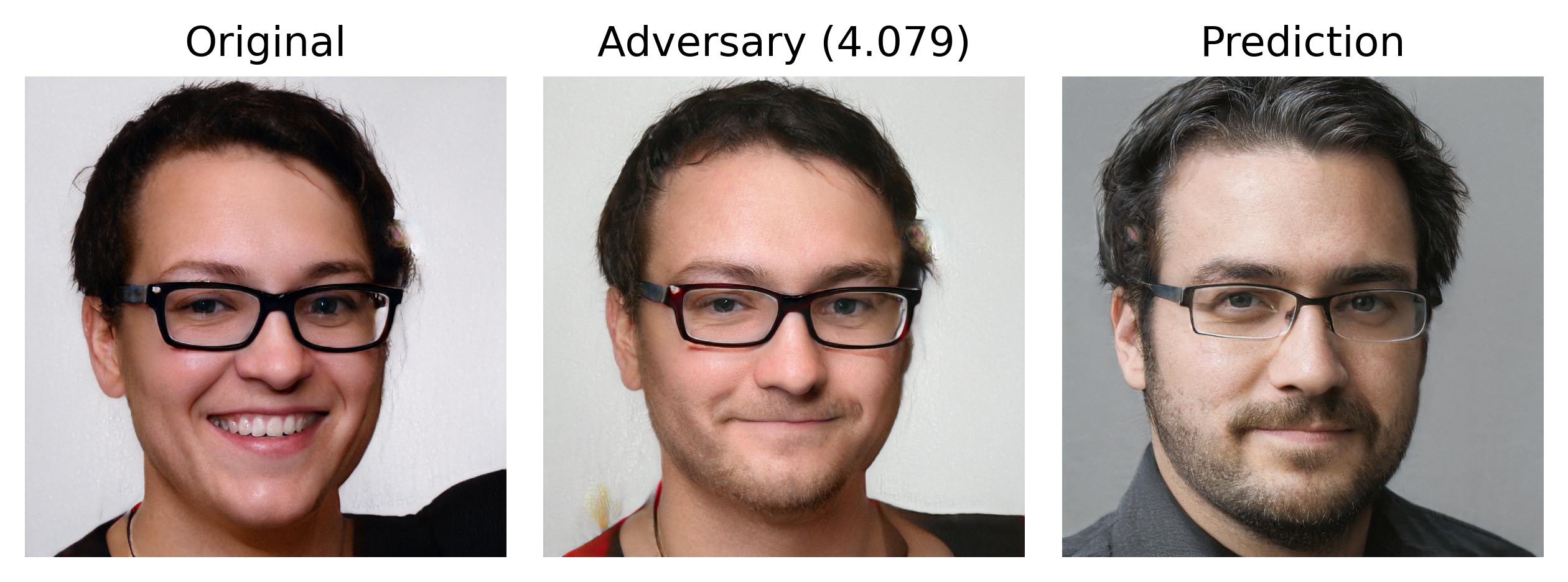}\\
    \raisebox{0.35in}{\rotatebox[origin=t]{90}{4.76}}\includegraphics[trim=0cm 0.3cm 0cm 0.7cm,clip,width=0.96\columnwidth]{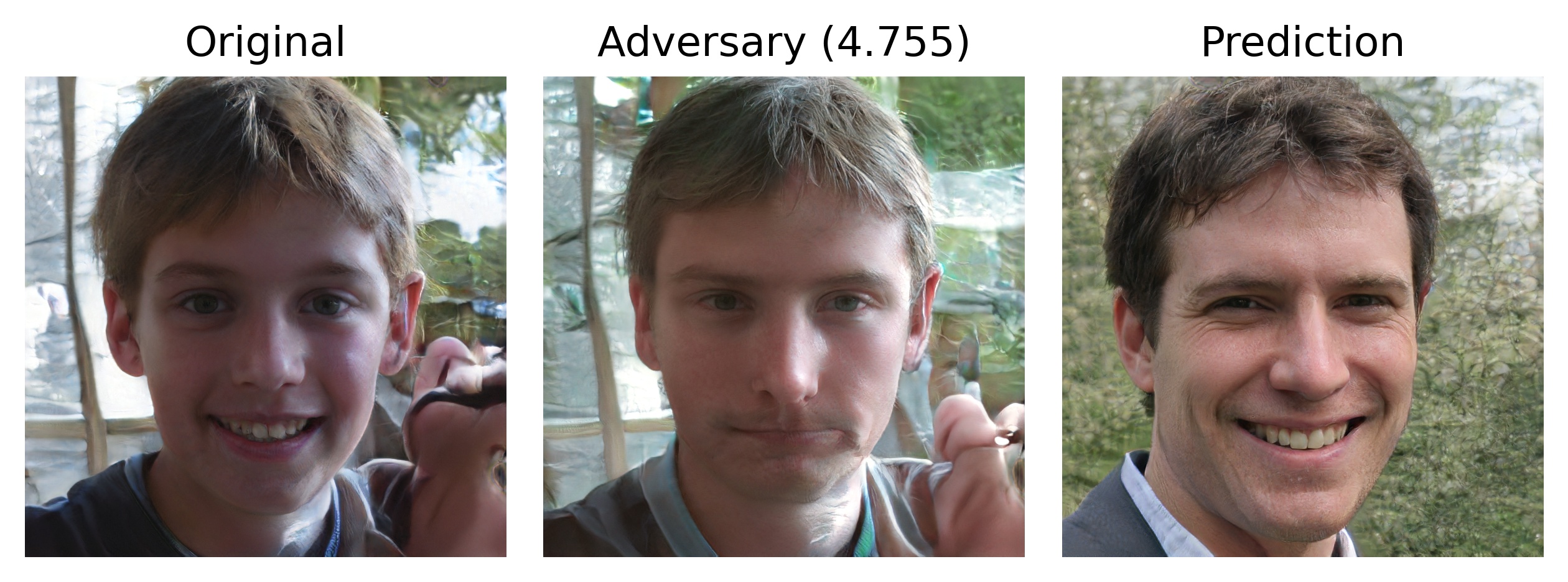}\\
    \raisebox{0.35in}{\rotatebox[origin=t]{90}{7.65}}\includegraphics[trim=0cm 0.3cm 0cm 0.7cm,clip,width=0.96\columnwidth]{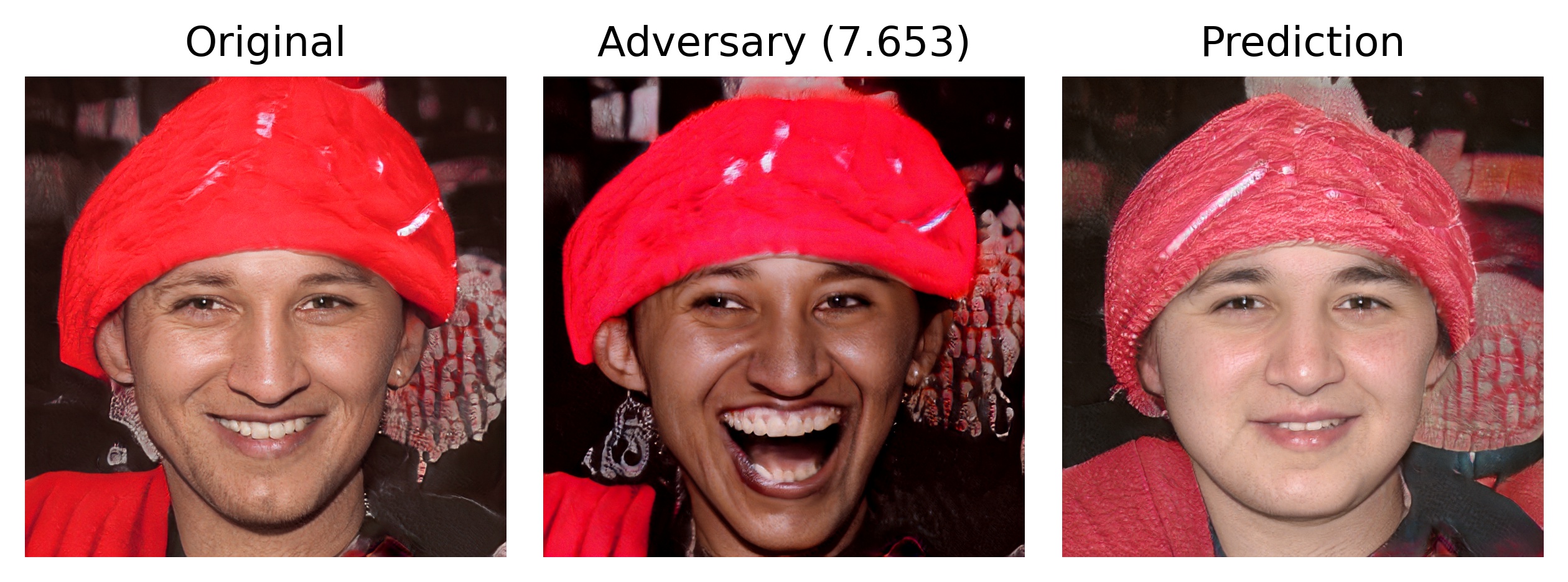}\\
    \raisebox{0.35in}{\rotatebox[origin=t]{90}{6.58}}\includegraphics[trim=0cm 0.3cm 0cm 0.7cm,clip,width=0.96\columnwidth]{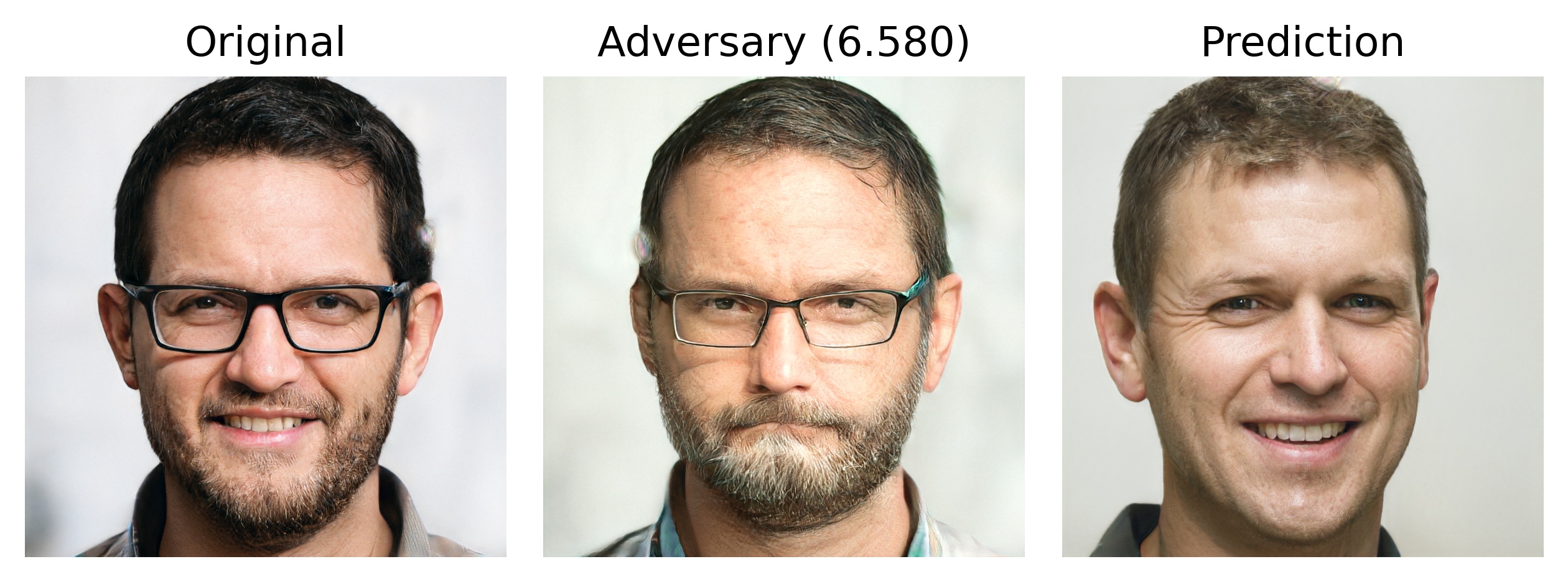}\\
    \raisebox{0.35in}{\rotatebox[origin=t]{90}{2.55}}\includegraphics[trim=0cm 0.3cm 0cm 0.7cm,clip,width=0.96\columnwidth]{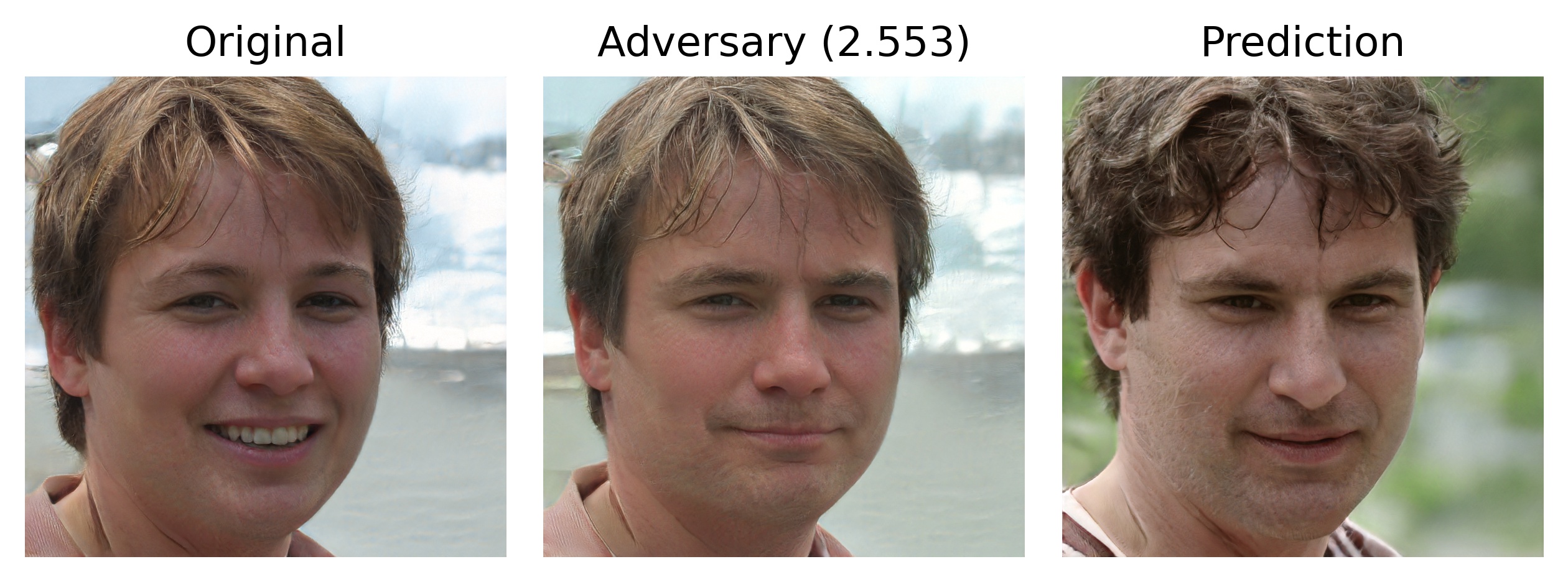}\\
    \raisebox{0.35in}{\rotatebox[origin=t]{90}{3.49}}\includegraphics[trim=0cm 0.3cm 0cm 0.7cm,clip,width=0.96\columnwidth]{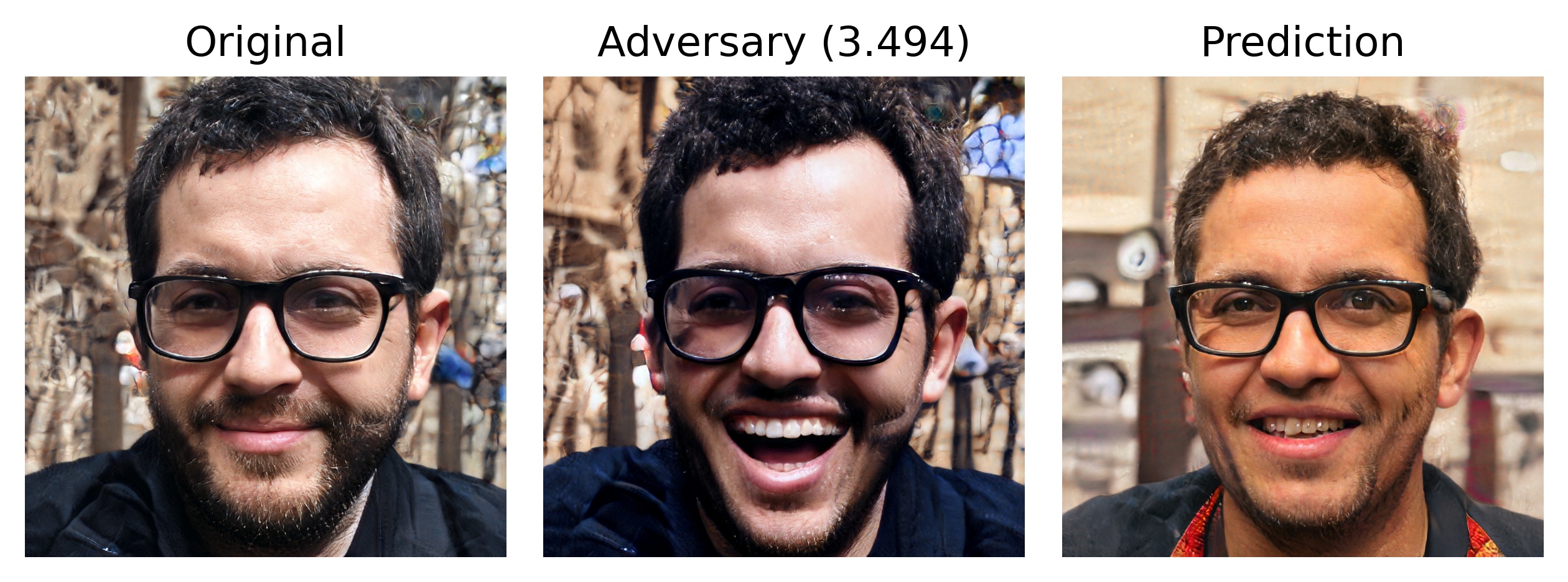}\\
    \raisebox{0.35in}{\rotatebox[origin=t]{90}{3.60}}\includegraphics[trim=0cm 0.3cm 0cm 0.7cm,clip,width=0.96\columnwidth]{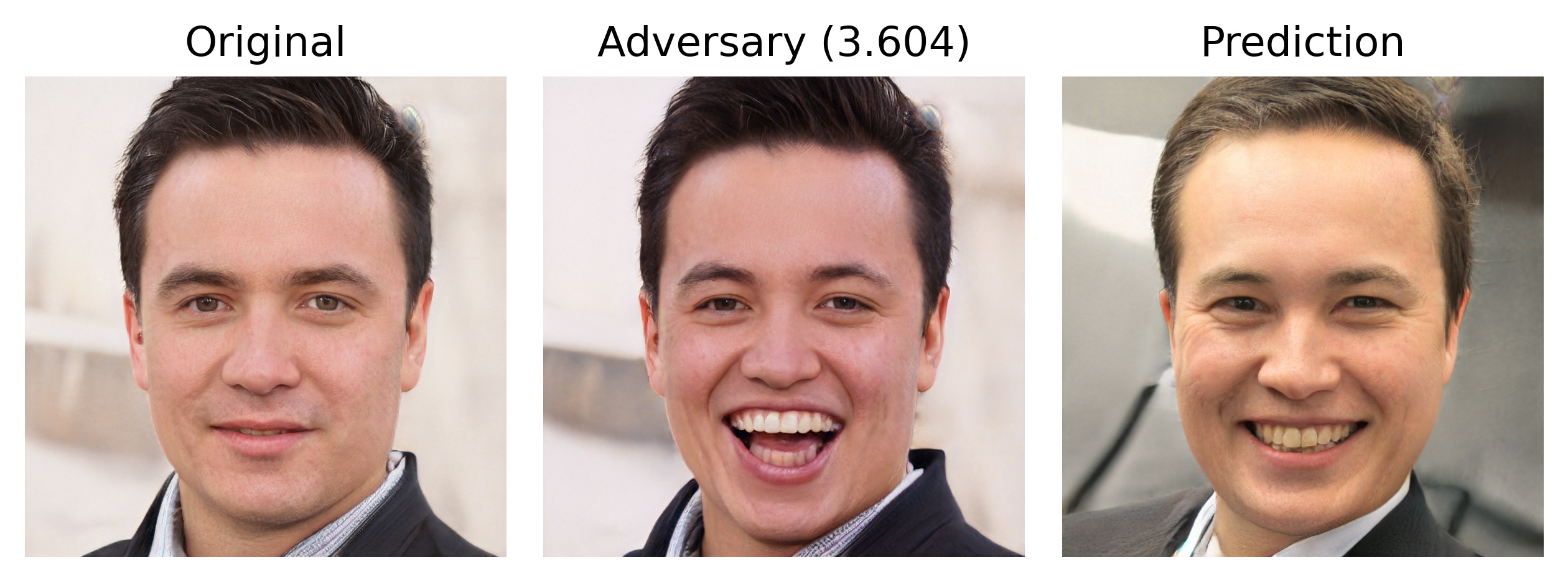}\\
    \caption{\textbf{Adversarial examples found by FAB when only attacking the \underline{Smile} attribute.} 
    Each row is a different identity. 
    We report each perturbation's energy, $\|\pmb{\delta}\|_{M, 2}$, at the far left. 
    \textit{Left:} original face \textcolor{blue}{\fontfamily{cmss}\selectfont\textbf{A}}, \textit{middle:} modified face \textcolor{orange}{\fontfamily{cmss}\selectfont\textbf{A$^\star$}}, \textit{right:} match \textcolor{purple}{\fontfamily{cmss}\selectfont\textbf{B}}. The FRM prefers to match \textcolor{orange}{\fontfamily{cmss}\selectfont\textbf{A$^\star$}} with \textcolor{purple}{\fontfamily{cmss}\selectfont\textbf{B}} rather than with \textcolor{blue}{\fontfamily{cmss}\selectfont\textbf{A}}.}
    \label{fig:fab_onlySmile_qual}
\end{figure}

\begin{figure}
    \centering
    \raisebox{0.35in}{\rotatebox[origin=t]{90}{5.08}}\includegraphics[trim=0cm 0.3cm 0cm 0.7cm,clip,width=0.96\columnwidth]{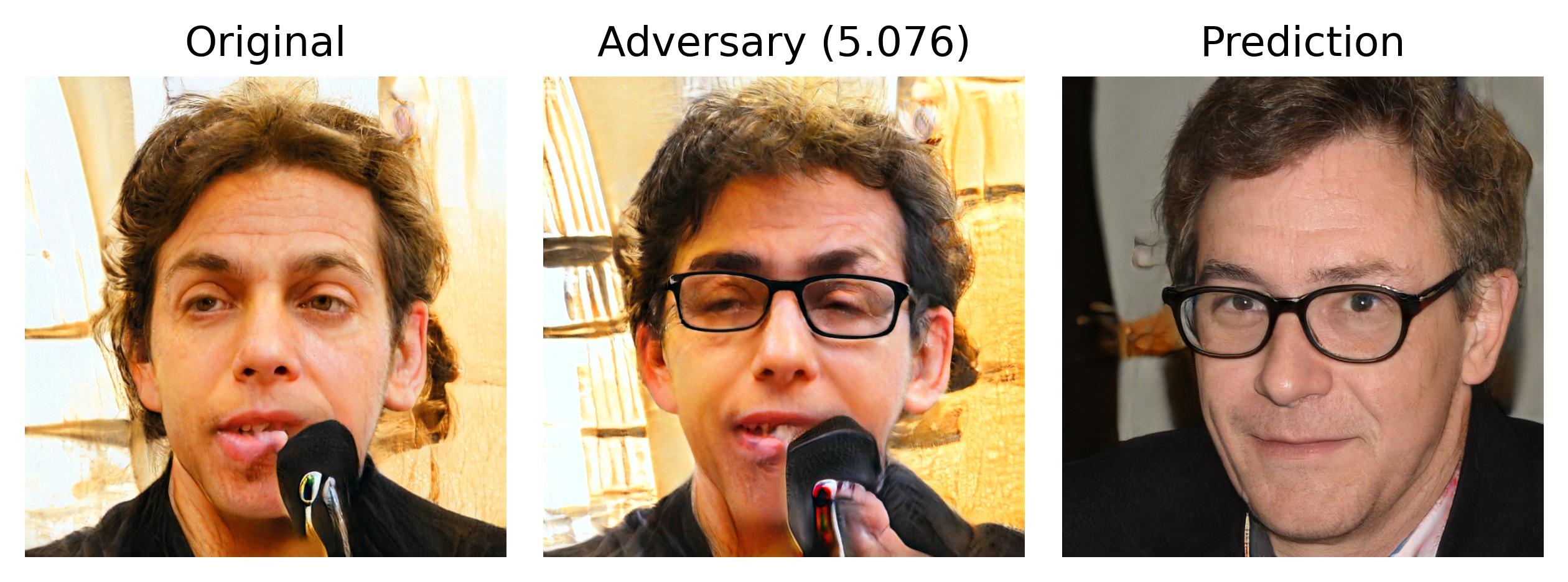}\\
    \raisebox{0.35in}{\rotatebox[origin=t]{90}{6.53}}\includegraphics[trim=0cm 0.3cm 0cm 0.7cm,clip,width=0.96\columnwidth]{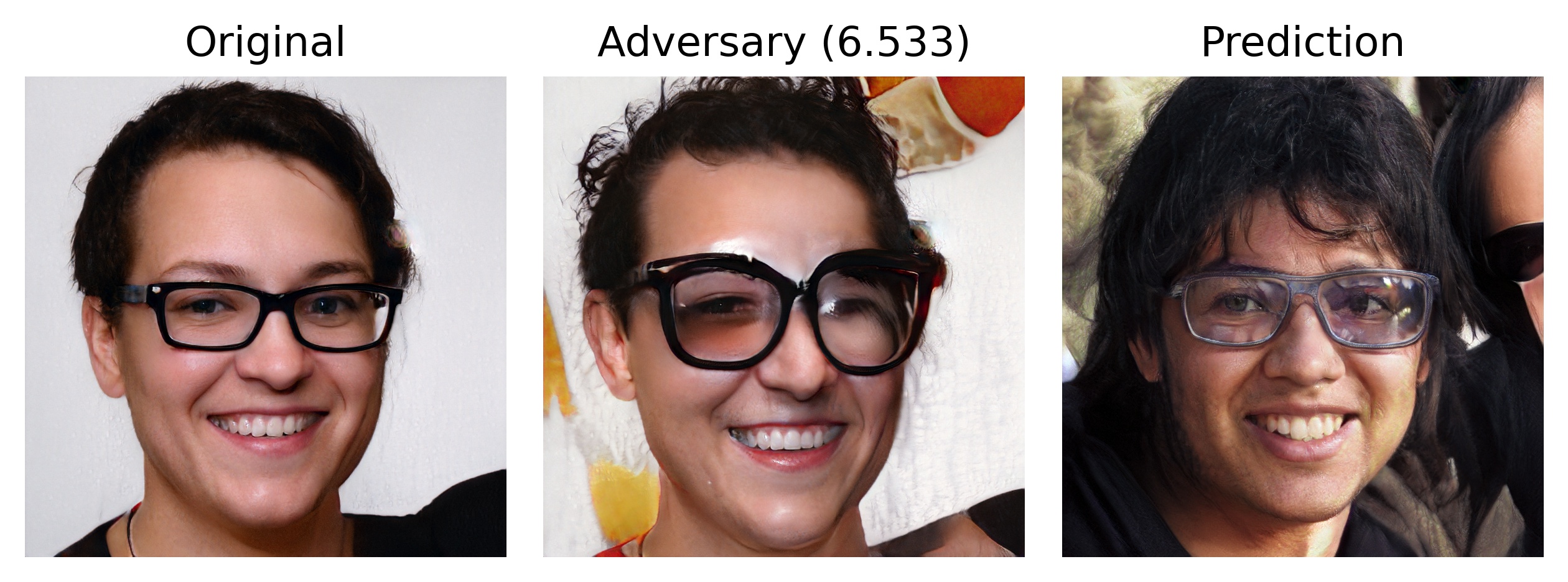}\\
    \raisebox{0.35in}{\rotatebox[origin=t]{90}{3.78}}\includegraphics[trim=0cm 0.3cm 0cm 0.7cm,clip,width=0.96\columnwidth]{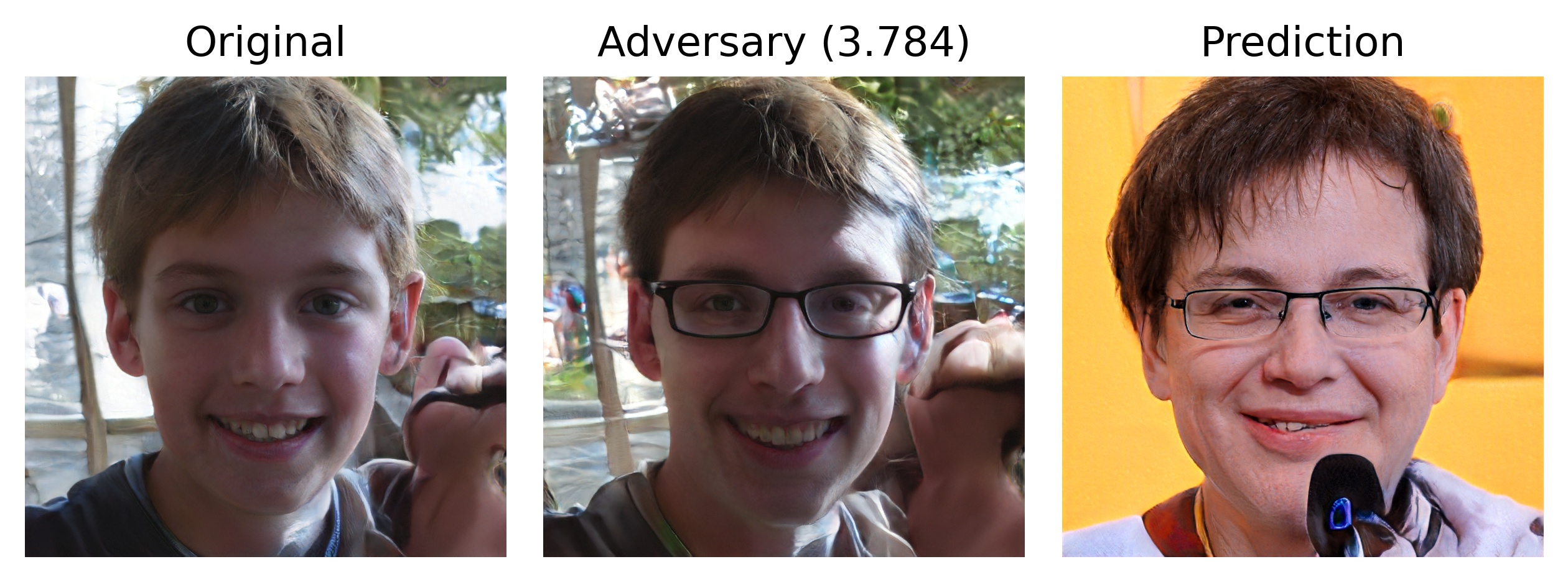}\\
    \raisebox{0.35in}{\rotatebox[origin=t]{90}{6.06}}\includegraphics[trim=0cm 0.3cm 0cm 0.7cm,clip,width=0.96\columnwidth]{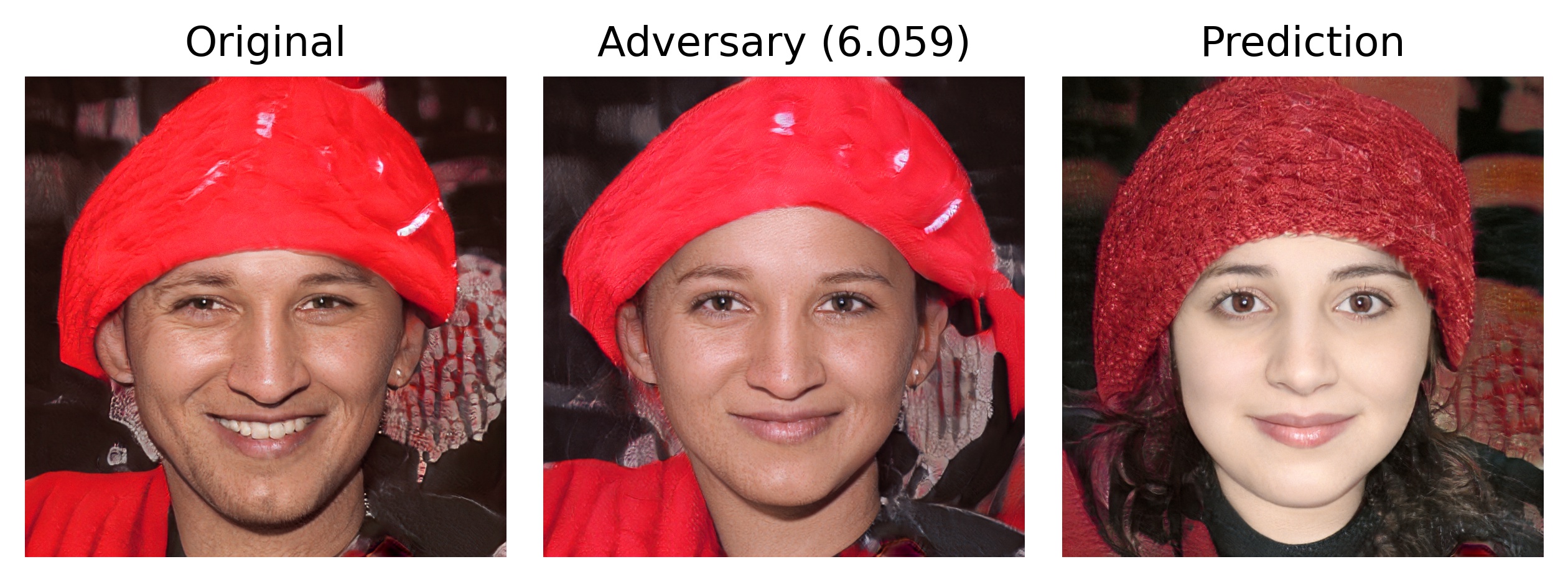}\\
    \raisebox{0.35in}{\rotatebox[origin=t]{90}{6.32}}\includegraphics[trim=0cm 0.3cm 0cm 0.7cm,clip,width=0.96\columnwidth]{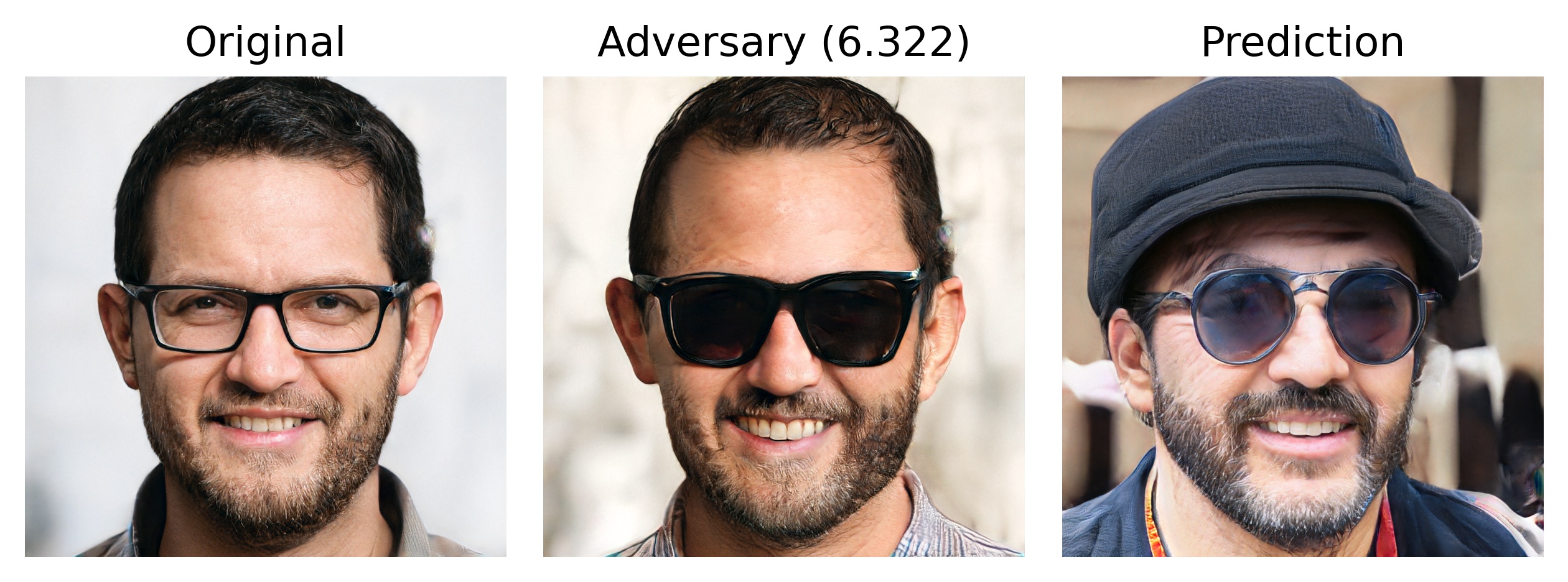}\\
    \raisebox{0.35in}{\rotatebox[origin=t]{90}{5.49}}\includegraphics[trim=0cm 0.3cm 0cm 0.7cm,clip,width=0.96\columnwidth]{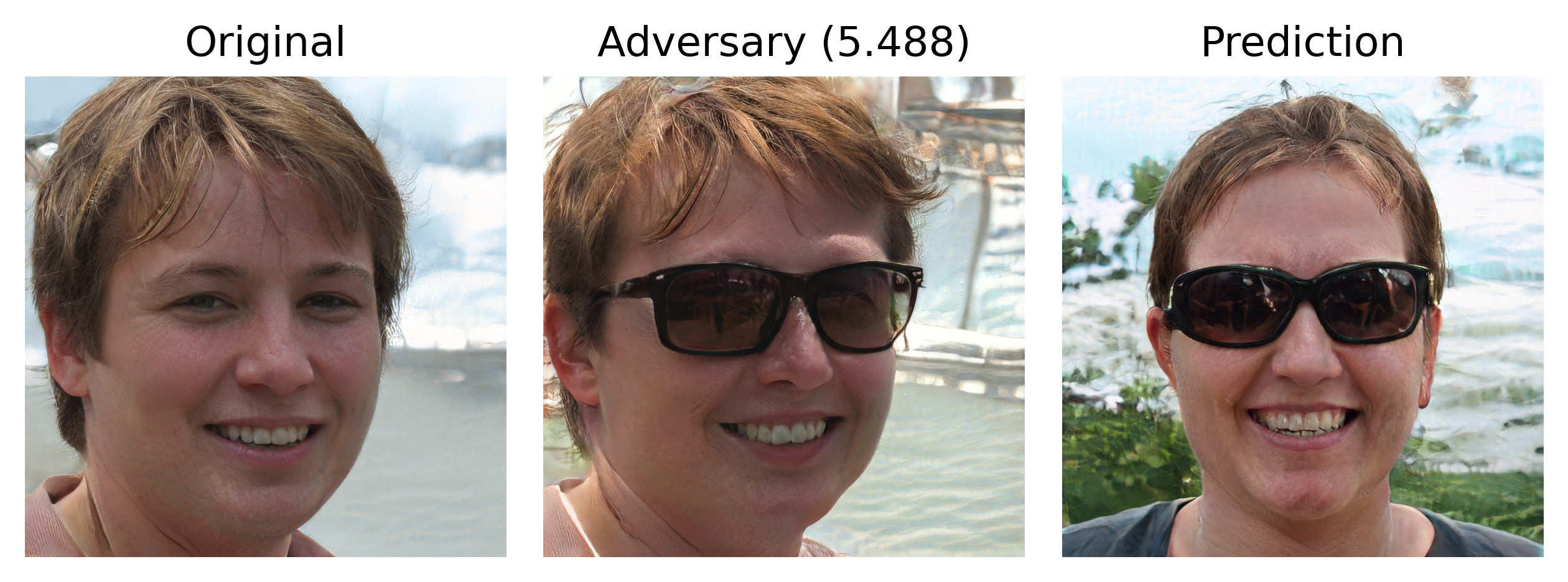}\\
    \raisebox{0.35in}{\rotatebox[origin=t]{90}{3.98}}\includegraphics[trim=0cm 0.3cm 0cm 0.7cm,clip,width=0.96\columnwidth]{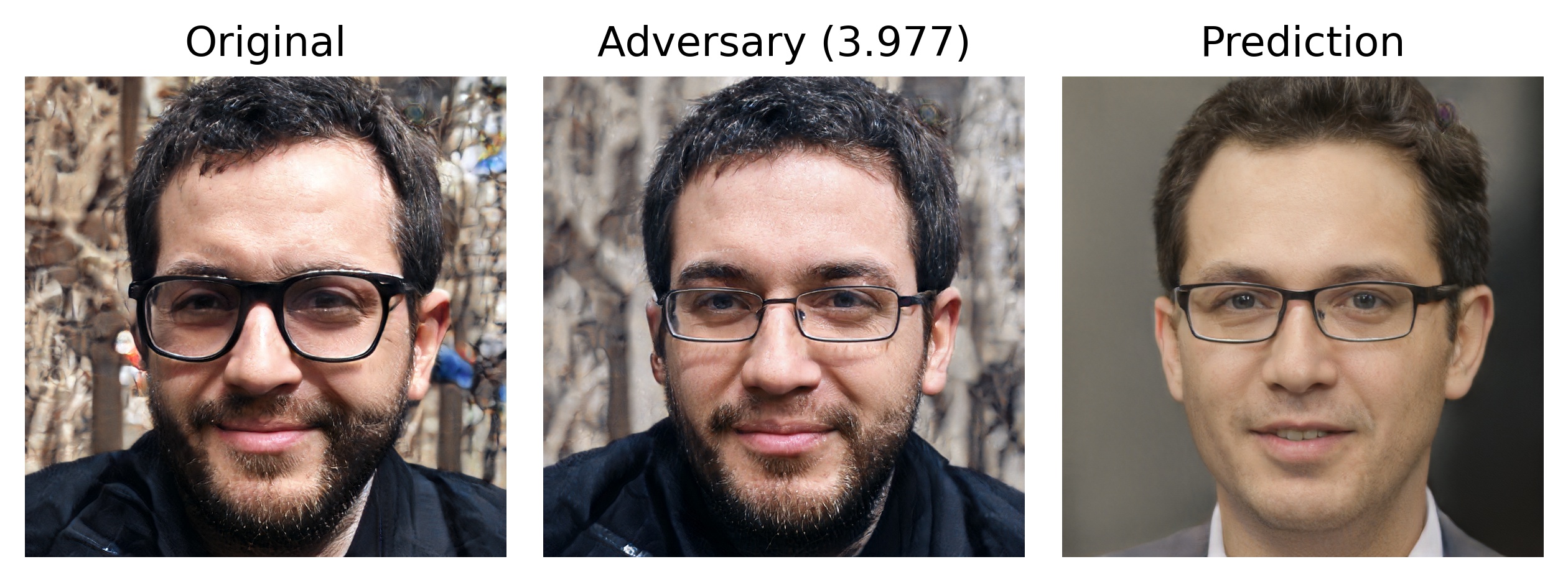}\\
    \raisebox{0.35in}{\rotatebox[origin=t]{90}{6.11}}\includegraphics[trim=0cm 0.3cm 0cm 0.7cm,clip,width=0.96\columnwidth]{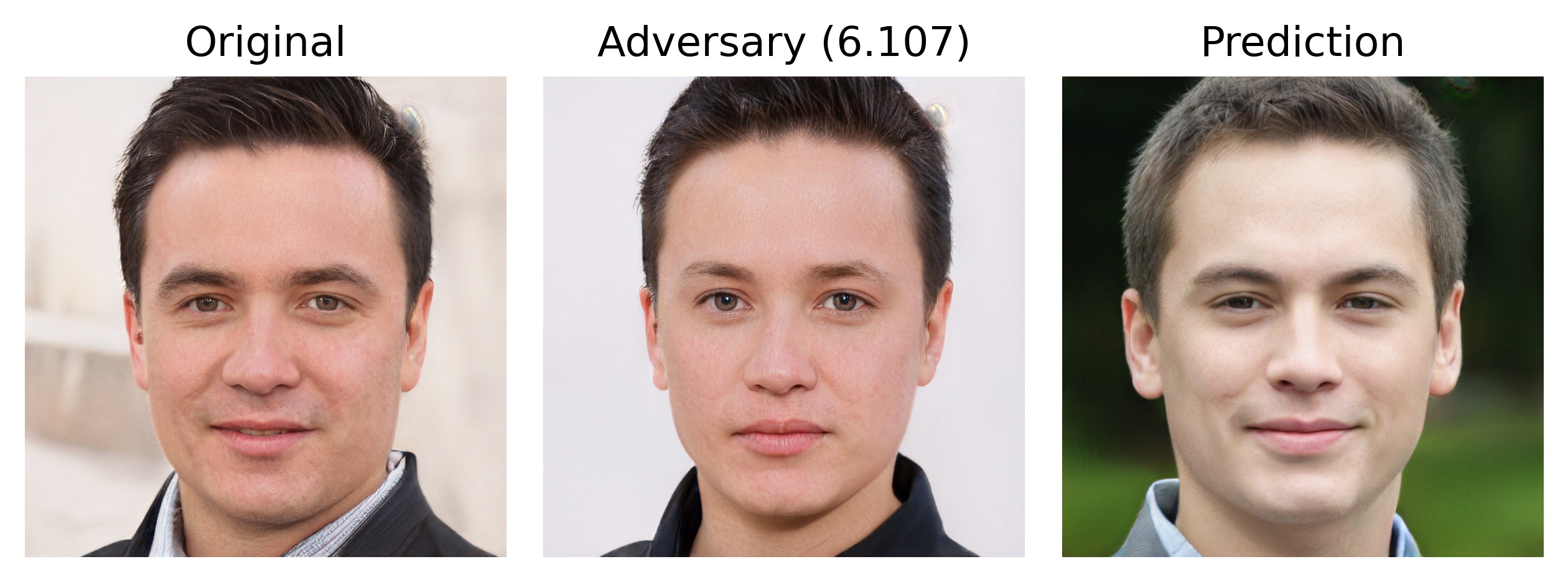}\\
    \caption{\textbf{Adversarial examples found by FAB when only attacking the \underline{Eyeglasses} attribute.} 
    Each row is a different identity. 
    We report each perturbation's energy, $\|\pmb{\delta}\|_{M, 2}$, at the far left. 
    \textit{Left:} original face \textcolor{blue}{\fontfamily{cmss}\selectfont\textbf{A}}, \textit{middle:} modified face \textcolor{orange}{\fontfamily{cmss}\selectfont\textbf{A$^\star$}}, \textit{right:} match \textcolor{purple}{\fontfamily{cmss}\selectfont\textbf{B}}. The FRM prefers to match \textcolor{orange}{\fontfamily{cmss}\selectfont\textbf{A$^\star$}} with \textcolor{purple}{\fontfamily{cmss}\selectfont\textbf{B}} rather than with \textcolor{blue}{\fontfamily{cmss}\selectfont\textbf{A}}.}
    \label{fig:fab_onlyEye_qual}
\end{figure}

\onecolumn

\end{document}